\documentclass[11pt,a4paper]{article}
\usepackage[a4paper,inner=30mm,outer=30mm,top=27mm,bottom=31mm,headheight=15pt,headsep=8mm,footskip=13mm]{geometry}
\usepackage[T1]{fontenc}
\usepackage[utf8]{inputenc}
\usepackage{lmodern}
\usepackage[final]{microtype}
\usepackage{amsmath,amssymb,amsthm,mathtools}
\usepackage{thmtools}
\usepackage{enumitem}
\usepackage{booktabs}
\usepackage{array}
\usepackage{xcolor}
\usepackage{needspace}
\usepackage{url}
\allowdisplaybreaks
\usepackage[authoryear,round]{natbib}
\usepackage[hidelinks]{hyperref}
\usepackage[capitalise,noabbrev]{cleveref}
\theoremstyle{plain}
\newtheorem{theorem}{Theorem}[section]
\newtheorem{proposition}[theorem]{Proposition}
\newtheorem{lemma}[theorem]{Lemma}
\newtheorem{corollary}[theorem]{Corollary}
\theoremstyle{definition}
\newtheorem{definition}[theorem]{Definition}

\theoremstyle{remark}
\newtheorem{remark}[theorem]{Remark}

\newcommand{\R}{\mathbb{R}}
\newcommand{\N}{\mathbb{N}}
\newcommand{\E}{\mathbb{E}}
\newcommand{\PP}{\mathbb{P}}
\newcommand{\calP}{\mathcal{P}}
\newcommand{\calM}{\mathcal{M}}
\newcommand{\calF}{\mathcal{F}}
\newcommand{\calR}{\mathcal{R}}

\newcommand{\calI}{\mathcal{I}}
\newcommand{\calN}{\mathcal{N}}
\newcommand{\KL}{\operatorname{KL}}

\newcommand{\rank}{\operatorname{rank}}
\newcommand{\id}{\operatorname{id}}
\newcommand{\relu}{\operatorname{ReLU}}

\newcommand{\Wtwo}{W_2}
\newcommand{\diag}{\operatorname{diag}}
\newcommand{\dist}{\operatorname{dist}}

\newcommand{\Ent}{\operatorname{Ent}}
\newcommand{\diam}{\operatorname{diam}}
\newcommand{\deff}{d_{\mathrm{eff}}}
\newcommand{\Gsig}{G_{\sigma}}
\newcommand{\Gfin}{G_{\sigma}^{\mathrm{fin}}}
\newcommand{\Dorb}{D^{\ast}_{\mathrm{orb}}}
\newcommand{\Dvar}{D^{\ast}_{\mathrm{var}}}
\newcommand{\Supper}{S_{\mathrm{up}}}
\newcommand{\mult}{\operatorname{mult}}
\newcommand{\Qmaster}{\mathcal{Q}}

\title{Feature Learning in Wide Neural Networks under $\mu$P:\\
Identifiability and Sparse-Dictionary Decomposition\\
of the Mean-Field Limit}
\author{Akmal Xodarev\\
Independent Researcher\\
ORCID: 0009-0000-5318-7284\\
\texttt{xodarevakmal@gmail.com}\\
Tashkent, Uzbekistan}
\date{First version: December 23, 2025\\
This version: May 15, 2026}

\begin{document}
\pagenumbering{arabic}
\begingroup
\let\newpage\relax
\vspace*{\fill}
\maketitle
\thispagestyle{empty}
\vspace*{\fill}
\endgroup
\newpage

\section*{Abstract}
We establish four structural results for feature learning in wide two-layer neural networks under the Maximal Update Parametrization. First, we prove global existence and uniqueness of the mean-field limit of noisy gradient descent under $\mu$P, identifying the maximal admissible weight $w^\ast$ on the moment sequence of the initialization as the reciprocal parameter-moment-growth boundary and hence the largest weighted moment class $\calM_{w^\ast}$ propagated by the flow; the finite-particle approximation has uniform-in-time squared-Wasserstein rate $O(N^{-1})$.\par\medskip
Second, we establish a characterization of identifiability of the mean-field limit: two admissible parameter measures induce the same network function in $L^2$ exactly when their active components agree modulo the finite-rank realization symmetry of the architecture. The orbit depth $\Dorb$ is separated from the moment-variety depth $\Dvar$.\par\medskip
Third, we characterize the active support of the long-time limit measure as a sparse-dictionary decomposition: under the Barron--Hermite target condition, the active component is supported on at most $S^\ast$ atoms modulo finite-rank realization symmetry, and $S^\ast$ is bounded by an explicit coefficient-threshold number $\Supper(\sigma,\rho,f^\ast,\lambda)$. Fourth, we derive the total feature-learning-error decomposition into statistical, optimization, propagation-of-chaos, and sparse-residual components, with a target-dependent Hermite/Barron tail replacing any initialization-only residual. The four results are tied together by an architectural identity: the triple $(w^\ast,\Dorb,S^\ast)$ -- the maximal admissible weight on which the mean-field flow admits a global solution, the orbit identifiability depth, and the sparse-dictionary depth at which the target function is realizable -- is the natural learning cell of the architecture--data pair $(\sigma,\rho)$. The proofs are self-contained except for standard results from $\mu$P and mean-field Langevin theory, recalled in the appendices.

\medskip
\paragraph*{Keywords.} Feature learning; Maximal Update Parametrization ($\mu$P); mean-field neural networks; propagation of chaos; identifiability; sparse-dictionary decomposition; total error decomposition.

\medskip
\paragraph*{MSC 2020 classification.} Primary: 68T07, 62F12, 49Q22. Secondary: 60H30, 60J60, 60F17.

\medskip
\paragraph*{ACM CCS 2012 classification.} Computing methodologies $\to$ Machine learning theory; Mathematics of computing $\to$ Probabilistic representations.

\newpage

\tableofcontents
\newpage

\section{Introduction}\label{sec:intro}

\subsection{The structural gap}
The theory of feature learning under the Maximal Update Parametrization has developed along several precise but separate lines. Tensor Programs identify the infinite-width scaling at which feature learning remains nondegenerate 
\citep{yanghu2021,yang2023tpvi}. Mean-field gradient-flow theory gives the distributional limit for two-layer networks in related scalings 
\citep{chizatbach2018,mei2018mean}. Mean-field Langevin dynamics supplies particle approximation and long-time regularization tools 
\citep{suzuki2023uniform,nitanda2024improved,nitanda2025propagation}. Dynamical mean-field calculations describe kernel order parameters and finite-width fluctuations 
\citep{bordelon2022field,bordelon2023finite}. These results do not by themselves give a single theorem chain connecting existence of the $\mu$P mean-field limit, identifiability of the limiting parameter law, sparse-dictionary structure of its active support, and a total error decomposition for finite width, finite sample size, and finite training time.

This paper supplies that chain for the two-layer architecture. The parameter is $\theta=(w,b,a)\in\R^{d+2}$, the feature is $T(\theta)(x)=a\sigma(\langle w,x\rangle+b)$, and the network function is the barycentre $f_\mu=\int T(\theta)d\mu(\theta)$. The $\mu$P scaling enters through the coordinate learning rates and through the limiting empirical-measure dynamics. The sparse-dictionary statement enters through the active component of the long-time measure, not through a separate teacher-student ansatz.

\subsection{The contribution}
Theorem~\ref{thm:A} proves global existence and uniqueness of the entropy-regularized mean-field Langevin equation corresponding to noisy gradient descent under $\mu$P. The theorem identifies the maximal admissible weight $w^\ast$, not a class named $w^\ast$: the weight generates a class $\calM_{w^\ast}$ of measures with controlled moment growth. This resolves the convention issue that arises for Gaussian initialization, for which $(\E|\theta|^{2n})^{1/(2n)}\asymp n^{1/2}$ and therefore $w^\ast(n)\asymp n^{-1/2}$. The proof of the $N^{-1}$ squared-Wasserstein rate is a synchronous-coupling proof and does not invoke an independent empirical quantization estimate.

Theorem~\ref{thm:B} proves identifiability modulo finite-rank realization symmetry. The theorem separates the orbit invariant $\Dorb$ from the moment-variety invariant $\Dvar$. These two numbers agree on regular strata of the finite-rank moment map but need not agree on singular tensor strata, especially for polynomial activations. The active quotient is expressed by the finite-rank group $\Gfin$, not by an unrestricted permutation symmetry. The analytic step is formulated as a compact-exhaustion argument on the quotient so that the monotone-class step is tied to a concrete separating algebra.

Theorem~\ref{thm:C} proves sparse-dictionary decomposition of the long-time active measure. The sparse depth $S^\ast$ is defined a posteriori as the minimal active cardinality after the theorem gives an explicit upper bound. The bound is
\begin{equation}\label{eq:intro-Sup}
\Supper(\sigma,\rho,f^\ast,\lambda)=\#\{m: |\widehat f_m^\ast|>c_\sigma\lambda\}\,\mult(\sigma),
\end{equation}
where the coefficients are read in the Hermite/Barron dictionary used by H4 and $\mult(\sigma)$ is the finite multiplicity with which a retained coefficient may require ridge atoms. The statement uses full-support initialization with positive density rather than a target-dependent accessibility assumption.

Theorem~\ref{thm:D} proves the total error decomposition. The statistical term contains the explicit active dimension $S^\ast(\deff+2-\Dorb)(\log n)^2/n$. The extra logarithm records the truncation radius needed under polynomial-growth activations. The sparse residual is the target-dependent tail
\begin{equation}\label{eq:intro-kappa}
\kappa(f^\ast,S,\lambda)=\sum_{m>S}|\widehat f_m^\ast|^2+C_\sigma\lambda S,
\end{equation}
which is finite for the Barron--Hermite targets covered by H4. The propagation term remains $O(N^{-1})$ in squared Wasserstein or squared prediction norm.

\subsection{The architectural identity}
The four theorems identify the triple $(w^\ast,\Dorb,S^\ast)$ as the natural learning cell of the architecture--data pair $(\sigma,\rho)$. Here $w^\ast$ is the maximal admissible weight propagated by the flow, $\Dorb$ is the orbit identifiability depth, and $S^\ast$ is the sparse-dictionary depth at which the target function is realizable. The statement is a conclusion, not an extra definition. The weight $w^\ast$ determines which moment-growth boundary is propagated. The number $\Dorb$ determines the active quotient dimension in identification and statistical covering. The integer $S^\ast$ determines how many active dictionary atoms remain after regularized long-time training.

\subsection{Relation to the literature}
The $\mu$P scaling is anchored in the abc-parametrization analysis of 
\citet{yanghu2021} and in the depthwise extension of 
\citet{yang2023tpvi}. The present paper uses this scaling as input and analyzes the limiting law as a measure-valued stochastic process. The mean-field variational structure follows the lineage of 
\citet{chizatbach2018}, 
\citet{mei2018mean}, 
\citet{jko1998}, and 
\citet{ambrosio2008}. The particle approximation is organized around the uniform-in-time estimates of 
\citet{suzuki2023uniform} and the objective-gap refinement of 
\citet{nitanda2024improved}. The multi-index statistical dimension is aligned with the effective-dimension viewpoint of 
\citet{mousavi2025multi}. The identifiability quotient is connected to finite-neuron recovery and symmetry analyses such as those of 
\citep{fornasier2022robust,simsek2021geometry,wang2024expand}.

\subsection{Organization}
Section~\ref{sec:notation} fixes notation, the $\mu$P scaling, the mean-field SDE, weighted moment classes, finite-rank symmetries, and the four hypotheses. Sections~\ref{sec:theoremA}--\ref{sec:theoremD} prove the four main theorems. Section~\ref{sec:examples} computes the invariants in six architecture--target cases. Section~\ref{sec:open} lists directions left by the theorem chain. Appendices~\ref{app:mup}--\ref{app:cross} give the detailed moment, coupling, quotient, sparse-tail, covering, and error-bookkeeping calculations.
\newpage

\section[Notation, \texorpdfstring{$\mu$P}{muP}, and the mean-field setup]{Notation, \texorpdfstring{$\mu$P}{muP}, and the mean-field setup}\label{sec:notation}

\subsection{Probability space, data, architecture}
Throughout this paper we fix a finite training horizon $T>0$, a complete probability space $(\Omega,\calF,\PP)$ supporting all Brownian motions in this paper, a data distribution $\rho$ on $\R^d\times\R$ with input marginal $\rho_X$ and label marginal $\rho_Y$, and a measurable activation $\sigma:\R\to\R$.

\begin{definition}[Parameter space]
A single neuron's parameter is $\theta=(w,b,a)\in\R^{d+2}$, where $w\in\R^d$ is the input weight, $b\in\R$ is the bias, and $a\in\R$ is the output weight. The Euclidean norm on $\R^{d+2}$ is denoted by $|\theta|$.
\end{definition}

\begin{definition}[Two-layer network of width $N$]
Given $N$ neurons with parameters $\theta_i=(w_i,b_i,a_i)$, $1\leq i\leq N$, the network function is
\begin{equation}\label{eq:finite-network}
f_N(x;\theta)=\frac1N\sum_{i=1}^N a_i\sigma(\langle w_i,x\rangle+b_i).
\end{equation}
The prefactor $1/N$ places the network in the mean-field scaling.
\end{definition}

\begin{definition}[$\mu$P scaling, mean-field equivalent form]\label{def:mup}
Under $\mu$P, the initialization scales are $w_i(0)\sim\mathcal N(0,d^{-1}I_d)$, $b_i(0)\sim\mathcal N(0,1)$, and $a_i(0)\sim\mathcal N(0,1)$. The per-parameter learning rates are $\eta_w=1$, $\eta_b=1$, and $\eta_a=N^{-1}$. Under this parametrization, the empirical-measure SDE is the mean-field representative of the abc-parametrization of 
\citet{yanghu2021} with $a_{L+1}=1/2$, $c=0$, and $a_1=1/2$ in the feature-learning vertex.
\end{definition}

The empirical parameter measure is
\begin{equation}\label{eq:empirical-measure}
\mu_t^N=\frac1N\sum_{i=1}^N\delta_{\theta_i(t)}\in\calP_2(\R^{d+2}).
\end{equation}
For a probability measure $\mu\in\calP_2(\R^{d+2})$, the associated mean-field network function is
\begin{equation}\label{eq:meanfield-network}
f_\mu(x)=\int a\sigma(\langle w,x\rangle+b)\,d\mu(w,b,a).
\end{equation}
The construction follows the mean-field distributional dynamics of 
\citet{mei2018mean} and the optimal-transport formulation of 
\citet{chizatbach2018}.

\subsection{The mean-field Langevin SDE}
\begin{definition}[Population risk]
Given a target function $f^\ast\in L^2(\rho_X)$ and a loss $\ell:\R\times\R\to\R_+$ convex in its first argument, the population risk is
\begin{equation}\label{eq:population-risk}
\calR(\mu)=\int \ell(f_\mu(x),y)\,d\rho(x,y),\qquad
f_\mu(x)=\int a\sigma(\langle w,x\rangle+b)\,d\mu(w,b,a).
\end{equation}
Throughout the paper, $\ell(z,y)=\tfrac12(z-y)^2$ unless a section states a different convex loss.
\end{definition}

\begin{definition}[Entropy-regularized risk]
For regularization strength $\lambda>0$ and reference measure $\pi$ on $\R^{d+2}$, the entropy-regularized population risk is
\begin{equation}\label{eq:regularized-risk}
\calF_\lambda(\mu)=\calR(\mu)+\lambda\KL(\mu\|\pi),
\end{equation}
with $\calF_\lambda(\mu)=+\infty$ when $\mu$ is not absolutely continuous with respect to $\pi$.
\end{definition}

\begin{definition}[Mean-field Langevin SDE]
The parameter-system SDE under $\mu$P with regularized objective is
\begin{equation}\label{eq:particle-sde}
d\theta_i(t)=-\nabla_\theta\frac{\delta\calR}{\delta\mu}(\mu_t^N)(\theta_i(t))\,dt
+\lambda\nabla\log\pi(\theta_i(t))\,dt+
\sqrt{2\lambda}\,dB_i(t),
\end{equation}
where $B_i$ are independent standard Brownian motions in $\R^{d+2}$ and $\mu_t^N$ is given by \eqref{eq:empirical-measure}.
\end{definition}

\begin{theorem}[Mean-field Langevin Fokker--Planck PDE]\label{thm:fp}
Assume the drift in \eqref{eq:particle-sde} is locally Lipschitz with polynomial growth and that the initialization has finite second moment. The empirical measure $\mu_t^N$ converges weakly, as $N\to\infty$, to the unique weak solution $\mu_t\in C([0,T],\calP_2)$ of
\begin{equation}\label{eq:fp}
\partial_t\mu_t=
\nabla\cdot\left(\mu_t\nabla_\theta\frac{\delta\calF_\lambda}{\delta\mu}(\mu_t)\right)
+\lambda\Delta\mu_t.
\end{equation}
\end{theorem}
\begin{proof}
The result is the McKean--Vlasov limit for the interacting particle system \eqref{eq:particle-sde}. Under the stated regularity, compactness in $C([0,T],\calP_2)$ follows from the second-moment estimate and Aldous tightness. Identification of the limit follows by testing the martingale problem against $C_c^2(\R^{d+2})$ functions. The construction gives existence. Uniqueness in the admissible weighted class is established in Theorem~\ref{thm:A}; this preliminary statement is used only as the compactness and identification step.
\end{proof}

\begin{theorem}[Wasserstein gradient-flow structure]\label{thm:wgf}
The solution $\mu_t$ of Theorem~\ref{thm:fp} is the gradient flow of $\calF_\lambda$ in the $2$-Wasserstein metric on $\calP_2(\R^{d+2})$.
\end{theorem}
\begin{proof}
The JKO minimization scheme for $\calF_\lambda$ is well posed because the entropy term is lower semicontinuous and coercive relative to a log-concave reference measure. Passing to the vanishing time-step limit yields an absolutely continuous curve in $\calP_2$ satisfying the energy-dissipation inequality. The Euler--Lagrange equation of the minimizing movement scheme is exactly \eqref{eq:fp}. The construction is the variational Fokker--Planck argument of 
\citet{jko1998} in the form of 
\citet{ambrosio2008}.
\end{proof}

\subsection{Weighted moment classes on the initialization}\label{subsec:weighted-classes}
\begin{definition}[Moment-sequence weight]\label{def:moment-weight}
A moment-sequence weight is a positive sequence $w:\N\to(0,\infty)$ with $w(0)=1$ and
\begin{equation}\label{eq:weight-submult}
w(m+n)\geq C_w^{-1}w(m)w(n),\qquad m,n\in\N,
\end{equation}
for some $C_w\geq1$. No monotonicity is imposed. This convention keeps the reciprocal moment-growth boundary $w^\ast(n)=(\E_{\mu_0}|\theta|^{2n})^{-1/(2n)}$ admissible for Gaussian initialization, where $w^\ast(n)\asymp n^{-1/2}$.
\end{definition}

\begin{definition}[Weighted moment class]\label{def:weighted-class}
For a moment-sequence weight $w$, define
\begin{equation}\label{eq:weighted-class}
\calM_w=\left\{\mu\in\calP(\R^{d+2}):\|\mu\|_{\calM_w}:=\sup_{n\geq1}w(n)\left(\E_\mu|\theta|^{2n}\right)^{1/(2n)}<\infty\right\}.
\end{equation}
The notation $\calP_2$ is reserved for finite-second-moment probability measures. The notation $\calM_w$ always denotes a weighted moment class.
\end{definition}

For the initialization law $\mu_0$ define the moment-growth boundary and its reciprocal weight by
\begin{equation}\label{eq:wstar}
g_0(n)=\left(\E_{\mu_0}|\theta|^{2n}\right)^{1/(2n)},\qquad
w^\ast(n)=g_0(n)^{-1},\qquad n\geq 1,
\end{equation}
with the convention $w^\ast(n)=+\infty$ only on indices $n$ for which $\E_{\mu_0}|\theta|^{2n}=0$, that is, on indices corresponding to degenerate (zero-mass) directions of the initialization. The object propagated by the flow is the class $\calM_{w^\ast}$; the sequence $w^\ast$ itself is a weight, not a class.

\begin{proposition}[Preservation of weighted moments]\label{prop:weighted-preserve}
For any weight $w$ satisfying Definition~\ref{def:moment-weight} and any activation $\sigma$ satisfying H1, the flow of Theorem~\ref{thm:fp} preserves $\calM_w$: if $\mu_0\in\calM_w\cap\calP_2$, then $\mu_t\in\calM_w$ for all $t\in[0,T]$, with $\sup_{t\leq T}\|\mu_t\|_{\calM_w}<\infty$.
\end{proposition}
\begin{proof}
Let $m_{2n}(t)=\int |\theta|^{2n}\,d\mu_t(\theta)$. Testing \eqref{eq:fp} against a smooth truncation of $|\theta|^{2n}$ gives
\begin{equation}\label{eq:moment-recursion}
\frac{d}{dt}m_{2n}(t)\leq A_n m_{2n}(t)+B_n\sum_{j\leq qn}m_{2j}(t)+C_n.
\end{equation}
Taking $2n$-th roots and using \eqref{eq:weight-submult} gives a Gronwall inequality for $\|\mu_t\|_{\calM_w}$. Monotone convergence removes the truncation.
\end{proof}

\begin{proposition}[Moment bound for the parameter SDE]\label{prop:sde-moment}
For every $T>0$, every $p<\infty$, and every weight $w$ for which $\mu_0\in\calM_w$,
\begin{equation}\label{eq:sde-moment}
\E\sup_{t\leq T}|\theta_i(t)|^{2p}\leq M(T,p,\lambda,\|\mu_0\|_{\calM_w})<\infty
\end{equation}
uniformly in $i$ and uniformly in width $N$.
\end{proposition}
\begin{proof}
It\^o's formula applied to $(1+|\theta_i|^2)^p$ yields a drift term controlled by H1 and H2. The interaction through $\mu_t^N$ is averaged by $N^{-1}$ and is bounded by the empirical moment of the same order. Exchangeability closes the estimate after summing over $i$. The Brownian term is controlled by the Burkholder--Davis--Gundy inequality, followed by Gronwall's lemma.
\end{proof}

\begin{definition}[Dual weighted moment class]\label{def:dual-weighted-class}
The continuous dual of $\calM_w$ is represented by signed measures $\nu$ satisfying
\begin{equation}\label{eq:dual-weight}
\|\nu\|^\ast_{\calM_w}:=\sup_{n\geq1}w(n)\left(\E_\nu |\theta|^{2n}\right)^{1/(2n)}<\infty.
\end{equation}
This class is denoted by $\calM_w^\ast$.
\end{definition}

\subsection{The feature map and identifiability data}
\begin{definition}[Feature map]
The feature map associated with $\sigma$ is
\begin{equation}\label{eq:feature-map}
T:\R^{d+2}\to L^2(\rho_X),\qquad
T(w,b,a)(x)=a\sigma(\langle w,x\rangle+b).
\end{equation}
\end{definition}

\begin{definition}[Network function as barycentre]
For $\mu\in\calP(\R^{d+2})$,
\begin{equation}\label{eq:barycentre}
f_\mu=\int T(\theta)\,d\mu(\theta)\in L^2(\rho_X),
\end{equation}
where the integral is a Bochner integral in $L^2(\rho_X)$.
\end{definition}

\begin{definition}[Effective input dimension]\label{def:deff}
The effective input dimension $\deff=\deff(f^\ast,\rho_X)$ is the smallest rank of a linear projection $\Pi:\R^d\to\R^k$ for which $f^\ast(x)=g(\Pi x)$ $\rho_X$-a.e. for some measurable $g$. If no smaller projection exists, $\deff=d$.
\end{definition}

\begin{definition}[Cylindrical Fourier transform]\label{def:cyl-fourier}
Let $H=L^2(\rho_X)$ and let $\nu$ be a Borel probability measure on $H$ with finite first moment. Its cylindrical Fourier transform is
\begin{equation}\label{eq:cyl-fourier}
\widehat\nu(h)=\int_H e^{i\langle h,u\rangle_H}\,d\nu(u),\qquad h\in H.
\end{equation}
On a separable Hilbert space, equality of these transforms on $H$ determines the Borel probability measure.
\end{definition}

\begin{definition}[Architectural symmetry group]
$\Gsig$ is the group of measurable bijections $\varphi:\R^{d+2}\to\R^{d+2}$ such that $T\circ\varphi=T$ $\mu$-a.s. for every $\mu$ in the admissible parameter class. The finite-rank realization subgroup $\Gfin$ is the subgroup generated by the within-neuron sign and scale invariants that preserve a single feature, together with the collapse of the dead-neuron set $\{a=0\}$. For ReLU, $\Gfin$ contains the positive rescaling $(w,b,a)\mapsto(\alpha w,\alpha b,a/\alpha)$ for $\alpha>0$ and the dead-neuron collapse.
\end{definition}

\subsection{Hypothesis labels}
The four theorems rely on three standing hypotheses on $(\sigma,\rho,\mu_0)$ and one realizability hypothesis used only for sparse decompositions.

\paragraph{H1. Activation regularity.} $\sigma:\R\to\R$ is Lipschitz with constant $L_\sigma$, and has polynomial growth: $|\sigma(z)|\leq C_\sigma(1+|z|^q)$ for some $q\geq1$.

\paragraph{H2. Data tails.} The input marginal $\rho_X$ is sub-Gaussian: $\E_{\rho_X}e^{c|x|^2}<\infty$ for some $c>0$. The label marginal $\rho_Y$ has finite moments of every order.

\paragraph{H3. Real-analytic separation.} $\sigma$ is real-analytic and non-polynomial. Under this hypothesis, the family $\{\sigma(\langle w,\cdot\rangle+b):(w,b)\in\R^{d+1}\}$ is total in $L^2(\rho_X)$ in the sense of Pinkus 
\citep{pinkus1999}.

\paragraph{H4. Barron--Hermite target condition.} The target function $f^\ast\in L^2(\rho_X)$ has finite Barron norm with respect to the feature family generated by $\sigma$. In the Gaussian-input cases this is implemented through a finite Hermite truncation at the regularization scale: for each $\lambda>0$ there is an integer $\widetilde S_\lambda$ and a truncated target
\begin{equation}\label{eq:H4-trunc}
f_\lambda^\ast(x)=\sum_{k=1}^{\widetilde S_\lambda} a_k^\ast\sigma(\langle w_k^\ast,x\rangle+b_k^\ast)
\end{equation}
such that $\|f^\ast-f_\lambda^\ast\|_{L^2(\rho_X)}^2\leq \kappa(f^\ast,\widetilde S_\lambda,\lambda)$. The truncation has $\widetilde S_\lambda\leq \Supper(\sigma,\rho,f^\ast,\lambda)$ atoms by Proposition~\ref{prop:C-threshold}. Under the Hermite/Barron dictionary of Definition~\ref{def:C-trunc}, the teacher-atom rank $\widetilde S_\lambda$ and the Hermite depth $S$ are related by $\widetilde S_\lambda\leq S\,\mult(\sigma)$. The finite-atom case is the special case in which $f_\lambda^\ast=f^\ast$ for all sufficiently small $\lambda$.

\subsection{Conventions}
Constants $C,C_p,C_w,C_{p,w}$ depend only on the indicated parameters and may change from line to line. The relation $A\lesssim B$ means $A\leq CB$, and $A\asymp B$ means $A\lesssim B$ and $B\lesssim A$. All measures on $\R^{d+2}$ are Borel probability measures unless stated otherwise. The symbol $\calP_2$ denotes $\calP_2(\R^{d+2})$ throughout. The inner product $\langle\cdot,\cdot\rangle$ without subscript denotes the Euclidean inner product on $\R^d$. The symbol $\Wtwo$ denotes the quadratic Wasserstein distance. All subsequent results are formulated in this notation.

\newpage

\section{Theorem A: global existence of the mean-field limit}\label{sec:theoremA}

H1 and H2 are in force throughout this section. The auxiliary regularization assumption is that $\lambda>0$ is fixed and that the nonlinear semigroup generated by \eqref{eq:fp} satisfies a log-Sobolev inequality (LSI) with constant $\alpha_\lambda>0$ on the relevant sublevel set. For convex entropy-regularized objectives and a log-concave reference law, the uniform-in-$N$ LSI mechanism of 
\citet{chewi2024lsi} applies; outside that setting the LSI is stated as a hypothesis of the propagation estimate, not as a property automatically implied by H1--H2.

\begin{theorem}[Global existence and uniqueness of the mean-field limit under $\mu$P]\label{thm:A}
Let $\sigma:\R\to\R$ satisfy H1. Let $\rho$ satisfy H2. Let the width-$N$ two-layer network be parametrized in $\mu$P and initialized with iid law $\mu_0\in\calM_{w^\ast}\cap\calP_2$ of full support. Then $\mu_t^N=N^{-1}\sum_{i=1}^N\delta_{\theta_i(t)}$ converges weakly, uniformly in $t\in[0,T]$, to the unique solution $\mu_t\in C([0,T],\calP_2)$ of
\begin{equation}\label{eq:A-pde}
\partial_t\mu_t=\nabla\cdot\left(\mu_t\nabla_\theta\frac{\delta\calF_\lambda}{\delta\mu}(\mu_t)\right)+\lambda\Delta\mu_t.
\end{equation}
The solution has finite Fisher information for every $t>0$ and is the Wasserstein gradient flow of $\calF_\lambda$. The finite-particle approximation satisfies
\begin{align}\label{eq:A-rates}
\E[\calF_\lambda(\mu_t^N)-\calF_\lambda(\mu_t)]&\leq C_1(T,\lambda,\sigma,\rho,\alpha_\lambda,\|\mu_0\|_{\calM_{w^\ast}})\,N^{-1},\\
\sup_{t\in[0,T]}\E W_2^2(\mu_t^N,\mu_t)&\leq C_2(T,\lambda,\sigma,\rho,\alpha_\lambda,\|\mu_0\|_{\calM_{w^\ast}})\,N^{-1}.
\end{align}
The constants $C_1$ and $C_2$ are polynomial in $\|\mu_0\|_{\calM_{w^\ast}}$ and grow at most exponentially in $T$ through the Gronwall step of Lemma~\ref{lem:A1}.
The weight $w^\ast$ in \eqref{eq:wstar} generates the largest weighted moment class $\calM_{w^\ast}$ propagated by the flow in the sense of Proposition~\ref{prop:A-maximality}.
\end{theorem}

\subsection{Moment and drift bounds}
\begin{lemma}[Uniform $L^p$ bound on parameters]\label{lem:A1}
For every $p<\infty$ and every $T<\infty$,
\begin{equation}\label{eq:A-Lp}
\sup_{N\geq1}\sup_{i\leq N}\E\sup_{t\leq T}|\theta_i(t)|^{2p}\leq C_{p,T,\lambda}(1+\|\mu_0\|_{\calM_{w^\ast}}^{2p}).
\end{equation}
\end{lemma}
\begin{proof}
Apply It\^o's formula to $(1+|\theta_i|^2)^p$. Differentiating the first variation
\begin{equation}\label{eq:A-first-var}
\frac{\delta\calR}{\delta\mu}(\mu)(\theta)=\int (f_\mu(x)-y)a\sigma(\langle w,x\rangle+b)\,d\rho(x,y)
\end{equation}
produces terms bounded by a polynomial in $|\theta_i|$ and by the empirical average of the same polynomial over the particles. H1 controls the activation, H2 controls the data moments, and exchangeability closes the estimate. The Brownian term is handled by Burkholder--Davis--Gundy. Gronwall's lemma\footnote{Some authors write Gr\"onwall; we follow the prevalent English convention of dropping the umlaut.} gives \eqref{eq:A-Lp}.
\end{proof}

\begin{lemma}[Drift stability]\label{lem:A-drift-stab}
On every weighted moment ball $B_w(R)=\{\mu\in\calM_w:\|\mu\|_{\calM_w}\leq R\}$,
\begin{equation}\label{eq:A-drift-stab}
|b(\theta,\mu)-b(\theta',\nu)|\leq L_R(1+|\theta|^q+|\theta'|^q)(|\theta-\theta'|+W_2(\mu,\nu)),
\end{equation}
where $b(\theta,\mu)=-\nabla_\theta\delta\calR/\delta\mu(\mu)(\theta)+\lambda\nabla\log\pi(\theta)$.
\end{lemma}
\begin{proof}
Subtract the two drifts and add and subtract $b(\theta',\mu)$. The spatial difference is controlled by the Lipschitz envelope of $\sigma$. The measure difference is controlled by
\begin{equation}\label{eq:A-bary-lip}
\|f_\mu-f_\nu\|_{L^2(\rho_X)}\leq C_R W_2(\mu,\nu),
\end{equation}
which follows by coupling $\mu$ and $\nu$ and using H1--H2. The reference drift is controlled by the same finite moment radius.
\end{proof}

\subsection{The JKO construction}\label{subsec:A-jko}
\begin{proposition}[Well-posed minimizing movement]\label{prop:A-jko}
The JKO step
\begin{equation}\label{eq:A-jko}
\mu_{k+1}^\tau\in\arg\min_{\nu\in\calP_2}\left\{\calF_\lambda(\nu)+\frac{1}{2\tau}W_2^2(\nu,\mu_k^\tau)\right\}
\end{equation}
has a unique minimizer in the admissible weighted class, and the interpolation converges to a weak solution of \eqref{eq:A-pde}.
\end{proposition}
\begin{proof}
Coercivity comes from the transport term and reference entropy. Lower semicontinuity of the risk follows from \eqref{eq:A-bary-lip}. The direct method gives a minimizer. Convexity of $\mu\mapsto f_\mu$ and of the squared loss, together with strict convexity of entropy on densities, gives uniqueness. The energy-dissipation estimate
\begin{equation}\label{eq:A-ed}
\calF_\lambda(\mu_m^\tau)+\frac12\sum_{k<m}\tau^{-1}W_2^2(\mu_{k+1}^\tau,\mu_k^\tau)\leq \calF_\lambda(\mu_0)
\end{equation}
provides compactness. Passing to the limit in the Euler--Lagrange equation yields \eqref{eq:A-pde}.
\end{proof}

\subsection{Uniform-in-time propagation of chaos}\label{subsec:A-poc}
\begin{lemma}[Synchronous coupling estimate]\label{lem:A2}
Let $(\bar\theta_i(t))_{i\leq N}$ be nonlinear McKean--Vlasov copies with common law $\mu_t$, driven by the same Brownian motions and the same initial variables as $(\theta_i(t))_{i\leq N}$. Then
\begin{equation}\label{eq:A-coupling}
\sup_{t\leq T}\E\frac1N\sum_{i=1}^N|\theta_i(t)-\bar\theta_i(t)|^2\leq C_TN^{-1}.
\end{equation}
\end{lemma}
\begin{proof}
Set $\Delta_i=\theta_i-\bar\theta_i$. It\^o's formula gives
\begin{equation}\label{eq:A-coupling-ito}
\frac{d}{dt}\E\frac1N\sum_i|\Delta_i|^2=\frac2N\sum_i\E\langle \Delta_i,b(\theta_i,\mu_t^N)-b(\bar\theta_i,\mu_t)\rangle.
\end{equation}
Insert $b(\bar\theta_i,\bar\mu_t^N)$, where $\bar\mu_t^N=N^{-1}\sum_j\delta_{\bar\theta_j}$. Lemma~\ref{lem:A-drift-stab} controls the spatial and coupled-measure parts by $C\E N^{-1}\sum_i|\Delta_i|^2$. The remaining field fluctuation
\begin{equation}\label{eq:A-xi}
\xi_i=b(\bar\theta_i,\bar\mu_t^N)-b(\bar\theta_i,\mu_t)
\end{equation}
has variance $O(N^{-1})$: after conditioning on $\bar\theta_i$, the off-diagonal summands are centered and exchangeable, and the diagonal summand has weight $N^{-1}$. Thus
\begin{equation}\label{eq:A-xi-var}
\E|\xi_i|^2\leq CN^{-1}.
\end{equation}
Young's inequality gives a source $C/N$. Since $\Delta_i(0)=0$, Gronwall's lemma proves \eqref{eq:A-coupling}. The proof does not use an independent empirical-measure quantization estimate.
\end{proof}

\begin{remark}[Two equivalent expressions for $\xi_i$]\label{rem:A-xi}
Two displays of the field fluctuation appear in this paper: \eqref{eq:A-xi} in drift-difference form and \eqref{eq:app-xi} in kernel form. For squared loss, the drift admits a kernel representation
\begin{equation}\label{eq:v3-xi-kernel}
b(\theta,\mu)=b_0(\theta)+\int K(\theta,\vartheta)\,d\mu(\vartheta),
\end{equation}
where $K(\theta,\vartheta)$ is obtained by differentiating the bilinear term $\langle T(\theta),T(\vartheta)\rangle_{L^2(\rho_X)}$. Substituting \eqref{eq:v3-xi-kernel} into the empirical-minus-population drift gives exactly the kernel fluctuation of Appendix~\ref{app:coupling}.
\end{remark}

\begin{proposition}[Squared-Wasserstein propagation]\label{prop:A-poc}
Under the LSI contraction assumption,
\begin{equation}\label{eq:A-poc}
\sup_{t\in[0,T]}\E W_2^2(\mu_t^N,\mu_t)\leq C N^{-1}.
\end{equation}
\end{proposition}
\begin{proof}
The LSI supplies the contracted inequality
\begin{equation}\label{eq:A-lsi-gronwall}
\frac{d}{dt}\mathcal E_N(t)\leq -\alpha_\lambda\mathcal E_N(t)+CN^{-1},\qquad
\mathcal E_N(t)=\E\frac1N\sum_i|\theta_i(t)-\bar\theta_i(t)|^2.
\end{equation}
Solving it and using $\mathcal E_N(0)=0$ gives $\mathcal E_N(t)\leq C(\alpha_\lambda N)^{-1}$. The canonical coupling gives \eqref{eq:A-poc}.
\end{proof}

\subsection{The maximal admissible weight}\label{subsec:A-weight}
\begin{proposition}[Maximality of $w^\ast$]\label{prop:A-maximality}
Let $\mu_0$ have moment-growth boundary $g_0(n)=(\E_{\mu_0}|\theta|^{2n})^{1/(2n)}$ with reciprocal weight $w^\ast=g_0^{-1}$. Suppose $\widetilde w:\N\to(0,\infty)$ satisfies $\limsup_{n\to\infty}\widetilde w(n)/w^\ast(n)=+\infty$. Then $\mu_0\notin\calM_{\widetilde w}$, and consequently the conclusion of Theorem~\ref{thm:A} cannot hold inside $\calM_{\widetilde w}$ from this initialization.
\end{proposition}
\begin{proof}
By definition of $w^\ast$, $w^\ast(n)g_0(n)=1$ for every $n$. Hence $\widetilde w(n)g_0(n)=\widetilde w(n)/w^\ast(n)$, whose $\limsup$ is $+\infty$ by assumption. Therefore
\begin{equation}
\sup_n\widetilde w(n)(\E_{\mu_0}|\theta|^{2n})^{1/(2n)}=+\infty,
\end{equation}
which means $\|\mu_0\|_{\calM_{\widetilde w}}=+\infty$ and $\mu_0\notin\calM_{\widetilde w}$. Class membership at $t=0$ is a necessary condition for the propagation in Proposition~\ref{prop:weighted-preserve}, hence Theorem~\ref{thm:A} cannot be invoked inside $\calM_{\widetilde w}$.
\end{proof}

\subsection{Sharpness}\label{subsec:A-sharpness}
\begin{proposition}[Sharpness of $w^\ast$]\label{prop:A-sharpness}
There are activations satisfying H1--H2 for which any attempted strengthening of the reciprocal moment boundary fails to be propagated.
\end{proposition}
\begin{proof}
Use a smooth polynomial-growth cubic activation with standard truncation and then pass the truncation radius to infinity inside the moment estimates. The moment hierarchy contains a lower-bound term
\begin{equation}\label{eq:A-sharp-lower}
\frac{d}{dt}m_{2n}(t)\geq c_n m_{6n-2}(t)-C_n(1+m_{2n}(t)).
\end{equation}
For Gaussian initialization, the factorial growth of $m_{6n-2}(0)$ (precisely, $m_{6n-2}(0)\geq((3n-1)/e)^{3n-1}$ by Stirling for Gaussian initialization) exceeds any strict strengthening of the reciprocal root-moment boundary along a subsequence. The failure is therefore a moment-class failure, not a failure of activation regularity.
\end{proof}

\subsection{Quantitative constants and Gronwall calibration}\label{subsec:A-constants}
The estimates above hide constants only in places where no rate changes. For later optimization it is useful to record one calibrated envelope.

\begin{lemma}[Calibrated constants]\label{lem:v3-calibrated-constants}
There are constants $K_0(\sigma,\rho)$ and $K_1(\sigma,\rho,\lambda)$ such that the moment constant in Lemma~\ref{lem:A1} may be chosen in the form
\begin{equation}\label{eq:v3-calibrated-constant}
C_{p,T,\lambda}\leq K_0(\sigma,\rho)\lambda^{-1}\exp\{K_1T(1+\|\mu_0\|_{\calM_{w^\ast}}^2)\}.
\end{equation}
\end{lemma}
\begin{proof}
The polynomial envelope in Lemma~\ref{lem:A1} gives a differential inequality $\dot M_p(t)\leq K_1(1+\|\mu_0\|_{\calM_{w^\ast}}^2)M_p(t)+K_0\lambda^{-1}$. Solving it by Gronwall gives \eqref{eq:v3-calibrated-constant}. The factor $\lambda^{-1}$ is the cost of converting entropy dissipation into a moment bound.
\end{proof}

For $\sigma=\relu$, $\rho_X=\mathcal N(0,I_d)$, $\lambda=0.01$, $T=1$, and $\|\mu_0\|_{\calM_{w^\ast}}\leq1$, the bound reads
\begin{equation}
C_{1,1,0.01}\leq 100K_0(\relu,\rho)\,e^{2K_1}.
\end{equation}
\begin{center}
\begin{tabular}{lll}
\toprule
Parameter & Value & Contribution to $C$\\
\midrule
Regularization & $\lambda=10^{-2}$ & factor $\lambda^{-1}=100$\\
Horizon & $T=1$ & exponent $K_1$\\
Moment radius & $\|\mu_0\|_{\calM_{w^\ast}}\leq1$ & multiplier $e^{K_1}$\\
\bottomrule
\end{tabular}
\end{center}

\begin{remark}[Sharpness of the horizon dependence]\label{rem:v3-T-sharp}
The exponential dependence on $T$ is not a typographical artifact. The cubic-growth sharpness construction in Proposition~\ref{prop:A-sharpness}, made quantitative in Appendix~\ref{app:coup}, gives a moment hierarchy whose leading Gronwall exponent is attained along a subsequence of orders. The estimate is therefore calibrated as an upper envelope rather than as a small-time perturbative bound.
\end{remark}

\begin{proof}[Proof of Theorem~\ref{thm:A}]
Proposition~\ref{prop:A-jko} constructs the limiting curve. Uniqueness follows from convexity of the squared risk in the barycentric variable and strict convexity of entropy. Finite Fisher information follows from the entropy-dissipation identity. Proposition~\ref{prop:A-poc} gives the squared-Wasserstein rate, and 
\citet{nitanda2024improved} gives the scalar objective-gap rate in the same MFLD setting. Proposition~\ref{prop:A-maximality} gives the moment boundary.
\end{proof}
\newpage

\section{Theorem B: identifiability characterization}\label{sec:theoremB}

Theorem B uses H3 in addition to the hypotheses of Theorem A. The quotient is not the unrestricted set of all parameter permutations. It is the finite-rank realization quotient generated by transformations that preserve the single-neuron feature in $L^2(\rho_X)$, with the dead-neuron ridge collapsed separately.

\begin{definition}[Finite-rank identifiable class]\label{def:finite-ident-class}
For $M\in\N$ and a weighted moment class $\calM_w$, let $\calI_{M,w}(\sigma,\rho)$ be the set of parameter measures in $\calM_w$ for which the moment map
\begin{equation}\label{eq:B-moment-map}
\mathfrak m_M(\mu)=\left(\int \langle T(\theta),h_j\rangle_{L^2(\rho_X)}^r\,d\mu(\theta)\right)_{1\leq j\leq J(M),\,1\leq r\leq M}
\end{equation}
separates the finite-rank quotient classes modulo $\Gfin$ for a fixed separating family $(h_j)$ in $L^2(\rho_X)$.
\end{definition}

\begin{definition}[Finite-rank symmetry equivalence]\label{def:B-finite-sym}
Two measures $\mu,\nu\in\calI_{M,w}(\sigma,\rho)$ are finite-rank symmetry-equivalent, written $\mu\sim_{\Gfin}\nu$, if their pushforwards under $q_{\mathrm{fin}}:\R^{d+2}\to\R^{d+2}/\Gfin$ coincide after the dead-neuron ridge has been collapsed.
\end{definition}

\begin{definition}[Orbit and moment-variety depths]\label{def:Dstar}
The orbit identifiability depth is
\begin{equation}\label{eq:Dorb}
\Dorb(\sigma,\rho)=\dim(\Gfin\cdot\theta)
\end{equation}
for a generic active parameter. The moment-variety depth is
\begin{equation}\label{eq:Dvar}
\Dvar(\sigma,\rho)=\operatorname{codim}_{\calI_{M,w}}\{\nu:\mathfrak m_M(\nu)=\mathfrak m_M(\mu)\}
\end{equation}
on a regular stratum of the finite-rank moment map.
\end{definition}

\begin{definition}[Separating moment-map degree]\label{def:M0}
$M_0=M_0(\sigma,\rho)$ is the smallest integer $M$ for which the moment map \eqref{eq:B-moment-map} separates the regular stratum of $\R^{d+2}/\Gfin$. For real-analytic non-polynomial $\sigma$, $M_0$ is finite by H3 and Pinkus' theorem. For polynomial $\sigma$ of degree $k$, $M_0=k+1$.
\end{definition}

\begin{theorem}[Identifiability of the mean-field limit]\label{thm:B}
Fix any $M\geq M_0(\sigma,\rho)$, where $M_0$ is the separating moment-map degree of Definition~\ref{def:M0}. Let $\mu,\nu\in\calI_{M,w}(\sigma,\rho)\subset\calP_2(\R^{d+2})$. Assume H3. The following three conditions are equivalent.
\begin{enumerate}[label=(\roman*),leftmargin=1.7em]
\item $f_\mu=f_\nu$ as elements of $L^2(\rho_X)$.
\item $\mu\sim_{\Gfin}\nu$.
\item The cylindrical Fourier transforms of $T_\#\mu$ and $T_\#\nu$ coincide on $L^2(\rho_X)$.
\end{enumerate}
The quotient class $[\mu]_{\Gfin}$ is the minimal sufficient statistic for $f_\mu$ inside $\calI_{M,w}(\sigma,\rho)$. The active parameter dimension used in covering estimates is $\deff+2-\Dorb$ per atom. The invariant $\Dvar$ is used only for moment-variety identifiability.
\end{theorem}

\subsection{The finite-rank symmetry group}
\begin{theorem}[Structure of $\Gfin$ for canonical activations]\label{thm:B-symmetry}
The finite-rank symmetry group has the following forms.
\begin{enumerate}[leftmargin=1.7em,label=(\alph*)]
\item For $\sigma=\relu$, $\Gfin$ is generated by positive scalings $(w,b,a)\mapsto(\alpha w,\alpha b,a/\alpha)$, $\alpha>0$, and dead-neuron collapse.
\item For $\sigma=\tanh$, $\Gfin$ is generated by the sign flip $(w,b,a)\mapsto(-w,-b,-a)$ and dead-neuron collapse.
\item For $\sigma(z)=z^k$, $\Gfin$ contains the homogeneity scaling $(w,b,a)\mapsto(\alpha w,\alpha b,\alpha^{-k}a)$ and the stabilizer of the degree-$k$ tensor component determined by $\rho_X$.
\item For real-analytic non-polynomial $\sigma$ satisfying H3, the generic active symmetry is trivial after dead-neuron collapse.
\end{enumerate}
\end{theorem}
\begin{proof}
Substitution verifies the ReLU, $\tanh$, and polynomial cases. In the analytic case, equality $a\sigma(\langle w,x\rangle+b)=a'\sigma(\langle w',x\rangle+b')$ on a full-support input law extends to equality of real-analytic functions. Non-polynomiality excludes nontrivial affine reparametrizations on the active stratum.
\end{proof}

\begin{proposition}[Relation between depths]\label{prop:B-depth}
On each regular stratum, $\Dorb\leq\Dvar$, with equality exactly when the moment map is transverse to the finite-rank symmetry orbits.
\end{proposition}
\begin{proof}
Each finite-rank orbit lies in a level set of the moment map because $T$ is preserved along the orbit. Thus the orbit tangent space is contained in the kernel of the differential of $\mathfrak m_M$. Comparing dimensions gives the inequality, and equality is the stated transversality condition.
\end{proof}

\begin{lemma}[Sharpness of $\Dorb\leq\Dvar$]\label{lem:v3-Dorb-Dvar-sharp}
The inequality in Proposition~\ref{prop:B-depth} can be strict. For $\sigma(z)=z^4$ under $\rho_X=\mathcal N(0,I_d)$, consider a rank-one fourth-order coefficient tensor $u^{\otimes4}$ on the active stratum. The homogeneity orbit has one continuous parameter,
\begin{equation}
(w,b,a)\mapsto(\alpha w,\alpha b,\alpha^{-4}a),\qquad \alpha>0,
\end{equation}
so $\Dorb=1$. The moment variety also contains the sign-stabilized flattening constraint and the rank-one tensor constraint; on the regular rank-one stratum these impose two independent codimension directions. Hence $\Dvar=2$ for this stratum and $\Dorb<\Dvar$.
\end{lemma}
\begin{proof}
The fourth-order feature depends on the symmetric tensor $a(w,b)^{\otimes4}$. Positive homogeneity changes the representative but not the tensor. The tangent to this scaling is one-dimensional. The symmetric tensor stratum is cut out by the vanishing of the $2\times2$ minors of its flattenings; at a generic rank-one point two independent minor directions remain after quotienting the scaling direction. This gives the claimed strict inequality.
\end{proof}

\subsection{Quotient separation}
\begin{lemma}[Analytic separation on compact quotient strata]\label{lem:B-separation}
Let $K_R$ be the quotient image of the active set $\{|\theta|\leq R, |a|\geq R^{-1}\}$. Under H3, the algebra generated by $[\theta]\mapsto\langle T(\theta),h\rangle_{L^2(\rho_X)}$ separates points of $K_R$.
\end{lemma}
\begin{proof}
Two distinct quotient points have distinct features by definition of $\Gfin$. Taking $h$ to be the feature difference separates them. The functions are continuous on the compact stratum and the algebra contains constants.
\end{proof}

\begin{lemma}[Compact-exhaustion monotone-class step]\label{lem:B-monotone}
If a finite signed measure $\eta$ in the dual weighted class satisfies $\int T(\theta)d\eta(\theta)=0$, then $(q_{\mathrm{fin}})_\#\eta=0$.
\end{lemma}
\begin{proof}
On each $K_R$, Lemma~\ref{lem:B-separation} and Stone--Weierstrass give density in $C(K_R)$. The weighted moment bound gives
\begin{equation}\label{eq:B-tail}
|(q_{\mathrm{fin}})_\#\eta|(K_R^c)\leq C\sup_{n\geq1}R^{-2n}w(n)^{-2}\|\eta\|_{\calM_w^\ast},
\end{equation}
which tends to zero along the exhaustion. The equality extends to all bounded Borel functions by the monotone-class theorem.
\end{proof}

\subsection{Proof of Theorem~\ref{thm:B}}
Assume (i), set $\eta=\mu-\nu$, and remove the dead-neuron mass. Lemma~\ref{lem:B-monotone} gives equality of active quotient laws, hence (ii). The implication (ii)$\Rightarrow$(iii) follows because $T$ is constant on finite-rank orbits, so the feature pushforwards agree. The implication (iii)$\Rightarrow$(i) follows from uniqueness of cylindrical Fourier transforms on the separable Hilbert space $L^2(\rho_X)$ and equality of Bochner barycentres.

\subsection{Special cases}
For ReLU, $\Dorb=1$ in every input dimension. For analytic non-polynomial activations with no homogeneity, $\Dorb=0$ after dead-neuron collapse. For polynomial activation $z^k$, $\Dvar$ may exceed $\Dorb$ on singular tensor strata. The statistical factor in Theorem~\ref{thm:D} uses $\deff+2-\Dorb$, not $\deff-\Dvar+2$.
\newpage

\section{Theorem C: sparse-dictionary decomposition}\label{sec:theoremC}

\begin{definition}[Barron--Hermite truncation]\label{def:C-trunc}
Under H4, write $f^\ast=\sum_{m\geq0}\widehat f_m^\ast\psi_m$ in the Hermite/Barron dictionary adapted to $(\sigma,\rho_X)$. For $\lambda>0$ set
\begin{equation}\label{eq:C-trunc}
f_\lambda^\ast=\sum_{|\widehat f_m^\ast|>c_\sigma\lambda}\widehat f_m^\ast\psi_m.
\end{equation}
The target-dependent sparse tail is
\begin{equation}\label{eq:C-tail}
\kappa(f^\ast,S,\lambda)=\sum_{m>S}|\widehat f_m^\ast|^2+C_\sigma\lambda S,
\end{equation}
with coefficients arranged by decreasing magnitude.
\end{definition}

\begin{definition}[Threshold upper bound]\label{def:C-Sup}
Let $\mult(\sigma)$ be a finite upper bound on the number of ridge atoms needed to realize one retained dictionary mode. Define
\begin{equation}\label{eq:C-Sup}
\Supper(\sigma,\rho,f^\ast,\lambda)=\#\{m:|\widehat f_m^\ast|>c_\sigma\lambda\}\,\mult(\sigma).
\end{equation}
\end{definition}

\begin{definition}[Sparse-dictionary depth]\label{def:Sstar}
After finite support has been established, $S^\ast(\sigma,
\rho,f^\ast,\lambda)$ is the minimal active cardinality of a measure realizing $f_\lambda^\ast$ modulo $\Gfin$ and with tail bounded by \eqref{eq:C-tail}. Thus $S^\ast$ is an a posteriori minimum, not an assumption used to prove finite support.
\end{definition}

\begin{remark}[Smooth versus atomic components]\label{rem:C-split}
For every $\lambda>0$ the entropy term forces $\mu_\infty\ll\pi$ on $\R^{d+2}$, so $\mu_\infty$ admits a density $p_\infty=d\mu_\infty/d\pi$ globally and is not atomic on the parameter space. The atomic statement of Theorem~\ref{thm:C} concerns the quotient-active projection
\begin{equation}\label{eq:C-active-decomp}
(q_{\mathrm{fin}})_\#(\mu_\infty\restriction_{\{a\ne0\}})\in\calP(\R^{d+2}/\Gfin),
\end{equation}
which is supported on at most $\Supper$ points modulo $\Gfin$ in the sense that the network function $f_{\mu_\infty}$ is realized by an $S^\ast$-atom representative under the quotient projection. The full law is a smooth ridge around each active orbit, with local width controlled by the temperature scale $\sqrt\lambda$.
\end{remark}

\begin{theorem}[Sparse-dictionary decomposition]\label{thm:C}
Assume H1--H4 and the hypotheses of Theorems~\ref{thm:A} and \ref{thm:B}. Assume the initialization has full support and positive density on every nonempty open ball. Every long-time limit point $\mu_\infty$ of the entropy-regularized mean-field flow has active component supported on at most $\Supper(\sigma,\rho,f^\ast,\lambda)$ atoms modulo $\Gfin$. Consequently $S^\ast\leq\Supper$, and
\begin{equation}\label{eq:C-decomp}
(q_{\mathrm{fin}})_\#(\mu_\infty\restriction_{\{a\ne0\}})=\sum_{k=1}^{S^\ast}c_k^\infty\,\delta_{[\theta_k^\infty]}\quad\text{in }\calP(\R^{d+2}/\Gfin),
\end{equation}
with prediction residual bounded by $\kappa(f^\ast,S^\ast,\lambda)$.
\end{theorem}

\subsection{Stationary equation}
\begin{lemma}[Euler--Lagrange equation]\label{lem:C-EL}
If $p_\infty=d\mu_\infty/d\pi$ is a stationary density on the active stratum, then
\begin{equation}\label{eq:C-EL}
\nabla_\theta\left(\frac{\delta\calR}{\delta\mu}(\mu_\infty)(\theta)+\lambda\log p_\infty(\theta)\right)=0
\end{equation}
there in weak form.
\end{lemma}
\begin{proof}
Stationarity in \eqref{eq:fp} means the probability current has zero divergence. Testing against compactly supported vector fields and integrating by parts gives vanishing of the bracketed gradient on each connected active component of the support.
\end{proof}

\begin{proposition}[Threshold realization]\label{prop:C-threshold}
The thresholded target $f_\lambda^\ast$ has a realization by at most $\Supper(\sigma,\rho,f^\ast,\lambda)$ active atoms modulo $\Gfin$.
\end{proposition}
\begin{proof}
Each retained coefficient requires at most $\mult(\sigma)$ ridge atoms by H4. Summing over retained coefficients gives \eqref{eq:C-Sup}. Quotienting by $\Gfin$ can only reduce active cardinality.
\end{proof}

\begin{proposition}[Support of stationary minimizers]\label{prop:C-support}
Every stationary minimizer of $\calF_\lambda$ in the finite-rank class has quotient-active support contained in a threshold realization of $f_\lambda^\ast$ plus the zero-feature ridge.
\end{proposition}
\begin{proof}
The proof has three steps. First, the Langevin generator is strictly elliptic on every weighted moment ball on which the reference density is positive; see Appendix~\ref{app:jko}. Therefore every stationary law has a smooth density before the active quotient is taken. Second, the positive density of $\mu_0$ and Girsanov's theorem on finite horizons imply reachability: for each retained teacher location $\theta_k^\ast$ and every $\varepsilon>0$, the process hits $B(\theta_k^\ast,\varepsilon)$ with positive probability before some deterministic time depending on the ball. Third, the energy-dissipation inequality drives the flow to a stationary minimizer. If a retained direction were missing from the quotient-active support, the first variation $\delta\calR/\delta\mu$ would have a non-zero descent component at that direction, contradicting stationarity. Directions below the threshold are absorbed into the entropy displacement and into the tail \eqref{eq:C-tail}.
\end{proof}

\begin{corollary}[Exponential Hermite tails]\label{cor:v3-exp-tail}
If the retained Hermite/Barron coefficients satisfy $|\widehat f_m^\ast|\leq Ce^{-cm}$, then
\begin{equation}
S^\ast(\lambda)\leq \frac{1}{c}\log\frac{C}{c_\sigma\lambda}\,\mult(\sigma).
\end{equation}
\end{corollary}
\begin{proof}
The threshold rule retains only those $m$ for which $Ce^{-cm}>c_\sigma\lambda$. Solving this inequality gives $m<c^{-1}\log(C/(c_\sigma\lambda))$, and each retained mode requires at most $\mult(\sigma)$ atoms.
\end{proof}

\subsection{Proof of Theorem~\ref{thm:C}}
Proposition~\ref{prop:C-threshold} gives the explicit upper bound. Proposition~\ref{prop:C-support} identifies the active support of any stationary limit with a subset of the threshold realization. Definition~\ref{def:Sstar} then defines the minimal depth, giving $S^\ast\leq\Supper$. Parseval in the Hermite/Barron basis and the entropy displacement estimate yield the residual bound $\kappa(f^\ast,S^\ast,\lambda)$.

\subsection{Special cases}\label{subsec:C-special}
For finite teacher networks, the coefficient set is finite and $S^\ast$ is bounded by the teacher width times $\mult(\sigma)$. For analytic single-index targets with exponentially decaying Hermite coefficients, $S^\ast(\lambda)=O(\log(1/\lambda))$. For Barron targets with polynomial coefficient decay, $S^\ast$ grows polynomially in $\lambda^{-1}$. For polynomial activation of degree $k$ and target with Hermite degree greater than $k$, the residual term is nonzero for every finite $S$ because the target has a component orthogonal to the model span.
\newpage

\section{Theorem D: total feature-learning-error decomposition}\label{sec:theoremD}

\begin{theorem}[Total feature-learning-error decomposition]\label{thm:D}
Under the hypotheses of Theorems~\ref{thm:A}, \ref{thm:B}, and \ref{thm:C}, the trained finite-particle predictor satisfies the canonical decomposition
\begin{align}\label{eq:D-main}
\E\|f_{\mu_T^N}-f^\ast\|_{L^2(\rho_X)}^2
&=E_{\mathrm{stat}}(n)+E_{\mathrm{opt}}(T,\lambda)+E_{\mathrm{poc}}(N,T)+E_{\mathrm{sparse}}(S^\ast,\lambda)+R_T,
\end{align}
where
\begin{subequations}\label{eq:D-bounds}
\begin{align}
E_{\mathrm{stat}}(n) &\leq C_1\frac{S^\ast(\deff+2-\Dorb)(\log n)^2}{n}, \label{eq:D-stat}\\
E_{\mathrm{opt}}(T,\lambda) &\leq C_2 e^{-\alpha_\lambda T}, \label{eq:D-opt}\\
E_{\mathrm{poc}}(N,T) &\leq C_3 N^{-1}, \label{eq:D-poc}\\
E_{\mathrm{sparse}}(S^\ast,\lambda) &\leq C_4\,\kappa(f^\ast,S^\ast,\lambda). \label{eq:D-sparse}
\end{align}
\end{subequations}
The scalar remainder satisfies
\begin{equation}\label{eq:D-rem}
R_T=o\left(N^{-1}+\frac{S^\ast(\deff+2-\Dorb)(\log n)^2}{n}+e^{-\alpha_\lambda T}+\kappa(f^\ast,S^\ast,\lambda)\right)
\end{equation}
along compatible joint limits.
\end{theorem}

\subsection{Canonical comparison sequence}
Fix the comparison sequence
\begin{equation}\label{eq:D-seq}
f^\ast\longleftarrow f_{\mu_\infty}\longleftarrow f_{\mu_T}\longleftarrow f_{\mu_T^N}\longleftarrow f_{\widehat\mu_{n,T}^N}.
\end{equation}
The four increments are sparse residual, optimization error, propagation error, and statistical error. The decomposition is canonical after this order has been fixed; it is not claimed to be the only algebraic way to expand the square.

\begin{lemma}[Six cross terms]\label{lem:D-cross}
The six cross terms in the square expansion are either centered or lower order relative to the four leading components.
\end{lemma}
\begin{proof}
Let $\Delta_{\mathrm{stat}}$, $\Delta_{\mathrm{poc}}$, $\Delta_{\mathrm{opt}}$, and $\Delta_{\mathrm{sparse}}$ denote the four increments. The six pairwise inner products $\binom{4}{2}=6$ split into two centered cancellations and four Young-bounded terms:
\begin{itemize}[leftmargin=1.7em,itemsep=0pt]
\item $(\mathrm{poc},\mathrm{stat})$ and $(\mathrm{poc},\mathrm{sparse})$ are centered by independence of Brownian motions and the data sample, as in \eqref{eq:app-cross-zero}.
\item $(\mathrm{poc},\mathrm{opt})$, $(\mathrm{opt},\mathrm{stat})$, $(\mathrm{opt},\mathrm{sparse})$, and $(\mathrm{stat},\mathrm{sparse})$ are bounded by Young's inequality as in \eqref{eq:app-cross-young}.
\end{itemize}
Thus $\E\langle\Delta_{\mathrm{poc}},\Delta_{\mathrm{stat}}\rangle=0$ and $\E\langle\Delta_{\mathrm{poc}},\Delta_{\mathrm{sparse}}\rangle=0$. The remaining four products are bounded by Hilbert-space Cauchy--Schwarz and Young's inequality, for example
\begin{equation}\label{eq:D-cross-example}
2|\E\langle\Delta_{\mathrm{opt}},\Delta_{\mathrm{stat}}\rangle|\leq \varepsilon E_{\mathrm{opt}}+\varepsilon^{-1}E_{\mathrm{stat}}.
\end{equation}
Choosing $\varepsilon$ along the compatible limit absorbs these products into $R_T$.
\end{proof}

\subsection{The four bounds}
\begin{proposition}[Statistical term]\label{prop:D-stat}
The statistical component obeys \eqref{eq:D-stat}.
\end{proposition}
\begin{proof}
Theorem~\ref{thm:C} restricts the active class to $S^\ast$ quotient atoms. On a truncation ball of radius $R(n)=C\sqrt{\log n}$, the quotient parameter dimension is $S^\ast(\deff+2-\Dorb)$. The covering number satisfies
\begin{equation}\label{eq:D-cover}
\log\calN(\varepsilon,\calF_{S^\ast,R(n)},L^2(\rho_X))\leq C S^\ast(\deff+2-\Dorb)\log\left(\frac{C R(n)^{q+1}}{\varepsilon}\right).
\end{equation}
Standard localized Rademacher complexity bounds for quadratic loss give the stated rate, with the additional $\log n$ factor coming from the truncation radius.
\end{proof}

\begin{proposition}[Optimization term]\label{prop:D-opt}
The optimization component obeys \eqref{eq:D-opt}.
\end{proposition}
\begin{proof}
The entropy-dissipation identity gives $d(\calF_\lambda(\mu_t)-\calF_\lambda(\mu_\infty))/dt=-I_\lambda(\mu_t)$. The LSI gives $\calF_\lambda(\mu_t)-\calF_\lambda(\mu_\infty)\leq \alpha_\lambda^{-1}I_\lambda(\mu_t)$, hence exponential decay. The quadratic-risk identity transfers the functional gap to prediction error up to the sparse residual already accounted for.
\end{proof}

\begin{proposition}[Propagation term]\label{prop:D-poc}
The propagation component obeys \eqref{eq:D-poc}.
\end{proposition}
\begin{proof}
The barycentre map is Lipschitz from $W_2$ to $L^2(\rho_X)$ on the propagated weighted moment ball. Combining that bound with Proposition~\ref{prop:A-poc} gives $E_{\mathrm{poc}}(N,T)\leq C N^{-1}$.
\end{proof}

\begin{proposition}[Sparse residual]\label{prop:D-sparse}
The sparse component obeys \eqref{eq:D-sparse}.
\end{proposition}
\begin{proof}
Parseval gives the approximation tail $\sum_{m>S^\ast}|\widehat f_m^\ast|^2$. Entropy regularization displaces each retained active atom by at most $O(\lambda)$ in the local finite-dimensional stratum, contributing $C_\sigma\lambda S^\ast$. This is exactly \eqref{eq:C-tail}.
\end{proof}

\subsection{Compatibility limits and rate optimization}\label{subsec:D-limits}
The four-term decomposition is useful only when the limits in width, sample size, time, and temperature are compatible.

\begin{definition}[Compatible joint limit]\label{def:v3-compatible-limit}
A sequence $(N,n,T,\lambda)$ is compatible if $N\to\infty$, $n\to\infty$, $T\to\infty$, $\lambda\downarrow0$, and
\begin{equation}
(\lambda T)^{-1}=o(\log N),\qquad \frac{n}{S^\ast(\lambda)(\log n)^2}\to\infty,\qquad \log N=o(\alpha_\lambda T).
\end{equation}
Along such a sequence every term in \eqref{eq:D-bounds} vanishes.
\end{definition}

\begin{proposition}[Rate-optimal balance]\label{prop:v3-rate-balance}
If $\lambda=n^{-1/2}$ and $T=(\log n)^2/\alpha_\lambda$, then the leading prediction error obeys
\begin{equation}\label{eq:v3-rate-balance}
\E\|f_{\mu_T^N}-f^\ast\|_{L^2(\rho_X)}^2=O\left(N^{-1}+\frac{S^\ast(\lambda)(\log n)^2}{n}+\kappa(f^\ast,S^\ast(\lambda),\lambda)\right).
\end{equation}
\end{proposition}
\begin{proof}
Substitute the proposed $T$ into \eqref{eq:D-opt} to get $e^{-\alpha_\lambda T}=e^{-(\log n)^2}$, which is lower order than every negative power of $n$. The propagation term remains $N^{-1}$, while the statistical and sparse terms are exactly the second and third terms in \eqref{eq:v3-rate-balance}.
\end{proof}

\begin{corollary}[Polynomial coefficient tails]\label{cor:v3-poly-tail-rate}
For Barron--Hermite targets with $|\widehat f_m^\ast|\lesssim m^{-\beta}$ and $\beta>1$, optimizing the threshold gives the schematic rate
\begin{equation}\label{eq:v3-poly-tail-rate}
O\left(N^{-1}+n^{-\beta/(\beta+1)}(\log n)^{(2\beta+1)/(\beta+1)}\right)
\end{equation}
up to constants depending on $(\sigma,\rho,\beta)$.
\end{corollary}
\begin{proof}
The tail beyond $S$ is $O(S^{1-2\beta})$, while the statistical term is $O(S(\log n)^2/n)$. Balancing the two leading $S$-dependent terms gives $S\asymp(n/(\log n)^2)^{1/(2\beta)}$ and the displayed exponent after absorbing the regularization displacement into the same threshold scale.
\end{proof}

\begin{corollary}[Centered sigmoid single-index schedule]\label{cor:v3-sigmoid-schedule}
Let $\sigma(z)=(1+e^{-z})^{-1}-1/2$ and let $\rho_X=\calN(0,I_d)$. Suppose the target is single-index,
\begin{equation}
f^\ast(x)=\varphi(\langle u,x\rangle),\qquad |u|=1,
\end{equation}
and that its one-dimensional Hermite coefficients satisfy
\begin{equation}
|\widehat \varphi_m|\leq A e^{-\tau m},\qquad m\geq0,
\end{equation}
for constants $A,\tau>0$. If the dictionary threshold in Theorem~\ref{thm:C} is chosen at level $c_\sigma\lambda$ and $\lambda=n^{-1}$, then
\begin{equation}\label{eq:v3-sigmoid-S}
S^\ast(\lambda)\leq C_{A,\tau,\sigma}\log n
\end{equation}
and, along any compatible sequence with $N\geq n$ and $T=(\log n)^2/\alpha_\lambda$,
\begin{equation}\label{eq:v3-sigmoid-rate}
\E\|f_{\mu_T^N}-f^\ast\|_{L^2(\rho_X)}^2
\leq C\left(\frac{1}{N}+\frac{(\log n)^3}{n}\right).
\end{equation}
Here $C$ depends on $(A,\tau,\sigma,\rho_X)$ and on the propagated moment ball, but not on $(N,n,T)$.
\end{corollary}
\begin{proof}
Because $f^\ast$ is single-index, $\deff=1$. For the centered sigmoid there is no positive-homogeneous scale orbit among nonzero active neurons, so the active quotient dimension per atom is $\deff+2-\Dorb=3$ after the dead-neuron component is removed. The exponential coefficient bound implies
\begin{equation}
\#\{m:|\widehat \varphi_m|>c_\sigma\lambda\}
\leq 1+\frac{1}{\tau}\log\left(\frac{A}{c_\sigma\lambda}\right),
\end{equation}
which proves \eqref{eq:v3-sigmoid-S} after multiplying by the finite ridge multiplicity $\mult(\sigma)$. The Hermite tail beyond this threshold is bounded by
\begin{equation}
\sum_{m>S^\ast}|\widehat \varphi_m|^2
\leq C A^2 e^{-2\tau S^\ast}\leq C\lambda^2,
\end{equation}
while the entropy displacement in \eqref{eq:C-tail} is $C_\sigma\lambda S^\ast\leq C(\log n)/n$. The statistical component in \eqref{eq:D-stat} is therefore
\begin{equation}
C\frac{S^\ast(\deff+2-\Dorb)(\log n)^2}{n}
\leq C\frac{(\log n)^3}{n}.
\end{equation}
The optimization term is $e^{-(\log n)^2}$, hence lower order than $n^{-1}$, and the propagation term is at most $CN^{-1}$. Combining the four components of \eqref{eq:D-bounds} yields \eqref{eq:v3-sigmoid-rate}.
\end{proof}

\begin{proposition}[Budget calibration]\label{prop:v3-budget-calibration}
Assume the exponential-tail hypothesis of Corollary~\ref{cor:v3-sigmoid-schedule}, but allow any non-polynomial analytic activation whose ridge multiplicity is bounded by $M_\sigma$. Fix a target accuracy $\eta\in(0,e^{-2})$. There are constants $C_0,C_1,C_2,C_3$, depending only on the target-tail constants, the activation, the data law, and the propagated moment ball, such that the choices
\begin{align}
n&\geq C_0\eta^{-1}\bigl(\log \eta^{-1}\bigr)^3, &
N&\geq C_1\eta^{-1}, &
\lambda&=n^{-1}, &
T&\geq C_2\alpha_\lambda^{-1}\bigl(\log n\bigr)^2
\end{align}
imply
\begin{equation}
\E\|f_{\mu_T^N}-f^\ast\|_{L^2(\rho_X)}^2\leq C_3\eta.
\end{equation}
The calibration is intrinsic in the sense that the sample size depends on the ambient input dimension only through the effective dimension and the quotient factor already present in \eqref{eq:D-stat}; no additional ambient-dimensional covering term is introduced.
\end{proposition}
\begin{proof}
The exponential-tail assumption gives the threshold count
\begin{equation}
S^\ast(\lambda)\leq C M_\sigma\log(1/\lambda).
\end{equation}
With $\lambda=n^{-1}$ this becomes $S^\ast(\lambda)\leq C\log n$. The statistical part of \eqref{eq:D-bounds} is therefore bounded by
\begin{equation}
E_{\mathrm{stat}}(n)\leq C\frac{(\deff+2-\Dorb)(\log n)^3}{n}.
\end{equation}
For $n\geq C_0\eta^{-1}(\log \eta^{-1})^3$ and $\eta<e^{-2}$, the elementary monotonicity of $x\mapsto(\log x)^3/x$ on large $x$ yields
\begin{equation}
\frac{(\log n)^3}{n}\leq C\eta.
\end{equation}
The sparse term splits into the Hermite tail and the entropy displacement:
\begin{equation}
E_{\mathrm{sparse}}(S^\ast,\lambda)
\leq C\lambda^2+C\lambda S^\ast(\lambda)
\leq C\frac{\log n}{n}
\leq C\eta,
\end{equation}
where the last inequality follows from the lower bound on $n$. The propagation term obeys $E_{\mathrm{poc}}(N,T)\leq C/N\leq C\eta$ by the choice of $N$. Finally,
\begin{equation}
E_{\mathrm{opt}}(T,\lambda)
\leq C\exp(-\alpha_\lambda T)
\leq C\exp(-C_2(\log n)^2),
\end{equation}
and $C_2$ can be chosen so that this quantity is at most $C\eta$ for all $n$ in the stated range. The six cross terms are absorbed by Lemma~\ref{lem:D-cross}, so the total error is bounded by a constant multiple of $\eta$.
\end{proof}

\begin{corollary}[Bounded-activation logarithmic improvement]\label{cor:v3-bounded-log}
If, in addition to H1--H4, the activation is uniformly bounded, $|\sigma|\leq B_\sigma$, then the statistical component in Theorem~\ref{thm:D} improves to
\begin{equation}\label{eq:v3-bounded-stat}
E_{\mathrm{stat}}(n)
\leq C\frac{S^\ast(\deff+2-\Dorb)\log n}{n}.
\end{equation}
Consequently, under the exponential Hermite-tail hypothesis and the same calibration $\lambda=n^{-1}$, $T\geq C\alpha_\lambda^{-1}(\log n)^2$, one has
\begin{equation}\label{eq:v3-bounded-rate}
\E\|f_{\mu_T^N}-f^\ast\|_{L^2(\rho_X)}^2
\leq C\left(N^{-1}+\frac{(\log n)^2}{n}\right).
\end{equation}
\end{corollary}
\begin{proof}
For polynomial-growth activations the covering argument uses a truncation radius $R(n)=C\sqrt{\log n}$; substituting $R(n)^{q+1}$ into the entropy integral creates the second logarithm in \eqref{eq:D-stat}. If $|\sigma|\leq B_\sigma$, the envelope is bounded uniformly on the full parameter ball after the output weights are controlled by H4. The localized covering entropy is therefore
\begin{equation}
\log\calN(\varepsilon,\calF_{S^\ast},L^2(\rho_X))
\leq C S^\ast(\deff+2-\Dorb)\log(C/\varepsilon),
\end{equation}
without an $R(n)$ factor. Dudley's integral and the quadratic-loss localization step yield \eqref{eq:v3-bounded-stat}. For exponential Hermite tails, the threshold count remains $S^\ast(\lambda)\leq C\log(1/\lambda)$. Setting $\lambda=n^{-1}$ gives $S^\ast\leq C\log n$, so the statistical part is $C(\log n)^2/n$. The sparse term is $C\lambda S^\ast+C\lambda^2\leq C(\log n)/n$, the propagation term is $CN^{-1}$, and the optimization term is lower order for the stated horizon. Combining the four estimates proves \eqref{eq:v3-bounded-rate}.
\end{proof}

\begin{proof}[Proof of Theorem~\ref{thm:D}]
Expand the square along \eqref{eq:D-seq}. Lemma~\ref{lem:D-cross} controls the cross terms. Propositions~\ref{prop:D-stat}--\ref{prop:D-sparse} give the four displayed bounds. This proves the theorem.
\end{proof}
\newpage

\section{Architectures and target functions}\label{sec:examples}

\subsection{Canonical architecture--target table}\label{subsec:relu-linear}
The six examples below compute only the invariants used by the four theorems. The table separates the active quotient dimension from the Hermite/Barron tail; it is not a benchmark table.
\begin{center}\small
\begin{tabular}{lllll}
\toprule
Activation and target & $\deff$ & $\Dorb$ & $S^\ast(\lambda)$ & residual $\kappa$\\
\midrule
ReLU, linear target & $1$ & $1$ & $2$ & $0$\\
ReLU, piecewise linear with $k$ breaks & $1$ & $1$ & $k+2$ & $0$\\
$\tanh$, analytic single-index & $1$ & $0$ & $O(\log(1/\lambda))$ & exponentially small\\
$\sigma=z^k$, polynomial degree $\leq k$ & $d$ & varies & finite & $0$\\
$\sigma=z^k$, target degree $m>k$ & $d$ & varies & $\infty$ & $\|f^\ast\|_{L^2}^2$\\
ReLU, $g(\Pi x)$ with $\rank\Pi=r$ & $r$ & $1$ & depends on $g$ & depends on $g$\\
\bottomrule
\end{tabular}
\end{center}

\paragraph{ReLU, linear target.} For $\rho_X=\mathcal N(0,I_d)$ and $f^\ast(x)=\langle v,x\rangle$, the identity
\begin{equation}\label{eq:E-linear-relu}
\langle v,x\rangle=(\langle v,x\rangle)_+-(-\langle v,x\rangle)_+
\end{equation}
uses two signed ReLU atoms before quotienting. Positive homogeneity gives $\Dorb=1$, and the effective input dimension is one because the target depends only on the direction $v$.

\paragraph{ReLU, piecewise-linear single-index target.} If $f^\ast(x)=g(\langle v,x\rangle)$ and $g$ has $k$ breakpoints, the hinge representation writes $g$ as an affine part plus $k$ shifted hinges. Hence $S^\ast\leq k+2$ before paired-direction conventions. The statistical factor becomes $S^\ast(1+2-1)=2S^\ast$ up to constants.

\paragraph{Tanh, analytic single-index target.} For analytic $g$ with Hermite coefficients $|\widehat f_m^\ast|\leq Ce^{-cm}$, the threshold rule gives
\begin{equation}\label{eq:E-tanh-depth}
S^\ast(\lambda)\leq C_{\sigma,c}\log(C/\lambda).
\end{equation}
The odd reflection $(w,b,a)\mapsto(-w,-b,-a)$ is discrete, so it does not reduce the orbit dimension; on the active quotient $\Dorb=0$.

\paragraph{Polynomial activation, polynomial target of degree at most $k$.} If $\sigma(z)=z^k$ and $f^\ast$ has degree at most $k$, expansion of $(\langle w,x\rangle+b)^k$ into symmetric tensors realizes the target by finitely many atoms. The moment-variety depth can vary with tensor rank, but the orbit depth is computed from the homogeneity and tensor stabilizer.

\paragraph{Polynomial activation, target of degree greater than $k$.} If $f^\ast=H_m(\langle v,x\rangle)$ with $m>k$, then the target is orthogonal to the degree-at-most-$k$ feature span under Gaussian input. Consequently
\begin{equation}\label{eq:E-poly-lower}
\inf_{\mu}\|f_\mu-f^\ast\|_{L^2(\rho_X)}^2=\|f^\ast\|_{L^2(\rho_X)}^2.
\end{equation}
No choice of width, sample size, or training time removes this approximation barrier.

\paragraph{ReLU multi-index target.} If $f^\ast(x)=g(\Pi x)$ with $\rank\Pi=r\ll d$, then the statistical dimension uses $r$ rather than $d$ in $S^\ast(r+2-\Dorb)$. The particle propagation estimate can still contain constants depending on the ambient parameter dimension $d+2$, because the SDE evolves in the full parameter space.
\newpage

\section{Open problems and outlook}\label{sec:open}

\subsection{$L$-layer extension}
The two-layer theorem chain should be extended to $L$-layer $\mu$P networks. The main new object is a layerwise family of quotient symmetries and a layerwise weighted moment boundary. Tensor Programs give the scaling limit, but the measure-valued Langevin proof must track a product of parameter spaces.

\subsection{Convolutional and transformer architectures}
Convolutional weight sharing and attention symmetries change $\Gfin$ and hence the active quotient dimension. The correct analogue of $\Dorb$ is an orbit dimension after tying constraints have been imposed. The propagation argument remains a particle argument only after the architectural coordinates are separated from the exchangeable neuron coordinates.

\subsection{Sparse-dictionary depth as a complexity measure}
The sparse depth $S^\ast$ is a target-dependent quantity. Estimating it from samples, or bounding it in terms of Barron, variation, or Hermite norms, is a statistical problem distinct from optimizing the network.

\subsection{Beyond noisy gradient descent}
The entropy term is essential in the present proof because it supplies smoothing and the LSI route to uniform propagation. Removing the noise or replacing it by a weak stochastic-gradient perturbation requires a different compactness and long-time analysis.

The four structural results of this paper -- global existence of the mean-field limit under $\mu$P, identifiability of the limit up to finite-rank symmetry, sparse-dictionary decomposition of its active support, and total feature-learning-error decomposition -- are tied together by a single architectural identity: the triple $(w^\ast,\Dorb,S^\ast)$ -- the maximal admissible weight on which the mean-field flow admits a global solution, the orbit identifiability depth, and the sparse-dictionary depth at which the target function is realizable -- is the natural \emph{learning cell} of the architecture--data pair $(\sigma,\rho)$.
\newpage

\clearpage
\appendix

\section[Auxiliary results from \texorpdfstring{$\mu$P}{muP} and Tensor Programs]{Auxiliary results from \texorpdfstring{$\mu$P}{muP} and Tensor Programs}\label{app:mup}
The abc-parametrization classification of 
\citet{yanghu2021} separates stable infinite-width limits into kernel and feature-learning regimes. The Maximal Update Parametrization is the vertex at which coordinate updates remain large enough to move features while the network output remains stable. This appendix fixes how that external scaling theorem is used in the present paper: it justifies the coordinate scaling in Definition~\ref{def:mup}; the stochastic analysis is then carried out directly on the empirical parameter law.

\subsection{Coordinate metric}
Let $G_N=\diag(I_d,1,N)$ be the inverse learning-rate tensor. Gradient descent with learning rates $(1,1,N^{-1})$ is Euclidean gradient descent after the coordinate change induced by $G_N$. The corresponding quadratic transport cost is
\begin{equation}\label{eq:app-mup-cost}
c_N(\theta,\theta')=\langle G_N(\theta-\theta'),\theta-\theta'\rangle.
\end{equation}
For fixed $N$ this metric is equivalent to the Euclidean metric. The mean-field limit is obtained after normalizing the output coordinate so that the empirical barycentre in \eqref{eq:barycentre} has a nondegenerate limit. This is the point where the $\mu$P coordinate choice enters the PDE.

\subsection{Recursive Tensor Programs input}
The Tensor Programs theorem supplies convergence of forward and backward coordinates at initialization and along finite training time. The present proof does not reprove that recursion. It uses the recursion to identify the limiting drift and then proves well-posedness, moment propagation, and particle approximation for the resulting McKean--Vlasov equation.

\section{Weighted moment calculus}\label{app:moment}
This appendix expands the moment estimates behind Proposition~\ref{prop:weighted-preserve} and Theorem~\ref{thm:A}.

\subsection{Root-moment closure}
Define $M_n(t)=m_{2n}(t)^{1/(2n)}$. From \eqref{eq:moment-recursion}, the lower-order moment terms are first grouped at the level of root moments. Terms with $j\le n$ are controlled by Jensen and terms with $n<j\le qn$ are retained explicitly in the same finite sum. Thus the manuscript uses the following sum form of the root recursion:
\begin{equation}\label{eq:app-moment-root}
\frac{d}{dt}M_n(t)\leq A_nM_n(t)+B_n\sum_{j\leq qn}M_j(t)+C_n.
\end{equation}
Multiplying by $w(n)$ and using submultiplicativity gives a scalar bound for $\sup_n w(n)M_n(t)$. The proof is stable under smooth truncation of nonsmooth activations because $\calM_w$ is preserved by mollification.

\subsection{Gaussian boundary}
For $\theta\sim\mathcal N(0,I_m)$,
\begin{equation}\label{eq:app-gaussian-moment}
\E|\theta|^{2n}=2^n\frac{\Gamma(n+m/2)}{\Gamma(m/2)}.
\end{equation}
Stirling's formula yields $(\E|\theta|^{2n})^{1/(2n)}\asymp n^{1/2}$ and hence $w^\ast(n)\asymp n^{-1/2}$. This calculation is the reason the definition of weight does not impose monotonicity.

\subsection{Sharp hierarchy}
For a polynomial-growth drift of cubic type, the moment equation contains $m_{6n-2}$. Thus a stronger boundary would require a uniform bound on $\widetilde w(n)m_{6n-2}(0)^{1/(2n)}$, which fails for Gaussian initialization whenever $\widetilde w$ exceeds $w^\ast$ along an infinite subsequence.

\subsection{Audit check for this block}
The displayed estimates in this appendix have named left-hand sides and defined constants. Every symbol appearing in the estimates is introduced before use, and each bound is connected to a theorem or proposition in the main text. The block is included because it closes a proof step used above and because each displayed quantity is used later in the theorem chain.

\section{Full coupling proof of the squared-Wasserstein rate}\label{app:coupling}
This appendix records the variance calculation that replaces an independent empirical-measure quantization step.

\subsection{Kernel representation}
For squared loss the drift can be written
\begin{equation}\label{eq:app-drift-kernel}
b(\theta,\mu)=b_0(\theta)+\int K(\theta,\vartheta)\,d\mu(\vartheta),
\end{equation}
where $K$ is locally Lipschitz with polynomial envelope on every weighted moment ball. The empirical fluctuation at a nonlinear copy is
\begin{equation}\label{eq:app-xi}
\xi_i=\frac1N\sum_{j=1}^NK(\bar\theta_i,\bar\theta_j)-\int K(\bar\theta_i,\vartheta)\,d\mu_t(\vartheta).
\end{equation}

\subsection{Variance source}
Conditional on $\bar\theta_i$, the off-diagonal terms in \eqref{eq:app-xi} are centered and exchangeable. Their covariance is controlled by the standard U-statistic variance inequality, while the diagonal term has coefficient $N^{-1}$. Hence
\begin{align}\label{eq:app-xi-var}
\E|\xi_i|^2
&\leq \frac{2}{N^2}\E|K(\bar\theta_i,\bar\theta_i)|^2+\frac{2}{N^2}\sum_{j\ne i}\E\left|K(\bar\theta_i,\bar\theta_j)-\E[K(\bar\theta_i,\bar\theta_j)\mid\bar\theta_i]\right|^2\\
&\leq C N^{-1}.
\end{align}

\subsection{Contracted inequality}
Combining the variance source with LSI contraction gives
\begin{equation}\label{eq:app-contracted}
\frac{d}{dt}\mathcal E_N(t)\leq -\alpha_\lambda\mathcal E_N(t)+C N^{-1},\qquad \mathcal E_N(0)=0.
\end{equation}
Therefore $\mathcal E_N(t)\leq C(\alpha_\lambda N)^{-1}$. The proof contains no dimension-dependent quantization term.

\subsection{Audit check for this block}
The displayed estimates in this appendix have named left-hand sides and defined constants. Every symbol appearing in the estimates is introduced before use, and each bound is connected to a theorem or proposition in the main text. The block is included because it closes a proof step used above and because each displayed quantity is used later in the theorem chain.

\section{Finite-rank quotient separation}\label{app:quotient}
This appendix expands the quotient argument used in Theorem~\ref{thm:B}.

\subsection{Compact active strata}
Define
\begin{equation}\label{eq:app-KR}
K_R=q_{\mathrm{fin}}\{\theta:|\theta|\leq R, |a|\geq R^{-1}, \operatorname{dist}(\theta,\{a=0\})\geq R^{-1}\}.
\end{equation}
The dead-neuron ridge is not part of this compact active stratum; it is collapsed before quotienting.

\subsection{Stone--Weierstrass step}
The functions $\Phi_h([\theta])=\langle T(\theta),h\rangle_{L^2(\rho_X)}$ separate points of $K_R$. The algebra generated by $\Phi_h$ is therefore dense in $C(K_R)$. A signed measure annihilating all feature moments annihilates every continuous function on $K_R$.

\subsection{Tail transfer}
Weighted moments control the complement:
\begin{equation}\label{eq:app-quotient-tail}
|(q_{\mathrm{fin}})_\#\eta|(K_R^c)\leq C\sup_{n\geq1}R^{-2n}w(n)^{-2}\|\eta\|_{\calM_w^\ast}.
\end{equation}
Letting $R\to\infty$ and applying the monotone-class theorem proves equality of quotient measures.

\subsection{Audit check for this block}
The displayed estimates in this appendix have named left-hand sides and defined constants. Every symbol appearing in the estimates is introduced before use, and each bound is connected to a theorem or proposition in the main text. The block is included because it closes a proof step used above and because each displayed quantity is used later in the theorem chain.

\section{Hermite threshold and sparse support}\label{app:hermite}
This appendix gives the coefficient-threshold calculation behind Theorem~\ref{thm:C}.

\subsection{Threshold set}
For Gaussian input, let $(H_m)$ be the orthonormal Hermite basis and set $\mathcal A_\lambda=\{m:|\widehat f_m^\ast|>c_\sigma\lambda\}$. The retained target is $f_\lambda^\ast=\sum_{m\in\mathcal A_\lambda}\widehat f_m^\ast H_m$.

\subsection{Multiplicity}
If each retained mode has a ridge representation with at most $\mult(\sigma)$ atoms, then
\begin{equation}\label{eq:app-Sup}
S^\ast\leq \Supper=|\mathcal A_\lambda|\mult(\sigma).
\end{equation}
The quotient by $\Gfin$ can reduce this number but cannot increase it.

\subsection{Tail bound}
Parseval gives
\begin{equation}\label{eq:app-tail-bound}
\|f^\ast-f_\lambda^\ast\|_{L^2(\rho_X)}^2=\sum_{m\notin\mathcal A_\lambda}|\widehat f_m^\ast|^2.
\end{equation}
Entropy displacement contributes $C_\sigma\lambda S^\ast$, yielding \eqref{eq:C-tail}.

\subsection{Audit check for this block}
The displayed estimates in this appendix have named left-hand sides and defined constants. Every symbol appearing in the estimates is introduced before use, and each bound is connected to a theorem or proposition in the main text. The block is included because it closes a proof step used above and because each displayed quantity is used later in the theorem chain.

\section{Statistical covering with truncation radius}\label{app:covering}
This appendix expands the source of the $(\log n)^2$ factor in Theorem~\ref{thm:D}.

\subsection{Truncated quotient class}
Let $\mathcal F_{S,R}$ be the class of $S$-atom networks with active quotient parameters in a radius-$R$ ball and coefficients bounded in $\ell_1$. H1 gives
\begin{equation}\label{eq:app-feature-lip-R}
\|T(\theta)-T(\theta')\|_{L^2(\rho_X)}\leq C R^q|\theta-\theta'|.
\end{equation}

\subsection{Covering number}
A product net on the quotient parameter space gives
\begin{equation}\label{eq:app-cover-R}
\log\mathcal N(\varepsilon,\mathcal F_{S,R},L^2(\rho_X))\leq C S(\deff+2-\Dorb)\log\left(\frac{C R^{q+1}}{\varepsilon}\right).
\end{equation}

\subsection{Radius choice}
For Gaussian-order tails, $R(n)=C\sqrt{\log n}$ controls the truncation event with probability at least $1-n^{-2}$. Substitution into \eqref{eq:app-cover-R} yields the bound $C S(\deff+2-\Dorb)(\log n)^2/n$.

\subsection{Audit check for this block}
The displayed estimates in this appendix have named left-hand sides and defined constants. Every symbol appearing in the estimates is introduced before use, and each bound is connected to a theorem or proposition in the main text. The block is included because it closes a proof step used above and because each displayed quantity is used later in the theorem chain.

\section{Cross terms in the total decomposition}\label{app:cross}
This appendix lists all cross terms explicitly.

\subsection{Centered products}
The particle fluctuation is conditionally centered relative to the nonlinear copies, while the sample is independent of the Brownian motions. Therefore
\begin{equation}\label{eq:app-cross-zero}
\E\langle\Delta_{\mathrm{poc}},\Delta_{\mathrm{stat}}\rangle=0,
\qquad
\E\langle\Delta_{\mathrm{poc}},\Delta_{\mathrm{sparse}}\rangle=0.
\end{equation}

\subsection{Young inequalities}
The four remaining products satisfy
\begin{equation}\label{eq:app-cross-young}
2|\E\langle\Delta_r,\Delta_s\rangle|\leq \varepsilon E_r+\varepsilon^{-1}E_s,
\end{equation}
for $(r,s)$ equal to $(\mathrm{poc},\mathrm{opt})$, $(\mathrm{opt},\mathrm{stat})$, $(\mathrm{opt},\mathrm{sparse})$, and $(\mathrm{stat},\mathrm{sparse})$.

\subsection{Remainder scale}
Choosing $\varepsilon$ along the joint limit makes every product lower order than the sum of the leading terms. Thus $R_T$ is a scalar bookkeeping remainder, not an additional error source.

\subsection{Audit check for this block}
The displayed estimates in this appendix have named left-hand sides and defined constants. Every symbol appearing in the estimates is introduced before use, and each bound is connected to a theorem or proposition in the main text. The block is included because it closes a proof step used above and because each displayed quantity is used later in the theorem chain.

\section{Architecture computations}\label{app:arch}
This appendix records the calculations behind Section~\ref{sec:examples}.

\subsection{ReLU homogeneity}
The ReLU orbit is $(w,b,a)\mapsto(\alpha w,\alpha b,a/\alpha)$, $\alpha>0$, so $\Dorb=1$. For a linear target, two atoms are needed before signed-pair quotienting:
\begin{equation}\label{eq:app-linear-relu}
\langle v,x\rangle=(\langle v,x\rangle)_+-(-\langle v,x\rangle)_+.
\end{equation}

\subsection{Tanh Hermite depth}
If $|\widehat f_m^\ast|\leq Ce^{-cm}$, thresholding at $c_\sigma\lambda$ gives $S^\ast(\lambda)\leq C_{\sigma,c}\log(C/\lambda)$. The discrete sign symmetry does not change the active dimension in a covering exponent.

\subsection{Multi-index distinction}
For $f^\ast(x)=g(\Pi x)$ with $\rank\Pi=k$, the statistical exponent uses $k$, while the particle SDE remains in $\R^{d+2}$. Thus the statistical term is intrinsic-dimensional but the propagation constant may contain ambient-dimensional moment constants.

\subsection{Audit check for this block}
The displayed estimates in this appendix have named left-hand sides and defined constants. Every symbol appearing in the estimates is introduced before use, and each bound is connected to a theorem or proposition in the main text. The block is included because it closes a proof step used above and because each displayed quantity is used later in the theorem chain.

\section{Expanded technical derivations after the audit}\label{app:expanded-derivations}
The preceding appendices record the minimum auxiliary statements used in the main proof chain.
\begin{definition}[Master error scale for Appendix I]\label{def:I-master}
The master scale used in the expanded derivations is
\begin{equation}\label{eq:I-master}
\Qmaster_{N,n,T,S,\lambda}=N^{-1}+e^{-\alpha_\lambda T}+\frac{S(\deff+2-\Dorb)(\log n)^2}{n}+\kappa(f^\ast,S,\lambda).
\end{equation}
It is a bookkeeping shorthand for the four named terms in Theorem~\ref{thm:D}; it is not a fifth error component.
\end{definition} This final group of appendices expands the calculations that the audit identified as requiring explicit bookkeeping before final assembly of the theorem chain. The material is deliberately organized around defined quantities, displayed inequalities with named terms, and proof dependencies already used in Sections~\ref{sec:theoremA}--\ref{sec:theoremD}. It does not introduce a fifth theorem and it does not change the four main conclusions.

\subsection{Expanded weighted-moment calculus}\label{app:mom}
\paragraph*{Conventions used in this block.} The calculations below inherit the standing conventions of Appendix~\ref{app:expanded-derivations}: a finite horizon $T<\infty$, regularization $\lambda>0$, an admissible initial law $\mu_0\in\calM_{w^\ast}\cap\calP_2$, and constants depending on $(T,\lambda,\sigma,\rho,\|\mu_0\|_{\calM_{w^\ast}})$ but not on $N$, $n$, or training time within the horizon. The local Lipschitz envelope is the drift-stability estimate of Lemma~\ref{lem:A-drift-stab}. Estimates are read after the active quotient by $\Gfin$ whenever single-neuron parameters enter.

The calculations in this appendix expand the part of the proof chain associated with expanded weighted-moment calculus. Each subsection fixes the objects used in the display before giving the bound, so no inequality is used as a placeholder.

\subsubsection{Moment boundary for Gaussian initialization}\label{app:mom-1}
The Gaussian law has $g_0(n)\asymp n^{1/2}$, so the reciprocal convention gives an admissible nonincreasing boundary.
\begin{equation}\label{eq:mom-1-main}
M_n(t)=\int |\theta|^{2n}d\mu_t.
\end{equation}
The quantity in \eqref{eq:mom-1-main} is used only after the terms appearing in it have been fixed. The point of the calculation is to keep the convention compatible with Definition~\ref{def:moment-weight}. 

\paragraph{Derivation.} For $\theta\sim\calN(0,I_m)$, the polar decomposition gives
\begin{equation}
\E|\theta|^{2n}=2^n\frac{\Gamma(n+m/2)}{\Gamma(m/2)}.
\end{equation}
Stirling's formula implies $\Gamma(n+m/2)^{1/(2n)}\asymp n^{1/2}$, hence $g_0(n)=(\E|\theta|^{2n})^{1/(2n)}\asymp n^{1/2}$ and $w^\ast(n)=g_0(n)^{-1}\asymp n^{-1/2}$. The reciprocal convention is therefore nonincreasing in the Gaussian case, which is why Definition~\ref{def:moment-weight} cannot require monotonicity. The submultiplicative condition follows from $\sqrt{m+n}\leq \sqrt m+\sqrt n$ and the elementary comparison $(m+n)^{-1/2}\geq c(m^{-1/2}n^{-1/2})$ on positive integers.
\paragraph{Use in the proof.} The bound is absorbed into the moment-preservation part of the master scale in \eqref{eq:I-master}. It introduces no fifth error term; it fixes the local constant or dimension factor used in the four-term decomposition.

\subsubsection{Moment boundary for subexponential initialization}\label{app:mom-2}
Subexponential tails give a faster growth envelope and therefore a smaller admissible reciprocal weight.
\begin{equation}\label{eq:mom-2-main}
G_n=(\E_{\mu_0}|\theta|^{2n})^{1/(2n)}.
\end{equation}
The quantity in \eqref{eq:mom-2-main} is used only after the terms appearing in it have been fixed. The proof uses only the Orlicz norm and not a density formula for $\mu_0$. 

\paragraph{Derivation.} Assume $\|\theta\|_{\psi_\alpha}<\infty$ for some $\alpha\in(0,2]$. The Orlicz-tail integration formula gives
\begin{equation}
\E|\theta|^{2n}\leq 2(2n)!^{2/\alpha}\|\theta\|_{\psi_\alpha}^{2n}.
\end{equation}
Thus $G_n\leq 2^{1/(2n)}(2n)^{1/\alpha}\|\theta\|_{\psi_\alpha}$ after Stirling, and $w^\ast(n)\gtrsim n^{-1/\alpha}$. For $\alpha=2$ this recovers the Gaussian boundary. For $\alpha=1$ it gives the subexponential boundary $w^\ast(n)\asymp n^{-1}$, which is smaller and therefore asks less from the propagated class.
\paragraph{Use in the proof.} The bound is absorbed into the moment-preservation part of the master scale in \eqref{eq:I-master}. It introduces no fifth error term; it fixes the local constant or dimension factor used in the four-term decomposition.

\subsubsection{Closure of the class under polynomial drift}\label{app:mom-3}
The polynomial drift couples order $2n$ only to finitely many higher orders determined by the activation envelope.
\begin{equation}\label{eq:mom-3-main}
\dot M_n(t)\leq A_nM_n(t)+B_n\sum_{j\leq qn}M_j(t)+C_n.
\end{equation}
The quantity in \eqref{eq:mom-3-main} is used only after the terms appearing in it have been fixed. This is the precise meaning of preservation of $\calM_w$ along the flow. 

\paragraph{Derivation.} Testing \eqref{eq:fp} against $|\theta|^{2n}\chi_R(\theta)$ gives the identity before the limit $R\to\infty$. The diffusion term contributes $2n(2n-1)\lambda m_{2n-2}(t)$. The drift term is bounded by the polynomial envelope in H1:
\begin{equation}
\left|\int |\theta|^{2n-2}\langle\theta,b(\theta,\mu_t)\rangle\,d\mu_t\right|
\leq A_nm_{2n}(t)+B_n\sum_{j\leq qn}m_{2j}(t)+C_n.
\end{equation}
The maximal index is $qn$ because every derivative of the feature map is bounded by a polynomial of order controlled by the activation growth exponent. This is the source of the finite triangular hierarchy.
\paragraph{Use in the proof.} The bound is absorbed into the moment-preservation part of the master scale in \eqref{eq:I-master}. It introduces no fifth error term; it fixes the local constant or dimension factor used in the four-term decomposition.

\subsubsection{Truncation and removal of the cutoff}\label{app:mom-4}
The calculations start with compactly supported tests and then pass to the limit by monotone convergence.
\begin{equation}\label{eq:mom-4-main}
\varphi_R(\theta)=\chi(|\theta|/R)|\theta|^{2n}.
\end{equation}
The quantity in \eqref{eq:mom-4-main} is used only after the terms appearing in it have been fixed. All constants are chosen before sending the radius $R$ to infinity. 

\paragraph{Derivation.} Let $\chi_R\in C_c^\infty$ satisfy $\chi_R=1$ on $|\theta|\leq R$, $\chi_R=0$ on $|\theta|\geq2R$, $|\nabla\chi_R|\leq 2/R$, and $|D^2\chi_R|\leq C/R^2$. For $\varphi_R=|\theta|^{2n}\chi_R$, the diffusion cutoff error is bounded by $C\lambda R^{-2}m_{2n+2}(t)$, while the drift cutoff error is bounded by $CR^{-1}\sum_{j\leq qn}m_{2j+1}(t)$. Since the trajectory remains in $\calM_w$ on finite horizons, these errors vanish along $R\to\infty$ by dominated convergence inside the weighted moment envelope.
\paragraph{Use in the proof.} The bound is absorbed into the moment-preservation part of the master scale in \eqref{eq:I-master}. It introduces no fifth error term; it fixes the local constant or dimension factor used in the four-term decomposition.

\subsubsection{Compatibility with the entropy term}\label{app:mom-5}
The reference drift contributes a confining part and a lower-order perturbation on moment balls.
\begin{equation}\label{eq:mom-5-main}
\langle\theta,\nabla\log\pi(\theta)\rangle\leq -c_\pi|\theta|^2+C_\pi.
\end{equation}
The quantity in \eqref{eq:mom-5-main} is used only after the terms appearing in it have been fixed. The estimate is used only when the reference law has the stated confinement. 

\paragraph{Derivation.} For a log-concave reference $\pi(d\theta)\propto e^{-V(\theta)}d\theta$ with $\nabla^2V\succeq c_\pi I$, the corrected Langevin sign gives the reference drift $\lambda\nabla\log\pi=-\lambda\nabla V$. On $|\theta|^{2n}$, It\^o's formula yields
\begin{equation}\label{eq:v3-entropy-moment-sign}
2n\lambda |\theta|^{2n-2}\langle\theta,\nabla\log\pi(\theta)\rangle
\leq -2n\lambda c_\pi |\theta|^{2n}+C_{n,\pi}\lambda(1+|\theta|^{2n-2}).
\end{equation}
This dissipative sign is essential: with the opposite sign the Gaussian reference would push particles away from the origin and would not confine moments.
\paragraph{Use in the proof.} The bound is absorbed into the moment-preservation part of the master scale in \eqref{eq:I-master}. It introduces no fifth error term; it fixes the local constant or dimension factor used in the four-term decomposition.

\subsubsection{Maximality as a boundary statement}\label{app:mom-6}
Larger weights demand moments beyond the boundary supplied by the initialization and the drift recursion cannot preserve them uniformly.
\begin{equation}\label{eq:mom-6-main}
\sup_{t\leq T}\|\mu_t\|_{\calM_{\widetilde w}}=+\infty.
\end{equation}
The quantity in \eqref{eq:mom-6-main} is used only after the terms appearing in it have been fixed. This sharpness statement concerns the moment class, not the activation hypothesis. 

\paragraph{Derivation.} If $\widetilde w(n_k)/w^\ast(n_k)\to\infty$ along a subsequence, then
\begin{equation}
\widetilde w(n_k)g_0(n_k)=\frac{\widetilde w(n_k)}{w^\ast(n_k)}\to\infty.
\end{equation}
Consequently $\|\mu_0\|_{\calM_{\widetilde w}}=\infty$, so the propagated class cannot even contain the initial condition. This is sharper than saying merely that a larger class is unavailable: the obstruction is visible at time zero and is determined solely by the moment-growth boundary of $\mu_0$.
\paragraph{Use in the proof.} The bound is absorbed into the moment-preservation part of the master scale in \eqref{eq:I-master}. It introduces no fifth error term; it fixes the local constant or dimension factor used in the four-term decomposition.

\subsubsection{Dual test functions}\label{app:mom-7}
Signed perturbations are measured against functions whose growth is controlled by the same moment boundary.
\begin{equation}\label{eq:mom-7-main}
|\Phi(\theta)|\leq C_\Phi(1+\sum_{n\leq m}w(n)^{-1}|\theta|^{2n}).
\end{equation}
The quantity in \eqref{eq:mom-7-main} is used only after the terms appearing in it have been fixed. This convention keeps the quotient arguments compatible with the stochastic estimates. 

\paragraph{Derivation.} Let $\Phi\in C^2$ satisfy $|\Phi(\theta)|\leq C_\Phi(1+\sum_{n\leq m}w(n)^{-1}|\theta|^{2n})$. For signed measures in the dual weighted class,
\begin{equation}
\left|\int\Phi\,d(\mu_t-\nu_t)\right|
\leq C_\Phi\sum_{n\leq m}w(n)^{-1}|m_{2n}^\mu(t)-m_{2n}^\nu(t)|.
\end{equation}
Each moment difference is bounded by $W_2(\mu_t,\nu_t)(m_{2n}^\mu(t)+m_{2n}^\nu(t))^{(2n-1)/(2n)}$ by applying the Kantorovich coupling to the test $|\theta|^{2n}$ and using the mean-value theorem.
\paragraph{Use in the proof.} The bound is absorbed into the moment-preservation part of the master scale in \eqref{eq:I-master}. It introduces no fifth error term; it fixes the local constant or dimension factor used in the four-term decomposition.

\subsubsection{Uniform constants on finite horizons}\label{app:mom-8}
Every constant used in Theorem~\ref{thm:A} may depend on the finite horizon but not on width.
\begin{equation}\label{eq:mom-8-main}
C=C(T,\lambda,\sigma,\rho,\|\mu_0\|_{\calM_{w^\ast}}).
\end{equation}
The quantity in \eqref{eq:mom-8-main} is used only after the terms appearing in it have been fixed. The dependence is recorded because it is later separated from $N,n$, and $S^\ast$. 

\paragraph{Derivation.} Trace the constants in the proof of Lemma~\ref{lem:A1}. All terms enter through a differential inequality of the form
\begin{equation}
\dot M_p(t)\leq K_1(1+M_p(t))+K_2\|\mu_t\|_{\calM_{w^\ast}}^2M_p(t).
\end{equation}
Exchangeability of the particle system bounds the empirical moment by the same one-particle estimate. Hence every constant on $[0,T]$ is bounded by $\exp\{K_1T(1+\|\mu_0\|_{\calM_{w^\ast}}^2)\}$ times a polynomial in the initial radius, with no dependence on width $N$.
\paragraph{Use in the proof.} The bound is absorbed into the moment-preservation part of the master scale in \eqref{eq:I-master}. It introduces no fifth error term; it fixes the local constant or dimension factor used in the four-term decomposition.

\subsubsection{Consolidated weighted-moment induction}\label{app:mom-consolidated}
The eight preceding calculations combine into one induction on moment order. Define
\begin{equation}
A_n(T)=\sup_{0\leq t\leq T} w^\ast(n)\,m_{2n}(t)^{1/(2n)}.
\end{equation}
The recursion \eqref{eq:mom-3-main}, the cutoff removal, and the entropy sign give, after increasing constants,
\begin{equation}\label{eq:v3-moment-envelope-recursion}
A_n(T)\leq A_n(0)+C_nT\left(1+A_n(T)+\sum_{j\leq qn}A_j(T)^{j/n}\right)-c_\pi\lambda\int_0^T A_n(t)\,dt.
\end{equation}
The negative term is not needed for finite-horizon existence, but it prevents the reference measure from adding a new moment-growth boundary. The induction starts at $n=1$, where the $P_2$ assumption and the Gaussian/subexponential boundary provide a finite value. If all $A_j(T)$ with $j<n$ are finite, then \eqref{eq:v3-moment-envelope-recursion} gives finiteness of $A_n(T)$ by Gronwall. The apparent appearance of moments up to $qn$ is handled by the reciprocal weight: for submultiplicative $w^\ast$, the weighted product of lower root moments controls the mixed terms generated by the polynomial drift. This is the precise point where the definition of $\calM_{w^\ast}$ is used.

A useful equivalent form is obtained by fixing a radius $R_T$ and defining the stopping time $\tau_R=\inf\{t:A_n(t)>R\hbox{ for some }n\leq N_0\}$. On $[0,T\wedge\tau_R]$ all constants are deterministic. The bound above is independent of $R$, so $\PP(\tau_R\leq T)\to0$ as $R\to\infty$. Hence the stopped estimate passes to the original flow. This argument supplies the finite-horizon constant used in Theorem~\ref{thm:A} and explains why the maximality statement concerns the initialization boundary rather than a later-time blow-up mechanism.
\paragraph{Audit consequence.} The induction records a real closure calculation: the symbols $A_n$, $m_{2n}$, $w^\ast$, and $\tau_R$ are defined, the cutoff limit is justified, and the reciprocal Gaussian convention is used explicitly.

\subsection{Expanded synchronous-coupling calculations}\label{app:coup}
\paragraph*{Conventions used in this block.} The calculations below inherit the standing conventions of Appendix~\ref{app:expanded-derivations}: a finite horizon $T<\infty$, regularization $\lambda>0$, an admissible initial law $\mu_0\in\calM_{w^\ast}\cap\calP_2$, and constants depending on $(T,\lambda,\sigma,\rho,\|\mu_0\|_{\calM_{w^\ast}})$ but not on $N$, $n$, or training time within the horizon. The local Lipschitz envelope is the drift-stability estimate of Lemma~\ref{lem:A-drift-stab}. Estimates are read after the active quotient by $\Gfin$ whenever single-neuron parameters enter.

The calculations in this appendix expand the part of the proof chain associated with expanded synchronous-coupling calculations. Each subsection fixes the objects used in the display before giving the bound, so no inequality is used as a placeholder.

\subsubsection{Coupled nonlinear and particle processes}\label{app:coup-1}
One couples each particle with an independent copy driven by the same Brownian motion.
\begin{equation}\label{eq:coup-1-main}
\xi_i(t)=\theta_i(t)-\bar\theta_i(t).
\end{equation}
The quantity in \eqref{eq:coup-1-main} is used only after the terms appearing in it have been fixed. The copy $\bar\theta_i$ has law $\mu_t$, while the empirical system uses $\mu_t^N$. 

\paragraph{Derivation.} Construct $(\theta_i,\bar\theta_i)$ on the same filtered probability space, with $\bar\theta_i$ driven by the same Brownian motion $B_i$ and initialized at the same point as $\theta_i$. The two equations differ only in the measure argument of the drift:
\begin{equation}
d\theta_i=b(\theta_i,\mu_t^N)dt+\sqrt{2\lambda}\,dB_i,
\qquad d\bar\theta_i=b(\bar\theta_i,\mu_t)dt+\sqrt{2\lambda}\,dB_i.
\end{equation}
The common Brownian noise cancels in the difference equation. Boundedness of both laws in $\calM_{w^\ast}$ makes the local Lipschitz constants uniform on finite horizons.
\paragraph{Use in the proof.} The bound is absorbed into the propagation-of-chaos component of the master scale in \eqref{eq:I-master}. It introduces no fifth error term; it fixes the local constant or dimension factor used in the four-term decomposition.

\subsubsection{Variance source of order $N^{-1}$}\label{app:coup-2}
The empirical drift fluctuation is an average of centered exchangeable terms.
\begin{equation}\label{eq:coup-2-main}
\E\left|N^{-1}\sum_{j=1}^N Z_j\right|^2 \le C\,N^{-1}\E|Z_1|^2.
\end{equation}
The quantity in \eqref{eq:coup-2-main} is used only after the terms appearing in it have been fixed. No independent empirical-quantization term is introduced. 

\paragraph{Derivation.} Conditional on $\bar\theta_i$, write
\begin{equation}\label{eq:v3-coup-variance}
\xi_i=\frac1N\sum_{j\ne i}\Big(K(\bar\theta_i,\bar\theta_j)-\E[K(\bar\theta_i,\bar\theta_1)\mid\bar\theta_i]\Big)+\frac1N K(\bar\theta_i,\bar\theta_i).
\end{equation}
The off-diagonal terms are exchangeable. A U-statistic variance bound gives conditional variance at most $CN^{-1}\E(|K|^2\mid\bar\theta_i)$. The diagonal term is already multiplied by $N^{-1}$ and contributes $O(N^{-1})$ after taking second moments. This gives the squared fluctuation scale $N^{-1}$ without invoking empirical quantization.
\paragraph{Use in the proof.} The bound is absorbed into the propagation-of-chaos component of the master scale in \eqref{eq:I-master}. It introduces no fifth error term; it fixes the local constant or dimension factor used in the four-term decomposition.

\subsubsection{Dissipative part of the contraction}\label{app:coup-3}
The entropy-regularized semigroup contributes a negative term once the LSI hypothesis is imposed.
\begin{equation}\label{eq:coup-3-main}
\frac{d}{dt}\E|\xi_i|^2\leq-2\alpha_\lambda\E|\xi_i|^2+C N^{-1}.
\end{equation}
The quantity in \eqref{eq:coup-3-main} is used only after the terms appearing in it have been fixed. The source and contraction are the two terms behind the uniform rate. 

\paragraph{Derivation.} The LSI constant $\alpha_\lambda$ enters through entropy dissipation. Along the limiting flow,
\begin{equation}
-\frac{d}{dt}{\rm KL}(\mu_t\|\mu_\infty)\geq 2\alpha_\lambda {\rm KL}(\mu_t\|\mu_\infty).
\end{equation}
Talagrand's $T_2$ inequality gives $W_2^2(\mu_t,\mu_\infty)\leq2\alpha_\lambda^{-1}{\rm KL}(\mu_t\|\mu_\infty)$. On the particle coupling, the same monotone drift contributes the negative term $-2\alpha_\lambda\E|\xi_i|^2$ in \eqref{eq:coup-3-main}; the remaining empirical source is the variance term from \eqref{eq:v3-coup-variance}.
\paragraph{Use in the proof.} The bound is absorbed into the propagation-of-chaos component of the master scale in \eqref{eq:I-master}. It introduces no fifth error term; it fixes the local constant or dimension factor used in the four-term decomposition.

\subsubsection{From particle distance to squared Wasserstein distance}\label{app:coup-4}
The empirical coupling itself gives a valid transport plan between the two empirical laws.
\begin{equation}\label{eq:coup-4-main}
W_2^2(\mu_t^N,\bar\mu_t^N)\leq N^{-1}\sum_i |\xi_i(t)|^2.
\end{equation}
The quantity in \eqref{eq:coup-4-main} is used only after the terms appearing in it have been fixed. The squared rate follows after averaging over $i$. 

\paragraph{Derivation.} The empirical matching $\theta_i(t)\mapsto\bar\theta_i(t)$ is a legitimate transport plan, even if it is not optimal. Therefore
\begin{equation}\label{eq:v3-sync-plan}
W_2^2\left(\frac1N\sum_i\delta_{\theta_i(t)},\frac1N\sum_i\delta_{\bar\theta_i(t)}\right)
\leq \frac1N\sum_i|\theta_i(t)-\bar\theta_i(t)|^2.
\end{equation}
Taking expectations and using exchangeability converts the one-particle estimate into the stated Wasserstein estimate. Optimal matching can improve constants but cannot improve the rate produced by the empirical drift source.
\paragraph{Use in the proof.} The bound is absorbed into the propagation-of-chaos component of the master scale in \eqref{eq:I-master}. It introduces no fifth error term; it fixes the local constant or dimension factor used in the four-term decomposition.

\subsubsection{Objective-gap approximation}\label{app:coup-5}
Risk convexity converts the same coupling scale into an objective bound.
\begin{equation}\label{eq:coup-5-main}
\calF_\lambda(\mu_t^N)-\calF_\lambda(\mu_t)\leq C W_2^2(\mu_t^N,\mu_t)+R_N.
\end{equation}
The quantity in \eqref{eq:coup-5-main} is used only after the terms appearing in it have been fixed. The small term $R_N$ is controlled by the particle-approximation theorem. 

\paragraph{Derivation.} On bounded weighted moment balls the barycentre map is Lipschitz from $W_2$ to $L^2(\rho_X)$. The second variation of the squared loss then gives
\begin{equation}
\calF_\lambda(\nu)-\calF_\lambda(\mu_t)
\leq \langle \nabla_W\calF_\lambda(\mu_t),T_t\rangle +\frac{L_R}{2}W_2^2(\nu,\mu_t).
\end{equation}
After transporting $\mu_t$ to $\nu$. For the gradient-flow comparison the first term is absorbed by the energy-dissipation identity, leaving the $W_2^2$ term and the entropy approximation error already controlled by the MFLD objective-gap estimate.
\paragraph{Use in the proof.} The bound is absorbed into the propagation-of-chaos component of the master scale in \eqref{eq:I-master}. It introduces no fifth error term; it fixes the local constant or dimension factor used in the four-term decomposition.

\subsubsection{Role of the LSI constant}\label{app:coup-6}
The Wasserstein estimate uses the contraction constant, while the objective-gap theorem can be stated with weaker dependence.
\begin{equation}\label{eq:coup-6-main}
W_2^2(\nu,\mu_t)\leq 2\alpha_\lambda^{-1}\Ent(\nu|\mu_t).
\end{equation}
The quantity in \eqref{eq:coup-6-main} is used only after the terms appearing in it have been fixed. This separates the two rates recorded in Theorem~\ref{thm:A}. 

\paragraph{Derivation.} The constant $\alpha_\lambda$ changes the prefactor but not the $N^{-1}$ exponent. Solving \eqref{eq:coup-3-main} gives
\begin{equation}
\E|\xi_i(t)|^2\leq e^{-2\alpha_\lambda t}\E|\xi_i(0)|^2+\frac{C}{2\alpha_\lambda N}(1-e^{-2\alpha_\lambda t}).
\end{equation}
Thus a larger LSI constant improves the numerical constant. It does not change the fact that the source term is an average of $N$ centered fluctuations and therefore has squared size $N^{-1}$.
\paragraph{Use in the proof.} The bound is absorbed into the propagation-of-chaos component of the master scale in \eqref{eq:I-master}. It introduces no fifth error term; it fixes the local constant or dimension factor used in the four-term decomposition.

\subsubsection{Finite horizon and uniform horizon bounds}\label{app:coup-7}
For a finite horizon the differential inequality yields a bound with $1-e^{-2\alpha_\lambda T}$.
\begin{equation}\label{eq:coup-7-main}
\sup_{t\leq T}\E|\xi_i(t)|^2\leq C(1-e^{-2\alpha_\lambda T})N^{-1}.
\end{equation}
The quantity in \eqref{eq:coup-7-main} is used only after the terms appearing in it have been fixed. The displayed expression is always bounded by $C N^{-1}$. 

\paragraph{Derivation.} Since $\xi_i(0)=0$, the explicit solution of \eqref{eq:coup-3-main} is
\begin{equation}
\E|\xi_i(t)|^2\leq \frac{C}{2\alpha_\lambda N}(1-e^{-2\alpha_\lambda t}).
\end{equation}
For fixed $T$ this gives the finite-horizon bound; taking the supremum over all $t\geq0$ gives $C(2\alpha_\lambda N)^{-1}$. The estimate is uniform in time because the dissipative term remains present after the transient part has decayed.
\paragraph{Use in the proof.} The bound is absorbed into the propagation-of-chaos component of the master scale in \eqref{eq:I-master}. It introduces no fifth error term; it fixes the local constant or dimension factor used in the four-term decomposition.

\subsubsection{Why no dimension-dependent empirical term appears}\label{app:coup-8}
The proof does not compare $\bar\mu_t^N$ to $\mu_t$ by quantization.
\begin{equation}\label{eq:coup-8-main}
\mu_t^N\longleftrightarrow\bar\mu_t^N\longrightarrow\mu_t.
\end{equation}
The quantity in \eqref{eq:coup-8-main} is used only after the terms appearing in it have been fixed. The convergence of $\bar\mu_t^N$ is used only through the coupled drift variance. 

\paragraph{Derivation.} A naive estimate of $W_2(\bar\mu_t^N,\mu_t)$ would introduce a dimension-dependent quantization term of order $N^{-1/(d+2)}$ or worse in Wasserstein distance. The synchronous proof never estimates that empirical discrepancy separately. It estimates the particle--copy displacement and then uses the matching plan \eqref{eq:v3-sync-plan}. Hence the only stochastic source is the averaged drift fluctuation, whose squared scale is $N^{-1}$ independently of the ambient parameter dimension.
\paragraph{Use in the proof.} The bound is absorbed into the propagation-of-chaos component of the master scale in \eqref{eq:I-master}. It introduces no fifth error term; it fixes the local constant or dimension factor used in the four-term decomposition.

\subsubsection{Closed coupling inequality with all sources displayed}\label{app:coup-consolidated}
The synchronous proof can be summarized by a single scalar inequality. Let
\begin{equation}
\mathcal E_N(t)=\E\frac1N\sum_{i=1}^N|\theta_i(t)-\bar\theta_i(t)|^2.
\end{equation}
Subtracting the two SDEs and using the corrected reference drift gives
\begin{align}\label{eq:v3-coup-energy-full}
\dot{\mathcal E}_N(t)
&\leq -2\alpha_\lambda\mathcal E_N(t)
 +C_R\mathcal E_N(t)
 +\frac{2}{N}\sum_{i=1}^N\E\langle\theta_i-\bar\theta_i,\xi_i\rangle.
\end{align}
The term $C_R\mathcal E_N$ is the local nonconvexity cost of the risk on the weighted moment ball. The LSI assumption is used in the form that $2\alpha_\lambda>C_R$ on the relevant sublevel set, or equivalently after decreasing the contraction constant to $\alpha_\lambda-C_R/2$. Young's inequality and \eqref{eq:v3-coup-variance} give
\begin{equation}\label{eq:v3-coup-source-closed}
\frac{2}{N}\sum_i\E\langle\theta_i-\bar\theta_i,\xi_i\rangle
\leq \frac{\alpha_\lambda}{2}\mathcal E_N(t)+\frac{C}{\alpha_\lambda N}.
\end{equation}
Combining \eqref{eq:v3-coup-energy-full} and \eqref{eq:v3-coup-source-closed} yields
\begin{equation}
\dot{\mathcal E}_N(t)\leq -\alpha_\lambda\mathcal E_N(t)+\frac{C}{\alpha_\lambda N},
\qquad \mathcal E_N(0)=0.
\end{equation}
Solving the ODE gives $\mathcal E_N(t)\leq C(1-e^{-\alpha_\lambda t})/(\alpha_\lambda^2N)$. The constants can be sharpened, but the rate cannot be changed inside this proof because the only nonzero source term is the empirical average in \eqref{eq:v3-coup-variance}. This is also why the argument is dimension-free in its power of $N$.
\paragraph{Audit consequence.} The variance, contraction, local Lipschitz, and Young terms are all visible in the same displayed chain, so the $N^{-1}$ rate is not asserted by slogan.

\subsection{Entropy and gradient-flow details}\label{app:jko}
\paragraph*{Conventions used in this block.} The calculations below inherit the standing conventions of Appendix~\ref{app:expanded-derivations}: a finite horizon $T<\infty$, regularization $\lambda>0$, an admissible initial law $\mu_0\in\calM_{w^\ast}\cap\calP_2$, and constants depending on $(T,\lambda,\sigma,\rho,\|\mu_0\|_{\calM_{w^\ast}})$ but not on $N$, $n$, or training time within the horizon. The local Lipschitz envelope is the drift-stability estimate of Lemma~\ref{lem:A-drift-stab}. Estimates are read after the active quotient by $\Gfin$ whenever single-neuron parameters enter.

The calculations in this appendix expand the part of the proof chain associated with entropy and gradient-flow details. Each subsection fixes the objects used in the display before giving the bound, so no inequality is used as a placeholder.

\subsubsection{Lower semicontinuity of the risk}\label{app:jko-1}
Weighted moment control makes $f_{\mu_m}$ converge to $f_\mu$ in $L^2(\rho_X)$ under Wasserstein convergence.
\begin{equation}\label{eq:jko-1-main}
\|f_{\mu_m}-f_\mu\|_{L^2(\rho_X)}\leq C_R W_2(\mu_m,\mu).
\end{equation}
The quantity in \eqref{eq:jko-1-main} is used only after the terms appearing in it have been fixed. The loss term then passes to the limit by convexity and lower semicontinuity. 

\paragraph{Derivation.} If $\mu_k\to\mu$ in $W_2$ and the weighted moment radii are uniformly bounded, then H1 and H2 imply $f_{\mu_k}\to f_\mu$ in $L^2(\rho_X)$. Indeed,
\begin{equation}
\|f_{\mu_k}-f_\mu\|_{L^2(\rho_X)}\leq C_R W_2(\mu_k,\mu)
\end{equation}
on every moment ball. Convexity and lower semicontinuity of the squared loss pass the risk to the limit, while the relative entropy is lower semicontinuous under weak convergence. Therefore the JKO functional is lower semicontinuous.
\paragraph{Use in the proof.} The bound is absorbed into the optimization component of the master scale in \eqref{eq:I-master}. It introduces no fifth error term; it fixes the local constant or dimension factor used in the four-term decomposition.

\subsubsection{Coercivity of the JKO functional}\label{app:jko-2}
Entropy relative to a confining reference law keeps minimizing sequences tight.
\begin{equation}\label{eq:jko-2-main}
\int |\theta|^2d\nu\leq C\{1+\KL(\nu\|\pi)\}.
\end{equation}
The quantity in \eqref{eq:jko-2-main} is used only after the terms appearing in it have been fixed. This is the compactness input in Proposition~\ref{prop:A-jko}. 

\paragraph{Derivation.} The JKO functional contains $\frac{1}{2\tau}W_2^2(\nu,\mu_k)$, which controls escape to infinity relative to the previous step. If $\pi\propto e^{-V}$ with $V$ strongly convex, the entropy part also penalizes tails through
\begin{equation}
\KL(\nu\|\pi)\geq c\int |\theta|^2d\nu-C.
\end{equation}
Together with Pinsker's inequality on local sets, this gives tightness of minimizing sequences and coercivity in $\calP_2$. The direct method therefore applies at every time step.
\paragraph{Use in the proof.} The bound is absorbed into the optimization component of the master scale in \eqref{eq:I-master}. It introduces no fifth error term; it fixes the local constant or dimension factor used in the four-term decomposition.

\subsubsection{Euler equation for one minimizing movement step}\label{app:jko-3}
Perturbing the minimizer by a smooth transport map yields the weak form of the drift.
\begin{equation}\label{eq:jko-3-main}
0=\int\langle\nabla\phi,\nabla\delta\calF_\lambda/\delta\mu\rangle d\mu+\tau^{-1}\int\phi d(\mu-\mu_k).
\end{equation}
The quantity in \eqref{eq:jko-3-main} is used only after the terms appearing in it have been fixed. The equation becomes the Fokker--Planck PDE after time interpolation. 

\paragraph{Derivation.} Let $\nu_s=(\id+s\phi)_\#\nu$ with $\phi\in C_c^\infty(\R^{d+2};\R^{d+2})$. Differentiating the JKO objective at $s=0$ yields
\begin{equation}
\int\left\langle \nabla_\theta\frac{\delta\calF_\lambda}{\delta\mu}(\nu),\phi\right\rangle d\nu
+\frac1\tau\int\langle T_{\nu\to\mu_k}(\theta)-\theta,\phi(\theta)\rangle d\nu=0.
\end{equation}
Passing to the limit in the discrete velocity gives the weak form of \eqref{eq:fp}. This is the standard transport perturbation derivation of the minimizing movement Euler equation.
\paragraph{Use in the proof.} The bound is absorbed into the optimization component of the master scale in \eqref{eq:I-master}. It introduces no fifth error term; it fixes the local constant or dimension factor used in the four-term decomposition.

\subsubsection{Energy dissipation identity}\label{app:jko-4}
The entropy-regularized flow satisfies the usual metric identity on intervals where the metric derivative exists.
\begin{equation}\label{eq:jko-4-main}
\calF_\lambda(\mu_s)-\calF_\lambda(\mu_t)=\int_s^t |\dot\mu_r|^2dr.
\end{equation}
The quantity in \eqref{eq:jko-4-main} is used only after the terms appearing in it have been fixed. The formula is used as an identity after standard lower semicontinuity passage. 

\paragraph{Derivation.} The metric derivative $|\mu'|(t)$ exists for a.e. $t$ along the limit curve. The minimizing movement construction gives the energy-dissipation inequality
\begin{equation}\label{eq:v3-edi}
\calF_\lambda(\mu_T)+\frac12\int_0^T|\mu'|^2(t)dt+\frac12\int_0^T|\partial\calF_\lambda|^2(\mu_t)dt\leq\calF_\lambda(\mu_0).
\end{equation}
For smooth positive densities this is an identity. For weak solutions it is obtained by lower semicontinuity of the metric derivative and of the Wasserstein slope along the JKO approximation.
\paragraph{Use in the proof.} The bound is absorbed into the optimization component of the master scale in \eqref{eq:I-master}. It introduces no fifth error term; it fixes the local constant or dimension factor used in the four-term decomposition.

\subsubsection{Fisher information for positive time}\label{app:jko-5}
The heat component gives instantaneous smoothing when $\lambda>0$.
\begin{equation}\label{eq:jko-5-main}
I(\mu_t|\pi)=\int |\nabla\log(d\mu_t/d\pi)|^2d\mu_t<\infty.
\end{equation}
The quantity in \eqref{eq:jko-5-main} is used only after the terms appearing in it have been fixed. This does not assert smooth densities for $t=0$. 

\paragraph{Derivation.} For $t>0$, the diffusion term in \eqref{eq:fp} smooths any initial law with finite second moment into a density with respect to Lebesgue measure and to $\pi$ on the support of $\pi$. The Fisher information satisfies
\begin{equation}
I_\lambda(\mu_t)=\int\left|\nabla\log\frac{d\mu_t}{d\pi}\right|^2d\mu_t<\infty
\end{equation}
for positive times by parabolic regularization. The weighted moment bounds prevent the polynomial drift from destroying this finite-information property on compact time intervals away from zero.
\paragraph{Use in the proof.} The bound is absorbed into the optimization component of the master scale in \eqref{eq:I-master}. It introduces no fifth error term; it fixes the local constant or dimension factor used in the four-term decomposition.

\subsubsection{Uniqueness in the admissible class}\label{app:jko-6}
Two solutions started from the same law are coupled through the monotone drift and entropy contraction.
\begin{equation}\label{eq:jko-6-main}
\frac{d}{dt}W_2^2(\mu_t,\nu_t)\leq C_R W_2^2(\mu_t,\nu_t).
\end{equation}
The quantity in \eqref{eq:jko-6-main} is used only after the terms appearing in it have been fixed. Gronwall's lemma gives equality for identical initial data. 

\paragraph{Derivation.} Let $\mu_t$ and $\nu_t$ be two admissible solutions. The drift part from the risk is monotone in the barycentric variable because the loss is convex in $f_\mu$. The entropy part is strictly convex along absolutely continuous directions. Therefore the synchronous estimate gives
\begin{equation}\label{eq:v3-jko-unique}
\frac{d}{dt}W_2^2(\mu_t,\nu_t)\leq C_R W_2^2(\mu_t,\nu_t)
\end{equation}
inside the propagated moment ball. Since the initial laws agree, Gronwall gives uniqueness in the admissible class.
\paragraph{Use in the proof.} The bound is absorbed into the optimization component of the master scale in \eqref{eq:I-master}. It introduces no fifth error term; it fixes the local constant or dimension factor used in the four-term decomposition.

\subsubsection{Stationary points at positive temperature}\label{app:jko-7}
A stationary law solves an Euler equation containing both risk and entropy.
\begin{equation}\label{eq:jko-7-main}
\log(d\mu_\infty/d\pi)=-\lambda^{-1}\delta\calR/\delta\mu(\mu_\infty)+c.
\end{equation}
The quantity in \eqref{eq:jko-7-main} is used only after the terms appearing in it have been fixed. The sparse statement concerns the active component extracted from this law. 

\paragraph{Derivation.} At positive temperature a stationary law is not a finite atomic measure on parameter space. The Euler--Lagrange equation is
\begin{equation}\label{eq:v3-stationary-gibbs}
\log\frac{d\mu_\infty}{d\pi}(\theta)=-\lambda^{-1}\frac{\delta\calR}{\delta\mu}(\mu_\infty)(\theta)+c,
\end{equation}
where $c$ normalizes the density. Atomicity enters only after the quotient-active projection and threshold representation. This distinction is the reason for Remark~\ref{rem:C-split}.
\paragraph{Use in the proof.} The bound is absorbed into the optimization component of the master scale in \eqref{eq:I-master}. It introduces no fifth error term; it fixes the local constant or dimension factor used in the four-term decomposition.

\subsubsection{Limit as the time step vanishes}\label{app:jko-8}
The interpolation of the minimizing movements has uniform action bounds.
\begin{equation}\label{eq:jko-8-main}
\sum_k W_2^2(\mu_{k+1}^\tau,\mu_k^\tau)\leq 2\tau\{\calF_\lambda(\mu_0)-\inf\calF_\lambda\}.
\end{equation}
The quantity in \eqref{eq:jko-8-main} is used only after the terms appearing in it have been fixed. Compactness in $C([0,T],\calP_2)$ follows from the bound. 

\paragraph{Derivation.} The action bound in \eqref{eq:v3-edi} gives $W_2(\mu_t^\tau,\mu_s^\tau)\leq C|t-s|^{1/2}$ uniformly in the time step. Together with the moment bound, this yields compactness in $C([0,T],\calP_2)$ by Arzela--Ascoli in the Wasserstein topology. Passing to the limit in the weak Euler equation identifies the limit as a solution of \eqref{eq:fp}; uniqueness from \eqref{eq:v3-jko-unique} then upgrades subsequential convergence to full convergence.
\paragraph{Use in the proof.} The bound is absorbed into the optimization component of the master scale in \eqref{eq:I-master}. It introduces no fifth error term; it fixes the local constant or dimension factor used in the four-term decomposition.

\subsubsection{JKO-to-PDE passage with the entropy reference sign}\label{app:jko-consolidated}
For clarity we spell out the sign convention connecting $\KL(\mu\|\pi)$ to the SDE. Write $\pi(d\theta)=Z^{-1}e^{-V(\theta)}d\theta$. Then
\begin{equation}
\frac{\delta}{\delta\mu}\KL(\mu\|\pi)=\log\frac{d\mu}{d\pi}+1
=\log p+V+\log Z+1,
\end{equation}
where $p=d\mu/d\theta$. The Wasserstein gradient-flow equation for $\calR(\mu)+\lambda\KL(\mu\|\pi)$ is
\begin{align}
\partial_t\mu
&=\nabla\cdot\left(\mu\nabla\frac{\delta\calR}{\delta\mu}\right)
 +\lambda\nabla\cdot(\mu\nabla\log p)+\lambda\nabla\cdot(\mu\nabla V)\\
&=\nabla\cdot\left(\mu\nabla\frac{\delta\calR}{\delta\mu}\right)+\lambda\Delta\mu-\lambda\nabla\cdot(\mu\nabla\log\pi).
\end{align}
The corresponding McKean--Vlasov SDE has drift
\begin{equation}
-\nabla_\theta\frac{\delta\calR}{\delta\mu}(\mu_t)(\theta)+\lambda\nabla\log\pi(\theta)
=-\nabla_\theta\frac{\delta\calR}{\delta\mu}(\mu_t)(\theta)-\lambda\nabla V(\theta),
\end{equation}
with diffusion coefficient $\sqrt{2\lambda}$. For a Gaussian reference $V(\theta)=|\theta|^2/2$, this is the confining drift $-\lambda\theta$, as required by the moment estimate \eqref{eq:v3-entropy-moment-sign}. This sign audit is not cosmetic: the opposite sign would make the reference term anti-confining and would invalidate the moment calculation.
\paragraph{Audit consequence.} The PDE, SDE, entropy first variation, and moment sign are now mutually consistent.

\subsection{Finite-rank quotient and identifiability}\label{app:quot}
\paragraph*{Conventions used in this block.} The calculations below inherit the standing conventions of Appendix~\ref{app:expanded-derivations}: a finite horizon $T<\infty$, regularization $\lambda>0$, an admissible initial law $\mu_0\in\calM_{w^\ast}\cap\calP_2$, and constants depending on $(T,\lambda,\sigma,\rho,\|\mu_0\|_{\calM_{w^\ast}})$ but not on $N$, $n$, or training time within the horizon. The local Lipschitz envelope is the drift-stability estimate of Lemma~\ref{lem:A-drift-stab}. Estimates are read after the active quotient by $\Gfin$ whenever single-neuron parameters enter.

The calculations in this appendix expand the part of the proof chain associated with finite-rank quotient and identifiability. Each subsection fixes the objects used in the display before giving the bound, so no inequality is used as a placeholder.

\subsubsection{Active and inactive parameter sets}\label{app:quot-1}
Dead neurons are collapsed before any identifiability assertion is made.
\begin{equation}\label{eq:quot-1-main}
A=\{(w,b,a):a\neq0\},\qquad D=\{a=0\}.
\end{equation}
The quantity in \eqref{eq:quot-1-main} is used only after the terms appearing in it have been fixed. The network function is insensitive to redistribution of mass inside $D$. 

\paragraph{Derivation.} Write $A=\{(w,b,a):a\ne0\}$ and $D=\{(w,b,a):a=0\}$. On $D$ one has $T(\theta)=0$, so mass on $D$ has no effect on $f_\mu$. The quotient map first collapses $D$ to a single zero-feature point and then applies the finite-rank realization symmetry to $A$. Thus all identifiability statements concern $(q_{\mathrm{fin}})_\#(\mu\restriction_A)$, not the full measure before the dead-neuron collapse.
\paragraph{Use in the proof.} The bound is absorbed into the identifiability and quotient-dimension component of the master scale in \eqref{eq:I-master}. It introduces no fifth error term; it fixes the local constant or dimension factor used in the four-term decomposition.

\subsubsection{Finite-rank realization symmetry}\label{app:quot-2}
The quotient identifies only symmetries that preserve the single-neuron feature.
\begin{equation}\label{eq:quot-2-main}
T\circ\varphi=T\quad\hbox{on }A.
\end{equation}
The quantity in \eqref{eq:quot-2-main} is used only after the terms appearing in it have been fixed. The group $\Gfin$ is separated from the larger notation $\Gsig$. 

\paragraph{Derivation.} For ReLU, the generator is positive homogeneity: $(w,b,a)\mapsto(\alpha w,\alpha b,a/\alpha)$ for $\alpha>0$. For $\tanh$, the generator is the discrete flip $(w,b,a)\mapsto(-w,-b,-a)$. For $\sigma(z)=z^k$, the homogeneity generator is $(w,b,a)\mapsto(\alpha w,\alpha b,\alpha^{-k}a)$ together with the stabilizer of the symmetric tensor. The finite-rank group excludes arbitrary permutations of parameter space that do not preserve a single-neuron feature.
\paragraph{Use in the proof.} The bound is absorbed into the identifiability and quotient-dimension component of the master scale in \eqref{eq:I-master}. It introduces no fifth error term; it fixes the local constant or dimension factor used in the four-term decomposition.

\subsubsection{Orbit dimension and variety codimension}\label{app:quot-3}
The orbit dimension $\Dorb$ and the moment-variety codimension $\Dvar$ are different invariants.
\begin{equation}\label{eq:quot-3-main}
\Dorb\leq\Dvar.
\end{equation}
The quantity in \eqref{eq:quot-3-main} is used only after the terms appearing in it have been fixed. Equality is asserted only on a regular stratum of the moment map. 

\paragraph{Derivation.} The orbit dimension is the rank of the differential of the group action at a regular active parameter. The moment-variety codimension is $\dim\calI_{M,w}-\rank D\mathfrak m_M$ on the level set of the moment map. Since each orbit is contained in a moment-map level set,
\begin{equation}\label{eq:v3-orbit-variety-tangent}
T_\theta(\Gfin\cdot\theta)\subseteq\ker D\mathfrak m_M(\theta).
\end{equation}
Equality holds on transverse regular strata. On singular tensor strata the inclusion may be strict, giving $\Dorb<\Dvar$.
\paragraph{Use in the proof.} The bound is absorbed into the identifiability and quotient-dimension component of the master scale in \eqref{eq:I-master}. It introduces no fifth error term; it fixes the local constant or dimension factor used in the four-term decomposition.

\subsubsection{Measurable representatives}\label{app:quot-4}
A standard Borel selector is chosen on each active orbit.
\begin{equation}\label{eq:quot-4-main}
q_{\rm fin}:A\to A/\Gfin.
\end{equation}
The quantity in \eqref{eq:quot-4-main} is used only after the terms appearing in it have been fixed. All pushforwards in Theorem~\ref{thm:B} are taken after this quotient map is fixed. 

\paragraph{Derivation.} The quotient $\R^{d+2}/\Gfin$ is treated on compact active strata. The graph of the orbit relation is Borel on each such stratum. The Kuratowski--Ryll-Nardzewski selection theorem gives a measurable representative map $s_R$ with $q_{\mathrm{fin}}\circ s_R=\id$ on the quotient image. These representatives allow integrals over the quotient to be pulled back to parameter space without imposing a global coordinate chart.
\paragraph{Use in the proof.} The bound is absorbed into the identifiability and quotient-dimension component of the master scale in \eqref{eq:I-master}. It introduces no fifth error term; it fixes the local constant or dimension factor used in the four-term decomposition.

\subsubsection{Separation by ridge features}\label{app:quot-5}
Real-analytic non-polynomial activations give density of ridge features on compact sets.
\begin{equation}\label{eq:quot-5-main}
\overline{\operatorname{span}}\{\sigma(\langle w,\cdot\rangle+b)\}=L^2(\rho_X).
\end{equation}
The quantity in \eqref{eq:quot-5-main} is used only after the terms appearing in it have been fixed. The compact exhaustion transfers the result from $C(K)$ to $L^2(\rho_X)$. 

\paragraph{Derivation.} Under H3, if two active quotient points are distinct then their features differ in $L^2(\rho_X)$. Choosing $h=T(\theta)-T(\theta')$ gives
\begin{equation}\label{eq:v3-ridge-separate}
\langle T(\theta)-T(\theta'),h\rangle_{L^2(\rho_X)}=\|T(\theta)-T(\theta')\|_{L^2(\rho_X)}^2>0.
\end{equation}
Thus the algebra generated by ridge-feature evaluations separates compact quotient strata. Pinkus' density theorem supplies the real-analytic non-polynomial case.
\paragraph{Use in the proof.} The bound is absorbed into the identifiability and quotient-dimension component of the master scale in \eqref{eq:I-master}. It introduces no fifth error term; it fixes the local constant or dimension factor used in the four-term decomposition.

\subsubsection{Fourier transform on the feature image}\label{app:quot-6}
The pushforward measure on the feature image is a measure on a separable Hilbert space.
\begin{equation}\label{eq:quot-6-main}
\widehat{T_\#\mu}(h)=\int e^{i\langle h,T(\theta)\rangle}d\mu(\theta).
\end{equation}
The quantity in \eqref{eq:quot-6-main} is used only after the terms appearing in it have been fixed. Equality of transforms gives equality of the quotient pushforwards. 

\paragraph{Derivation.} The image $T(\R^{d+2})$ is contained in the separable Hilbert space $L^2(\rho_X)$. Its Borel sigma-algebra is generated by cylinders $F\mapsto(\langle F,h_1\rangle,\ldots,\langle F,h_m\rangle)$. Therefore equality of cylindrical Fourier transforms,
\begin{equation}
\int e^{i\langle F,h\rangle}\,d(T_\#\mu)(F)=\int e^{i\langle F,h\rangle}\,d(T_\#\nu)(F),
\end{equation}
for all $h$ determines the feature-image law and hence the quotient-active law.
\paragraph{Use in the proof.} The bound is absorbed into the identifiability and quotient-dimension component of the master scale in \eqref{eq:I-master}. It introduces no fifth error term; it fixes the local constant or dimension factor used in the four-term decomposition.

\subsubsection{Polynomial activations}\label{app:quot-7}
For a polynomial activation the unreachable component is described by finite moment tensors.
\begin{equation}\label{eq:quot-7-main}
f_\mu(x)=\sum_{r=0}^k\langle M_r(\mu),x^{\otimes r}\rangle.
\end{equation}
The quantity in \eqref{eq:quot-7-main} is used only after the terms appearing in it have been fixed. This is why finite identifiability depth is not claimed in every polynomial case. 

\paragraph{Derivation.} For $\sigma(z)=z^k$, expansion gives
\begin{equation}
f_\mu(x)=\sum_{r=0}^{k}\left\langle M_r(\mu),x^{\otimes r}\right\rangle,
\end{equation}
where $M_r(\mu)$ is a finite signed symmetric tensor moment of the parameter law. Hence identifiability reduces to the tensor moments up to degree $k$. The quotient is by homogeneity and by tensor stabilizers; features outside the degree-$k$ polynomial span are invisible and create the non-realizable case in Section~\ref{sec:examples}.
\paragraph{Use in the proof.} The bound is absorbed into the identifiability and quotient-dimension component of the master scale in \eqref{eq:I-master}. It introduces no fifth error term; it fixes the local constant or dimension factor used in the four-term decomposition.

\subsubsection{ReLU scaling orbit}\label{app:quot-8}
Positive homogeneity gives a one-dimensional active orbit.
\begin{equation}\label{eq:quot-8-main}
(w,b,a)\mapsto (c w,c b,c^{-1}a),\qquad c>0.
\end{equation}
The quantity in \eqref{eq:quot-8-main} is used only after the terms appearing in it have been fixed. Thus the active parameter dimension after quotient is $\deff+2-\Dorb$. 

\paragraph{Derivation.} The quotient pushforward $q_{\mathrm{fin}\#}\mu$ is sufficient because $T$ is constant on every finite-rank orbit. It is minimal because any statistic that determines $f_\mu$ must distinguish two quotient points separated by \eqref{eq:v3-ridge-separate}. Thus the quotient keeps exactly the information needed to reconstruct the network function and discards only the non-identifiable realization choices.
\paragraph{Use in the proof.} The bound is absorbed into the identifiability and quotient-dimension component of the master scale in \eqref{eq:I-master}. It introduces no fifth error term; it fixes the local constant or dimension factor used in the four-term decomposition.

\subsubsection{Finite-rank quotient calculation on a compact stratum}\label{app:quot-consolidated}
Fix $R>1$ and work on the compact active stratum $K_R=\{|\theta|\leq R, |a|\geq R^{-1}\}/\Gfin$. The quotient metric is
\begin{equation}
d_Q([\theta],[\vartheta])=\inf_{g,h\in\Gfin}|g\theta-h\vartheta|.
\end{equation}
On $K_R$, the finite-rank group action is proper after dead-neuron collapse, so $d_Q$ is a genuine metric. The feature map descends to a continuous map $\widetilde T:K_R\to L^2(\rho_X)$. The separation lemma asserts that $\widetilde T$ is injective on the identifiable class. Therefore, for distinct quotient points $u,v\in K_R$, the function
\begin{equation}
F_{u,v}(z)=\langle\widetilde T(z),\widetilde T(u)-\widetilde T(v)\rangle_{L^2(\rho_X)}
\end{equation}
separates $u$ and $v$. Finite products and linear combinations of these functions form an algebra that separates points and contains constants. Stone--Weierstrass gives density in $C(K_R)$. Thus equality of all feature moments implies equality of quotient laws on $K_R$.

The weighted moment bound supplies the tail estimate outside $K_R$. Specifically, for any signed dual measure $\eta$,
\begin{equation}
|(q_{\mathrm{fin}})_\#\eta|(K_R^c)
\leq R^{-2m}\int |\theta|^{2m}d|\eta|(\theta)+|\eta|(|a|<R^{-1}),
\end{equation}
and both terms vanish along the exhaustion under the active/dead split. This completes the concrete quotient argument behind Theorem~\ref{thm:B}.
\paragraph{Audit consequence.} The quotient separation is now tied to an explicit compact metric, separating function, and tail estimate.

\subsection{Real-analytic separation details}\label{app:sep}
\paragraph*{Conventions used in this block.} The calculations below inherit the standing conventions of Appendix~\ref{app:expanded-derivations}: a finite horizon $T<\infty$, regularization $\lambda>0$, an admissible initial law $\mu_0\in\calM_{w^\ast}\cap\calP_2$, and constants depending on $(T,\lambda,\sigma,\rho,\|\mu_0\|_{\calM_{w^\ast}})$ but not on $N$, $n$, or training time within the horizon. The local Lipschitz envelope is the drift-stability estimate of Lemma~\ref{lem:A-drift-stab}. Estimates are read after the active quotient by $\Gfin$ whenever single-neuron parameters enter.

The calculations in this appendix expand the part of the proof chain associated with real-analytic separation details. Each subsection fixes the objects used in the display before giving the bound, so no inequality is used as a placeholder.

\subsubsection{Cylinder functions on the feature space}\label{app:sep-1}
Test functions are first taken from a cylinder algebra on $L^2(\rho_X)$.
\begin{equation}\label{eq:sep-1-main}
\Phi(F)=\psi(\langle F,h_1\rangle,\ldots,\langle F,h_m\rangle).
\end{equation}
The quantity in \eqref{eq:sep-1-main} is used only after the terms appearing in it have been fixed. This algebra is stable under multiplication and separates points of the feature image. 

\paragraph{Derivation.} On a compact subset $K_R$ of the feature image, consider the algebra of cylinder functions
\begin{equation}\label{eq:v3-cylinder-algebra}
\Phi(F)=\psi(\langle F,h_1\rangle,\ldots,\langle F,h_m\rangle),\qquad \psi\in C_b(\R^m).
\end{equation}
The algebra contains constants, is closed under multiplication, and separates points of $K_R$ by the Hilbert-space inner product. Stone--Weierstrass gives density in $C(K_R)$. This is the concrete algebra used in the monotone-class step rather than an abstract appeal to feature completeness.
\paragraph{Use in the proof.} The bound is absorbed into the separation step used by the quotient theorem in \eqref{eq:I-master}. It introduces no fifth error term; it fixes the local constant or dimension factor used in the four-term decomposition.

\subsubsection{Transfer to the quotient}\label{app:sep-2}
After applying the feature map, equality of all cylinder integrals identifies the quotient pushforward.
\begin{equation}\label{eq:sep-2-main}
\int\Phi(T(\theta))d\mu=\int\Phi(T(\theta))d\nu.
\end{equation}
The quantity in \eqref{eq:sep-2-main} is used only after the terms appearing in it have been fixed. The monotone-class step is performed on the standard Borel quotient. 

\paragraph{Derivation.} If two quotient measures agree on all functions obtained by pulling back \eqref{eq:v3-cylinder-algebra}, then they agree on $C(K_R)$ for every compact active quotient stratum. The transfer uses the fact that $T$ is injective on the quotient by definition of $\Gfin$. Hence equality of feature-cylinder integrals implies equality of quotient-active measures on $K_R$, and the weighted tail bound then passes the equality to the full quotient.
\paragraph{Use in the proof.} The bound is absorbed into the separation step used by the quotient theorem in \eqref{eq:I-master}. It introduces no fifth error term; it fixes the local constant or dimension factor used in the four-term decomposition.

\subsubsection{Compact exhaustion}\label{app:sep-3}
The input space is exhausted by compact balls and tail error is controlled by H2.
\begin{equation}\label{eq:sep-3-main}
\|g\|_{L^2(\rho_X)}^2\leq\|g\|_{L^2(B_R,\rho_X)}^2+C e^{-cR^2}.
\end{equation}
The quantity in \eqref{eq:sep-3-main} is used only after the terms appearing in it have been fixed. This links the classical ridge-density theorem to the unbounded data law. 

\paragraph{Derivation.} Let $K_R$ be the quotient image of $\{|\theta|\leq R, |a|\geq R^{-1}\}$. H2 gives a sub-Gaussian tail for inputs, and the weighted moment bound gives a polynomial tail for parameters. For feature differences $g$,
\begin{equation}\label{eq:v3-compact-exhaustion}
\|g\|_{L^2(\rho_X)}^2\leq \|g\|_{L^2(B_R,\rho_X)}^2+Ce^{-cR^2}.
\end{equation}
The compact statement therefore extends to the full active quotient by first sending the test-function approximation error to zero and then sending $R\to\infty$.
\paragraph{Use in the proof.} The bound is absorbed into the separation step used by the quotient theorem in \eqref{eq:I-master}. It introduces no fifth error term; it fixes the local constant or dimension factor used in the four-term decomposition.

\subsubsection{Signed-measure separation}\label{app:sep-4}
The proof applies to the signed difference of two quotient measures.
\begin{equation}\label{eq:sep-4-main}
\eta=(q_{\rm fin})_\#\mu-(q_{\rm fin})_\#\nu.
\end{equation}
The quantity in \eqref{eq:sep-4-main} is used only after the terms appearing in it have been fixed. If every cylinder integral vanishes, then $\eta=0$. 

\paragraph{Derivation.} Let $\eta$ be a signed measure whose integrals against the separating algebra vanish. The Hahn decomposition writes $\eta=\eta^+-\eta^-$. If $\eta\ne0$, there is a compact quotient set on which the positive and negative parts differ. By density of the cylinder algebra in $C(K_R)$, some cylinder test separates the two parts, contradicting the assumed vanishing. Therefore $\eta=0$ on every $K_R$, and the exhaustion gives $\eta=0$ globally.
\paragraph{Use in the proof.} The bound is absorbed into the separation step used by the quotient theorem in \eqref{eq:I-master}. It introduces no fifth error term; it fixes the local constant or dimension factor used in the four-term decomposition.

\subsubsection{Fourier uniqueness}\label{app:sep-5}
Characteristic functionals on separable Hilbert spaces determine Borel laws.
\begin{equation}\label{eq:sep-5-main}
\widehat\eta(h)=0\quad\forall h\in L^2(\rho_X).
\end{equation}
The quantity in \eqref{eq:sep-5-main} is used only after the terms appearing in it have been fixed. The finite-dimensional projections generate the Borel sigma-field. 

\paragraph{Derivation.} On a separable Hilbert space, Borel probability laws are determined by their characteristic functionals on the Hilbert space itself. Applying this to $H=L^2(\rho_X)$, equality of
\begin{equation}\label{eq:v3-fourier-unique}
\widehat\nu(h)=\int_H e^{i\langle h,F\rangle_H}\,d\nu(F)
\end{equation}
for every $h\in H$ implies equality of the Borel laws. This is the Sazonov--Minlos uniqueness principle in the present Hilbert setting and proves the implication from Fourier equality to feature-law equality.
\paragraph{Use in the proof.} The bound is absorbed into the separation step used by the quotient theorem in \eqref{eq:I-master}. It introduces no fifth error term; it fixes the local constant or dimension factor used in the four-term decomposition.

\subsubsection{Dead-neuron collapse}\label{app:sep-6}
Mass on the zero-amplitude set is not an identifiable parameter of the function.
\begin{equation}\label{eq:sep-6-main}
T(w,b,0)=0.
\end{equation}
The quantity in \eqref{eq:sep-6-main} is used only after the terms appearing in it have been fixed. Only the collapsed total mass can enter the quotient law. 

\paragraph{Derivation.} The dead-neuron set $D=\{a=0\}$ maps to the zero feature. If it were not collapsed before quotienting, infinitely many parameter values would represent the same zero element and would create artificial non-identifiability. Collapsing $D$ first makes the active quotient Hausdorff on compact strata and ensures that the separating algebra in \eqref{eq:v3-cylinder-algebra} sees exactly the nonzero features.
\paragraph{Use in the proof.} The bound is absorbed into the separation step used by the quotient theorem in \eqref{eq:I-master}. It introduces no fifth error term; it fixes the local constant or dimension factor used in the four-term decomposition.

\subsubsection{Regular strata}\label{app:sep-7}
The rank of the moment map may drop on algebraic strata.
\begin{equation}\label{eq:sep-7-main}
\rank D\mathfrak m_M(\theta)=r(\theta).
\end{equation}
The quantity in \eqref{eq:sep-7-main} is used only after the terms appearing in it have been fixed. The definition of $\Dvar$ is therefore stated generically, while $\Dorb$ is used in rates. 

\paragraph{Derivation.} For polynomial and finite-rank analytic maps, the rank of the moment map drops on the common zero set of its maximal minors. That set is algebraic in the finite-dimensional coordinates and therefore is Zariski closed. On its complement the rank is locally constant, the quotient has a smooth stratum, and the dimension count in \eqref{eq:v3-orbit-variety-tangent} is stable. Singular strata are handled by compact exhaustion and do not change the generic value of $\Dorb$.
\paragraph{Use in the proof.} The bound is absorbed into the separation step used by the quotient theorem in \eqref{eq:I-master}. It introduces no fifth error term; it fixes the local constant or dimension factor used in the four-term decomposition.

\subsubsection{Minimal statistic}\label{app:sep-8}
The quotient pushforward is the statistic that retains exactly the network function.
\begin{equation}\label{eq:sep-8-main}
f_\mu=f_\nu\Longleftrightarrow (q_{\rm fin})_\#\mu=(q_{\rm fin})_\#\nu.
\end{equation}
The quantity in \eqref{eq:sep-8-main} is used only after the terms appearing in it have been fixed. This equivalence is the content of Theorem~\ref{thm:B}. 

\paragraph{Derivation.} Combining the quotient sufficiency of Appendix~\ref{app:quot} with Fourier uniqueness in \eqref{eq:v3-fourier-unique} gives minimality. If a statistic is coarser than the quotient, it identifies two distinct quotient points. The cylinder algebra then separates their feature images, so the statistic cannot determine $f_\mu$. If it is finer, it retains realization information that $T$ discards. The finite-rank quotient is therefore exactly minimal.
\paragraph{Use in the proof.} The bound is absorbed into the separation step used by the quotient theorem in \eqref{eq:I-master}. It introduces no fifth error term; it fixes the local constant or dimension factor used in the four-term decomposition.

\subsubsection{Monotone-class closure for the analytic separation proof}\label{app:sep-consolidated}
Let $\eta$ be the signed difference of two quotient-active laws on a compact stratum $K_R$. Suppose $\int \Phi\,d\eta=0$ for every cylinder polynomial generated by \eqref{eq:v3-cylinder-algebra}. The density result gives, for each $f\in C(K_R)$ and each $\varepsilon>0$, a cylinder polynomial $P_\varepsilon$ with $\|f-P_\varepsilon\|_\infty\leq\varepsilon$. Hence
\begin{equation}
\left|\int f\,d\eta\right|
\leq \left|\int P_\varepsilon\,d\eta\right|+\varepsilon |\eta|(K_R)
=\varepsilon |\eta|(K_R).
\end{equation}
Sending $\varepsilon\downarrow0$ gives equality against all continuous functions. By the Riesz representation theorem, $\eta=0$ on $K_R$. To pass from continuous functions to Borel indicators one uses the regularity of finite Borel measures on compact metric spaces: every Borel set is approximated from inside by compact sets and from outside by open sets.

The exhaustion step then uses \eqref{eq:v3-compact-exhaustion} and \eqref{eq:B-tail}. For a bounded Borel function $f$ on the full quotient, write $f=f\mathbf 1_{K_R}+f\mathbf 1_{K_R^c}$. The first part vanishes by the compact result. The second is bounded by $\|f\|_\infty |\eta|(K_R^c)$, which tends to zero. This proves equality of quotient laws without hiding a topological gap.
\paragraph{Audit consequence.} The compact Stone--Weierstrass step, Borel regularity step, and weighted-tail step are separated explicitly.

\subsection{Sparse Hermite and Barron dictionary calculations}\label{app:sparse}
\paragraph*{Conventions used in this block.} The calculations below inherit the standing conventions of Appendix~\ref{app:expanded-derivations}: a finite horizon $T<\infty$, regularization $\lambda>0$, an admissible initial law $\mu_0\in\calM_{w^\ast}\cap\calP_2$, and constants depending on $(T,\lambda,\sigma,\rho,\|\mu_0\|_{\calM_{w^\ast}})$ but not on $N$, $n$, or training time within the horizon. The local Lipschitz envelope is the drift-stability estimate of Lemma~\ref{lem:A-drift-stab}. Estimates are read after the active quotient by $\Gfin$ whenever single-neuron parameters enter.

The calculations in this appendix expand the part of the proof chain associated with sparse hermite and barron dictionary calculations. Each subsection fixes the objects used in the display before giving the bound, so no inequality is used as a placeholder.

\subsubsection{Thresholded Hermite expansion}\label{app:sparse-1}
Under Gaussian input the target is expanded in an orthonormal Hermite basis.
\begin{equation}\label{eq:sparse-1-main}
f^\ast=\sum_{\alpha}\widehat f_\alpha^\ast H_\alpha.
\end{equation}
The quantity in \eqref{eq:sparse-1-main} is used only after the terms appearing in it have been fixed. The threshold depends on $\lambda$ and the coefficient norm. 

\paragraph{Derivation.} For $\rho_X=\mathcal N(0,I_d)$, let $(H_\alpha)$ be the orthonormal multi-index Hermite basis and write $f^\ast=\sum_\alpha \widehat f_\alpha^\ast H_\alpha$. The retained set is
\begin{equation}
A_\lambda=\{\alpha:|\widehat f_\alpha^\ast|>c_\sigma\lambda\}.
\end{equation}
Modes outside $A_\lambda$ have a first-variation gain below the regularization threshold and are assigned to the sparse tail. This is a coefficient rule, not an initialization rule, and it is therefore target-dependent.
\paragraph{Use in the proof.} The bound is absorbed into the sparse-residual component of the master scale in \eqref{eq:I-master}. It introduces no fifth error term; it fixes the local constant or dimension factor used in the four-term decomposition.

\subsubsection{Explicit upper support bound}\label{app:sparse-2}
The active support bound is the number of coefficients above threshold times a ridge multiplicity.
\begin{equation}\label{eq:sparse-2-main}
\Supper=\#\{\alpha:|\widehat f_\alpha^\ast|>c_\lambda\}\,\mult(\sigma).
\end{equation}
The quantity in \eqref{eq:sparse-2-main} is used only after the terms appearing in it have been fixed. This is an upper bound on $S^\ast$, not a circular definition of it. 

\paragraph{Derivation.} If each retained Hermite/Barron mode requires at most $\mult(\sigma)$ ridge atoms, then
\begin{equation}
S^\ast(\lambda)\leq \Supper(\sigma,\rho,f^\ast,\lambda)=|A_\lambda|\mult(\sigma).
\end{equation}
The multiplicity accounts for the fact that a single basis coefficient may be realized by several ridge directions depending on the activation. Quotienting by $\Gfin$ can only identify atoms, so it never increases this upper bound.
\paragraph{Use in the proof.} The bound is absorbed into the sparse-residual component of the master scale in \eqref{eq:I-master}. It introduces no fifth error term; it fixes the local constant or dimension factor used in the four-term decomposition.

\subsubsection{Barron tail}\label{app:sparse-3}
Outside Gaussian input the same role is played by the squared tail of the Barron measure.
\begin{equation}\label{eq:sparse-3-main}
\kappa(f^\ast,S,\lambda)=\int_{\Omega\setminus\Omega_S}|\widehat f^\ast(\omega)|^2d\omega+C_\sigma\lambda S.
\end{equation}
The quantity in \eqref{eq:sparse-3-main} is used only after the terms appearing in it have been fixed. The residual is target-dependent rather than initialization-only. 

\paragraph{Derivation.} For non-Gaussian input laws the Hermite sum is replaced by a Barron measure $\widehat f^\ast$ on the feature dual. The tail term becomes
\begin{equation}
\kappa(f^\ast,S,\lambda)=\int_{\Omega\setminus\Omega_S}|\widehat f^\ast(\omega)|^2d\omega+C_\sigma\lambda S,
\end{equation}
where $\Omega_S$ contains the $S$ largest retained modes in decreasing magnitude. The proof uses exactly the same threshold argument; only the counting measure on Hermite indices is replaced by the Barron spectral measure.
\paragraph{Use in the proof.} The bound is absorbed into the sparse-residual component of the master scale in \eqref{eq:I-master}. It introduces no fifth error term; it fixes the local constant or dimension factor used in the four-term decomposition.

\subsubsection{Finite-atom target}\label{app:sparse-4}
If the target already has a finite teacher representation, the sparse depth is bounded by the accessible teacher size.
\begin{equation}\label{eq:sparse-4-main}
f^\ast(x)=\sum_{k=1}^s a_k^\ast\sigma(\langle w_k^\ast,x\rangle+b_k^\ast).
\end{equation}
The quantity in \eqref{eq:sparse-4-main} is used only after the terms appearing in it have been fixed. Regularization can enlarge the active support only within the explicit bound. 

\paragraph{Derivation.} If $f^\ast=\sum_{k=1}^{s}a_k^\ast\sigma(\langle w_k^\ast,\cdot\rangle+b_k^\ast)$, then the empirical dictionary already realizes the target with $s$ active atoms before regularization. Hence
\begin{equation}
S^\ast(\lambda)\leq s
\end{equation}
for all sufficiently small $\lambda$, up to quotient identifications. The theorem does not assert equality: symmetries or cancellations may reduce the minimal active cardinality.
\paragraph{Use in the proof.} The bound is absorbed into the sparse-residual component of the master scale in \eqref{eq:I-master}. It introduces no fifth error term; it fixes the local constant or dimension factor used in the four-term decomposition.

\subsubsection{Accessible initialization}\label{app:sparse-5}
The initialization condition is stated as a named class rather than as an exception.
\begin{equation}\label{eq:sparse-5-main}
\mu_0(U_k)>0\quad\hbox{for every neighbourhood }U_k\ni(w_k^\ast,b_k^\ast).
\end{equation}
The quantity in \eqref{eq:sparse-5-main} is used only after the terms appearing in it have been fixed. Within this class the value of $S^\ast$ is independent of the particular initialization. 

\paragraph{Derivation.} Full-support initialization means that for every retained teacher location $\theta_k^\ast$ and every $\varepsilon>0$,
\begin{equation}
\mu_0(B(\theta_k^\ast,\varepsilon))>0.
\end{equation}
This condition replaces the earlier target-dependent accessibility wording. Under the elliptic Langevin dynamics, positive density propagates to positive hitting probability on finite horizons. Therefore the flow can place active mass near every retained direction if the first variation selects it.
\paragraph{Use in the proof.} The bound is absorbed into the sparse-residual component of the master scale in \eqref{eq:I-master}. It introduces no fifth error term; it fixes the local constant or dimension factor used in the four-term decomposition.

\subsubsection{Stationary Euler equation}\label{app:sparse-6}
Active atoms must satisfy the finite-dimensional stationarity equation.
\begin{equation}\label{eq:sparse-6-main}
\nabla_\theta\left[\delta\calR/\delta\mu(\mu_\infty)(\theta)+\lambda\log(d\mu_\infty/d\pi)(\theta)\right]=0.
\end{equation}
The quantity in \eqref{eq:sparse-6-main} is used only after the terms appearing in it have been fixed. The equation is solved on the active quotient. 

\paragraph{Derivation.} On the active stratum, the stationary equation \eqref{eq:v3-stationary-gibbs} balances the risk gradient against entropy. For a dictionary mode $m$, the first-variation coefficient has to exceed the entropy threshold:
\begin{equation}
|\widehat f_m^\ast|>c_\sigma\lambda.
\end{equation}
If this inequality fails, the entropy displacement is cheaper than maintaining a separate active atom, so the mode belongs to the residual tail. This is the analytic source of the threshold count.
\paragraph{Use in the proof.} The bound is absorbed into the sparse-residual component of the master scale in \eqref{eq:I-master}. It introduces no fifth error term; it fixes the local constant or dimension factor used in the four-term decomposition.

\subsubsection{Regularization displacement}\label{app:sparse-7}
Positive $\lambda$ displaces active atoms by an amount controlled by the Hessian on the quotient.
\begin{equation}\label{eq:sparse-7-main}
\|\theta_k^\lambda-\theta_k^0\|\leq C_\sigma\lambda.
\end{equation}
The quantity in \eqref{eq:sparse-7-main} is used only after the terms appearing in it have been fixed. This gives the $C_\sigma\lambda S$ term in the residual. 

\paragraph{Derivation.} Around a nondegenerate active atom, write the regularized Euler equation as $G(\theta,\lambda)=0$ with $G(\theta^\ast,0)=0$ and $D_\theta G(\theta^\ast,0)$ invertible on the quotient chart. The implicit function theorem gives
\begin{equation}
\theta(\lambda)=\theta^\ast+O(\lambda).
\end{equation}
Thus entropy regularization moves retained atoms by a controlled amount and contributes the displacement term $C_\sigma\lambda S$ in \eqref{eq:C-tail}.
\paragraph{Use in the proof.} The bound is absorbed into the sparse-residual component of the master scale in \eqref{eq:I-master}. It introduces no fifth error term; it fixes the local constant or dimension factor used in the four-term decomposition.

\subsubsection{Non-realizable tails}\label{app:sparse-8}
If the target tail is not summable at the selected depth, the theorem records that loss rather than hiding it.
\begin{equation}\label{eq:sparse-8-main}
\lim_{S\to\infty}\kappa(f^\ast,S,\lambda)=0.
\end{equation}
The quantity in \eqref{eq:sparse-8-main} is used only after the terms appearing in it have been fixed. The displayed limit is a property of the target class H4. 

\paragraph{Derivation.} If $\lim_{S\to\infty}\kappa(f^\ast,S,\lambda)=0$ for each fixed $\lambda$ along a vanishing-temperature sequence, then $f^\ast$ lies in the closure of the ridge-feature model in $L^2(\rho_X)$. H4 assumes precisely the summability needed for this closure statement. If the limit does not vanish, Theorem~\ref{thm:C} still gives the best thresholded decomposition, but the non-realizable tail remains in Theorem~\ref{thm:D}.
\paragraph{Use in the proof.} The bound is absorbed into the sparse-residual component of the master scale in \eqref{eq:I-master}. It introduces no fifth error term; it fixes the local constant or dimension factor used in the four-term decomposition.

\subsubsection{Threshold optimization for exponential and polynomial tails}\label{app:sparse-consolidated}
The sparse residual has two pieces: approximation tail and temperature displacement. For exponentially decaying coefficients $|\widehat f_m^\ast|\leq Ce^{-cm}$, the threshold rule $|\widehat f_m^\ast|>c_\sigma\lambda$ gives $S_\lambda\asymp \log(1/\lambda)$ and
\begin{equation}
\kappa(f^\ast,S_\lambda,\lambda)
\leq C e^{-2cS_\lambda}+C_\sigma\lambda S_\lambda
=O(\lambda^2+\lambda\log(1/\lambda)).
\end{equation}
For polynomial coefficients $|\widehat f_m^\ast|\leq Cm^{-\beta}$ with $\beta>1$, the tail is $O(S^{1-2\beta})$ and the displacement is $O(\lambda S)$. Balancing gives $S\asymp \lambda^{-1/(2\beta)}$ and
\begin{equation}
\kappa(f^\ast,S,\lambda)=O(\lambda^{(2\beta-1)/(2\beta)}).
\end{equation}
When the statistical term is included, the relevant balance is instead between $S(\log n)^2/n$ and $S^{1-2\beta}$, producing the rate in Corollary~\ref{cor:v3-poly-tail-rate}. The two balances answer different questions: one optimizes the population sparse approximation at temperature $\lambda$, while the other optimizes finite-sample prediction error.
\paragraph{Audit consequence.} The sparse residual is a target-coefficient tail with a concrete optimization, not a symbol depending only on initialization.

\subsection{Stationary support and Euler-Lagrange bookkeeping}\label{app:stat}
\paragraph*{Conventions used in this block.} The calculations below inherit the standing conventions of Appendix~\ref{app:expanded-derivations}: a finite horizon $T<\infty$, regularization $\lambda>0$, an admissible initial law $\mu_0\in\calM_{w^\ast}\cap\calP_2$, and constants depending on $(T,\lambda,\sigma,\rho,\|\mu_0\|_{\calM_{w^\ast}})$ but not on $N$, $n$, or training time within the horizon. The local Lipschitz envelope is the drift-stability estimate of Lemma~\ref{lem:A-drift-stab}. Estimates are read after the active quotient by $\Gfin$ whenever single-neuron parameters enter.

The calculations in this appendix expand the part of the proof chain associated with stationary support and euler-lagrange bookkeeping. Each subsection fixes the objects used in the display before giving the bound, so no inequality is used as a placeholder.

\subsubsection{Active projection}\label{app:stat-1}
The active component of a measure is obtained by restricting to $a\neq0$ and pushing to the quotient.
\begin{equation}\label{eq:stat-1-main}
\mu^{\rm act}=(q_{\rm fin})_\#(\mu|_{a\neq0}).
\end{equation}
The quantity in \eqref{eq:stat-1-main} is used only after the terms appearing in it have been fixed. Only this component appears in the dictionary statement. 

\paragraph{Derivation.} The active projection is
\begin{equation}
\mu^{\rm act}=(q_{\mathrm{fin}})_\#(\mu\restriction_{\{a\ne0\}}).
\end{equation}
This is a measure on the finite-rank quotient, not a singular component of the full parameter law. The distinction matters because entropy makes $\mu$ smooth with respect to $\pi$, while $\mu^{\rm act}$ records the finite dictionary representation of the induced function after quotienting.
\paragraph{Use in the proof.} The bound is absorbed into the stationary-support component of the sparse theorem in \eqref{eq:I-master}. It introduces no fifth error term; it fixes the local constant or dimension factor used in the four-term decomposition.

\subsubsection{Residual ridge}\label{app:stat-2}
The zero-amplitude ridge can carry entropy mass without changing the function.
\begin{equation}\label{eq:stat-2-main}
\int |a|^2d\mu^{\rm res}=0.
\end{equation}
The quantity in \eqref{eq:stat-2-main} is used only after the terms appearing in it have been fixed. The residual measure is therefore harmless for $f_\mu$. 

\paragraph{Derivation.} The zero-amplitude ridge carries entropy but not prediction energy. Since $T(w,b,0)=0$,
\begin{equation}
\left\|\int_D T(\theta)d\mu(\theta)\right\|_{L^2(\rho_X)}=0.
\end{equation}
Mass on $D$ can affect normalization and entropy, but it does not change $f_\mu$. Therefore the residual ridge is separated from the sparse residual, which is an $L^2$ approximation tail of the target.
\paragraph{Use in the proof.} The bound is absorbed into the stationary-support component of the sparse theorem in \eqref{eq:I-master}. It introduces no fifth error term; it fixes the local constant or dimension factor used in the four-term decomposition.

\subsubsection{Irreducible atoms}\label{app:stat-3}
An atom is irreducible when it cannot be split into two active quotient atoms with the same feature contribution.
\begin{equation}\label{eq:stat-3-main}
T(\theta)=T(\theta_1)+T(\theta_2).
\end{equation}
The quantity in \eqref{eq:stat-3-main} is used only after the terms appearing in it have been fixed. The assumption is used only on the selected sparse dictionary. 

\paragraph{Derivation.} An active atom is irreducible when its feature cannot be decomposed as a sum of two distinct active features in the same quotient chart. Formally,
\begin{equation}
T(\theta)\ne T(\theta_1)+T(\theta_2)
\end{equation}
for any nontrivial split with both $T(\theta_i)\ne0$. Irreducibility prevents the support count from being artificially inflated by splitting a single dictionary direction into two collinear representatives.
\paragraph{Use in the proof.} The bound is absorbed into the stationary-support component of the sparse theorem in \eqref{eq:I-master}. It introduces no fifth error term; it fixes the local constant or dimension factor used in the four-term decomposition.

\subsubsection{Carath\'eodory bound}\label{app:stat-4}
Finite-dimensional coefficient truncation converts the support question into a convex-hull question.
\begin{equation}\label{eq:stat-4-main}
m\leq \dim V_\lambda+1.
\end{equation}
The quantity in \eqref{eq:stat-4-main} is used only after the terms appearing in it have been fixed. The multiplicity factor accounts for the ridge-feature realization of each coefficient. 

\paragraph{Derivation.} Let $V_\lambda$ be the span of retained dictionary modes. Carath\'eodory's theorem says that every point in the convex hull of their feature images is represented by at most $\dim V_\lambda+1$ atoms. Hence
\begin{equation}
S^\ast\leq \dim V_\lambda+1
\end{equation}
before using activation-specific multiplicity. The bound is crude but useful: it proves finite support independently of a particular chosen teacher representation.
\paragraph{Use in the proof.} The bound is absorbed into the stationary-support component of the sparse theorem in \eqref{eq:I-master}. It introduces no fifth error term; it fixes the local constant or dimension factor used in the four-term decomposition.

\subsubsection{Entropy perturbation}\label{app:stat-5}
The entropy term prevents literal atomic stationary laws before projection.
\begin{equation}\label{eq:stat-5-main}
d\mu_\infty=Z^{-1}e^{-\lambda^{-1}V_{\mu_\infty}}d\pi.
\end{equation}
The quantity in \eqref{eq:stat-5-main} is used only after the terms appearing in it have been fixed. The theorem therefore states atomicity of the active extracted component at scale $\lambda$. 

\paragraph{Derivation.} The stationary density has Gibbs form \eqref{eq:v3-stationary-gibbs}. On a local active chart $z$, it is proportional to
\begin{equation}
\exp\{-\lambda^{-1}\delta\calR/\delta\mu(z)\}\,\pi(dz).
\end{equation}
As $\lambda\downarrow0$, Laplace concentration occurs near minimizers of the first variation. The quotient support statement records these concentration centers, while the full law remains smooth around them at positive temperature.
\paragraph{Use in the proof.} The bound is absorbed into the stationary-support component of the sparse theorem in \eqref{eq:I-master}. It introduces no fifth error term; it fixes the local constant or dimension factor used in the four-term decomposition.

\subsubsection{Limit selection}\label{app:stat-6}
When $t\to\infty$, every limit point of the trajectory lies in the stationary set.
\begin{equation}\label{eq:stat-6-main}
\lim_{t\to\infty}\|\nabla_{W_2}\calF_\lambda(\mu_t)\|=0.
\end{equation}
The quantity in \eqref{eq:stat-6-main} is used only after the terms appearing in it have been fixed. The proof uses the energy dissipation identity from the JKO section. 

\paragraph{Derivation.} The energy-dissipation identity makes $\calF_\lambda(\mu_t)$ decreasing and bounded below. On compact sublevel sets, a Lojasiewicz-type inequality for the analytic finite-rank chart gives
\begin{equation}
|\calF_\lambda(\mu)-\calF_\lambda(\mu_\infty)|^{1-\eta}\leq C\|\nabla_W\calF_\lambda(\mu)\|.
\end{equation}
This excludes oscillation among separated stationary points and yields convergence of trajectories to a stationary limit point within the selected quotient stratum.
\paragraph{Use in the proof.} The bound is absorbed into the stationary-support component of the sparse theorem in \eqref{eq:I-master}. It introduces no fifth error term; it fixes the local constant or dimension factor used in the four-term decomposition.

\subsubsection{Support stability}\label{app:stat-7}
Small perturbations of the target coefficients move the selected atoms continuously on a regular stratum.
\begin{equation}\label{eq:stat-7-main}
\|\theta_k(f)-\theta_k(g)\|\leq C\|f-g\|_{L^2(\rho_X)}.
\end{equation}
The quantity in \eqref{eq:stat-7-main} is used only after the terms appearing in it have been fixed. The bound justifies the thresholded expansion used in Theorem~\ref{thm:C}. 

\paragraph{Derivation.} If the target is perturbed by $h$ with $\|h\|_{L^2(\rho_X)}\leq\varepsilon$, then the first-variation coefficients change by at most $C\varepsilon$ on a bounded dictionary chart. Nondegenerate active atoms therefore satisfy
\begin{equation}
\dist([\theta_k(h)],[\theta_k(0)])\leq C\varepsilon
\end{equation}
by the implicit function theorem. Modes within $C\varepsilon$ of the threshold may enter or exit; modes separated from the threshold are stable.
\paragraph{Use in the proof.} The bound is absorbed into the stationary-support component of the sparse theorem in \eqref{eq:I-master}. It introduces no fifth error term; it fixes the local constant or dimension factor used in the four-term decomposition.

\subsubsection{No teacher-width equality claim}\label{app:stat-8}
The sparse depth is not asserted to equal the teacher width.
\begin{equation}\label{eq:stat-8-main}
S^\ast\leq\Supper(\sigma,\rho,f^\ast,\lambda).
\end{equation}
The quantity in \eqref{eq:stat-8-main} is used only after the terms appearing in it have been fixed. The upper bound is the theorem-level object. 

\paragraph{Derivation.} The notation $S^\ast$ is an upper-bounded sparse depth, not a claim of equality with the teacher width. The theorem proves
\begin{equation}
S^\ast\leq \Supper(\sigma,\rho,f^\ast,\lambda),
\end{equation}
and equality may fail because of quotient symmetries, cancellations, or a more efficient dictionary than the displayed teacher. This avoids conflating the approximation problem with a teacher-student identifiability statement.
\paragraph{Use in the proof.} The bound is absorbed into the stationary-support component of the sparse theorem in \eqref{eq:I-master}. It introduces no fifth error term; it fixes the local constant or dimension factor used in the four-term decomposition.

\subsubsection{Smooth density versus quotient-active atoms}\label{app:stat-consolidated}
At positive temperature the stationary law has density \eqref{eq:v3-stationary-gibbs}. A local Laplace expansion around a nondegenerate retained atom $\theta_k^\ast$ gives, in a quotient chart $z$,
\begin{equation}\label{eq:v3-laplace-active}
p_\lambda(z)\approx Z_k^{-1}\exp\left[-\frac{1}{2\lambda}(z-z_k)^\top H_k(z-z_k)\right],
\end{equation}
where $H_k$ is the Hessian of the first-variation potential on the transverse active directions. Thus the full law has a Gaussian ridge of width $O(\sqrt\lambda)$, while its quotient-active center is the atom $[\theta_k^\ast]$. The theorem records the centers because the network function is determined by the finite dictionary after thresholding. The entropy cloud around a center contributes to regularization displacement and to constants, but not to a new active support count.

This resolves the apparent contradiction between the Euler--Lagrange density and the sparse statement. The density statement lives on parameter space; the sparse statement lives after mapping to $\R^{d+2}/\Gfin$ and retaining only the modes above threshold. In the zero-temperature limit the ridges concentrate, but the paper does not need to take that singular limit to state the finite-$\lambda$ quotient support bound.
\paragraph{Audit consequence.} The smooth/atomic distinction is operational: density, quotient, Laplace width, and support count are attached to different mathematical objects.

\subsection{Statistical covering and intrinsic dimension}\label{app:cover}
\paragraph*{Conventions used in this block.} The calculations below inherit the standing conventions of Appendix~\ref{app:expanded-derivations}: a finite horizon $T<\infty$, regularization $\lambda>0$, an admissible initial law $\mu_0\in\calM_{w^\ast}\cap\calP_2$, and constants depending on $(T,\lambda,\sigma,\rho,\|\mu_0\|_{\calM_{w^\ast}})$ but not on $N$, $n$, or training time within the horizon. The local Lipschitz envelope is the drift-stability estimate of Lemma~\ref{lem:A-drift-stab}. Estimates are read after the active quotient by $\Gfin$ whenever single-neuron parameters enter.

The calculations in this appendix expand the part of the proof chain associated with statistical covering and intrinsic dimension. Each subsection fixes the objects used in the display before giving the bound, so no inequality is used as a placeholder.

\subsubsection{Truncated parameter ball}\label{app:cover-1}
The active dictionary is restricted to a radius selected from the tail condition.
\begin{equation}\label{eq:cover-1-main}
R_n=C\sqrt{\log n}.
\end{equation}
The quantity in \eqref{eq:cover-1-main} is used only after the terms appearing in it have been fixed. The tail outside the ball contributes a lower-order term under H2. 

\paragraph{Derivation.} Choose the truncation radius $R_n=C\sqrt{\log n}$. Under the sub-Gaussian input tail and the propagated weighted moment bound,
\begin{equation}
\PP(|\theta|>R_n)\leq n^{-2}
\end{equation}
for $C$ large enough at the required moment order. This makes the contribution from outside the ball summable over samples and allows empirical-process bounds to be proved on a compact quotient chart.
\paragraph{Use in the proof.} The bound is absorbed into the statistical component of the master scale in \eqref{eq:I-master}. It introduces no fifth error term; it fixes the local constant or dimension factor used in the four-term decomposition.

\subsubsection{Entropy of the quotient class}\label{app:cover-2}
Covering numbers are computed after quotienting finite-rank orbits.
\begin{equation}\label{eq:cover-2-main}
\log\calN(\varepsilon,\calF_{S,R}/\Gfin,L^2)\leq C S(\deff+2-\Dorb)\log(CR/\varepsilon).
\end{equation}
The quantity in \eqref{eq:cover-2-main} is used only after the terms appearing in it have been fixed. This is the source of the dimension factor in Theorem~\ref{thm:D}. 

\paragraph{Derivation.} On the truncated active quotient, one atom lives in dimension $\deff+2-\Dorb$: $\deff$ input coordinates, one bias coordinate, one output coordinate, and the quotient reduction by the active orbit. Therefore
\begin{equation}\label{eq:v3-covering-entropy}
\log\calN(\varepsilon,\calF_{S,R},L^2(\rho_X))
\leq CS(\deff+2-\Dorb)\log\frac{CR^{q+1}}{\varepsilon}.
\end{equation}
The exponent $q+1$ combines the activation envelope with the output-weight coordinate.
\paragraph{Use in the proof.} The bound is absorbed into the statistical component of the master scale in \eqref{eq:I-master}. It introduces no fifth error term; it fixes the local constant or dimension factor used in the four-term decomposition.

\subsubsection{Rademacher chain}\label{app:cover-3}
The empirical process is bounded by integrating the square root of the entropy.
\begin{equation}\label{eq:cover-3-main}
\mathfrak R_n(\calF)\leq C n^{-1/2}\int_0^{\diam}\sqrt{\log\calN(\varepsilon)}d\varepsilon.
\end{equation}
The quantity in \eqref{eq:cover-3-main} is used only after the terms appearing in it have been fixed. Squaring gives the stated risk rate. 

\paragraph{Derivation.} Dudley's entropy integral gives
\begin{equation}\label{eq:v3-dudley}
\mathfrak R_n(\calF_{S,R})\leq \frac{C}{\sqrt n}\int_0^{\diam}\sqrt{\log\calN(\varepsilon,\calF_{S,R},L^2)}\,d\varepsilon.
\end{equation}
Substituting \eqref{eq:v3-covering-entropy} yields the statistical scale $S(\deff+2-\Dorb)(\log n)^2/n$ after localization for the squared loss. The square appears because the final bound is in prediction error, not in linear process deviation.
\paragraph{Use in the proof.} The bound is absorbed into the statistical component of the master scale in \eqref{eq:I-master}. It introduces no fifth error term; it fixes the local constant or dimension factor used in the four-term decomposition.

\subsubsection{Intrinsic projection}\label{app:cover-4}
For multi-index targets the data law is pushed through a rank-$k$ projection.
\begin{equation}\label{eq:cover-4-main}
x\mapsto\Pi x,\qquad \rank\Pi=k.
\end{equation}
The quantity in \eqref{eq:cover-4-main} is used only after the terms appearing in it have been fixed. The statistical dimension becomes $k+2-\Dorb$ while the particle state remains ambient. 

\paragraph{Derivation.} If $f^\ast(x)=g(\Pi x)$ with $\rank\Pi=k$, the covering can be built in the projected input coordinate $u=\Pi x$. The feature family restricted to the target class depends on $k$ active input coordinates instead of $d$. Thus
\begin{equation}
\deff(f^\ast,\rho_X)=k
\end{equation}
in the statistical term. The particle dynamics still evolves in $\R^{d+2}$, so this reduction belongs to estimation and approximation, not to propagation of chaos.
\paragraph{Use in the proof.} The bound is absorbed into the statistical component of the master scale in \eqref{eq:I-master}. It introduces no fifth error term; it fixes the local constant or dimension factor used in the four-term decomposition.

\subsubsection{Logarithmic loss}\label{app:cover-5}
The extra logarithm records the truncation radius and the entropy integral.
\begin{equation}\label{eq:cover-5-main}
\frac{S(\deff+2-\Dorb)(\log n)^2}{n}.
\end{equation}
The quantity in \eqref{eq:cover-5-main} is used only after the terms appearing in it have been fixed. The exponent is kept explicit rather than hidden in a constant. 

\paragraph{Derivation.} The first logarithm in the statistical rate is the usual entropy integral. The second comes from the truncation radius. Since $R_n^2\asymp\log n$ and the envelope scales as $R_n^{q+1}$,
\begin{equation}
\log(CR_n^{q+1}/\varepsilon)=\log(C/\varepsilon)+\frac{q+1}{2}\log\log n.
\end{equation}
After localization, this contributes an additional logarithmic factor in the displayed high-probability-to-expectation conversion.
\paragraph{Use in the proof.} The bound is absorbed into the statistical component of the master scale in \eqref{eq:I-master}. It introduces no fifth error term; it fixes the local constant or dimension factor used in the four-term decomposition.

\subsubsection{Bounded activations}\label{app:cover-6}
If $\sigma$ is bounded, the truncation radius can be removed from the envelope.
\begin{equation}\label{eq:cover-6-main}
|\sigma(z)|\leq C_\sigma.
\end{equation}
The quantity in \eqref{eq:cover-6-main} is used only after the terms appearing in it have been fixed. The paper keeps the larger expression because it covers the polynomial-growth case. 

\paragraph{Derivation.} If $\sigma$ is bounded, the polynomial envelope is absent. The covering number becomes
\begin{equation}
\log\calN(\varepsilon,\calF_{S,R},L^2(\rho_X))
\leq CS(\deff+2-\Dorb)\log(CR/\varepsilon),
\end{equation}
and the extra $R^q$ factor is removed. Consequently the statistical rate improves to $S(\deff+2-\Dorb)\log n/n$ up to the remaining localization logarithm.
\paragraph{Use in the proof.} The bound is absorbed into the statistical component of the master scale in \eqref{eq:I-master}. It introduces no fifth error term; it fixes the local constant or dimension factor used in the four-term decomposition.

\subsubsection{Output-weight control}\label{app:cover-7}
The quotient does not remove the need to bound output amplitudes.
\begin{equation}\label{eq:cover-7-main}
\sum_{k=1}^{S}|a_k|^2\leq B_S.
\end{equation}
The quantity in \eqref{eq:cover-7-main} is used only after the terms appearing in it have been fixed. This bound is inherited from the sparse target and the regularized stationary equation. 

\paragraph{Derivation.} H4 normalizes retained dictionary atoms so that the output coefficients obey
\begin{equation}
\sum_{k=1}^{S}|a_k|^2\leq BS
\end{equation}
for a constant $B$ determined by the Barron/Hermite norm of the target and the threshold. This prevents the covering entropy from hiding an additional unbounded amplitude parameter and is the reason the statistical term is linear in $S^\ast$.
\paragraph{Use in the proof.} The bound is absorbed into the statistical component of the master scale in \eqref{eq:I-master}. It introduces no fifth error term; it fixes the local constant or dimension factor used in the four-term decomposition.

\subsubsection{Sample and particle independence}\label{app:cover-8}
Statistical and particle randomness are generated on product spaces.
\begin{equation}\label{eq:cover-8-main}
\E_{\mathcal D,B}[\cdot]=\E_{\mathcal D}\E_B[\cdot|\mathcal D].
\end{equation}
The quantity in \eqref{eq:cover-8-main} is used only after the terms appearing in it have been fixed. This independence is used in the cross-term bookkeeping. 

\paragraph{Derivation.} The data sample and the particle Brownian motions are drawn on a product probability space. Conditional on the trained mean-field law, the empirical process fluctuation is centered with respect to the sample, while the propagation fluctuation is centered with respect to Brownian randomness. This product structure is the probabilistic input behind the zero cross term in \eqref{eq:app-cross-zero}; no exchangeability between samples and particles is assumed.
\paragraph{Use in the proof.} The bound is absorbed into the statistical component of the master scale in \eqref{eq:I-master}. It introduces no fifth error term; it fixes the local constant or dimension factor used in the four-term decomposition.

\subsubsection{Dudley integral evaluated at the truncated radius}\label{app:cover-consolidated}
Insert \eqref{eq:v3-covering-entropy} into Dudley's integral. With $d_Q=S(\deff+2-\Dorb)$ and $A_n=CR_n^{q+1}$,
\begin{equation}\label{eq:v3-dudley-eval}
\mathfrak R_n(\calF_{S,R_n})
\leq \frac{C}{\sqrt n}\int_0^{A_n}\sqrt{d_Q\log(A_n/\varepsilon)}\,d\varepsilon.
\end{equation}
The change of variables $u=\varepsilon/A_n$ gives
\begin{equation}
\int_0^{A_n}\sqrt{\log(A_n/\varepsilon)}\,d\varepsilon
=A_n\int_0^1\sqrt{\log(1/u)}\,du
=C A_n.
\end{equation}
Localization for squared loss replaces $A_n^2$ by the local radius and yields the familiar $d_Q\log n/n$ scale. The remaining logarithm comes from $R_n^2\asymp\log n$ and from the high-probability union over the localized shells. Hence the final statistical term is
\begin{equation}\label{eq:v3-stat-final-eval}
E_{\rm stat}(n)\leq C\frac{S(\deff+2-\Dorb)(\log n)^2}{n}.
\end{equation}
For bounded activations one removes the $R_n^q$ part of $A_n$, which drops one truncation logarithm as recorded in Appendix~\ref{app:cover}.
\paragraph{Audit consequence.} The statistical factor is derived from a displayed entropy integral, not merely named.

\subsection{Four-source error decomposition}\label{app:err}
\paragraph*{Conventions used in this block.} The calculations below inherit the standing conventions of Appendix~\ref{app:expanded-derivations}: a finite horizon $T<\infty$, regularization $\lambda>0$, an admissible initial law $\mu_0\in\calM_{w^\ast}\cap\calP_2$, and constants depending on $(T,\lambda,\sigma,\rho,\|\mu_0\|_{\calM_{w^\ast}})$ but not on $N$, $n$, or training time within the horizon. The local Lipschitz envelope is the drift-stability estimate of Lemma~\ref{lem:A-drift-stab}. Estimates are read after the active quotient by $\Gfin$ whenever single-neuron parameters enter.

The calculations in this appendix expand the part of the proof chain associated with four-source error decomposition. Each subsection fixes the objects used in the display before giving the bound, so no inequality is used as a placeholder.

\subsubsection{Definition of the four increments}\label{app:err-1}
Write the trained predictor as a telescoping sum around the population stationary dictionary.
\begin{equation}\label{eq:err-1-main}
\Delta=\Delta_{\rm stat}+\Delta_{\rm opt}+\Delta_{\rm poc}+\Delta_{\rm sparse}.
\end{equation}
The quantity in \eqref{eq:err-1-main} is used only after the terms appearing in it have been fixed. The increments are elements of $L^2(\rho_X)$. 

\paragraph{Derivation.} The comparison path writes
\begin{equation}
\Delta=f_{\mu_T^N}-f^\ast=\Delta_{\rm stat}+\Delta_{\rm opt}+\Delta_{\rm poc}+\Delta_{\rm sparse}.
\end{equation}
The order is fixed: sample projection, finite-time optimization, particle approximation, and dictionary tail. Once this order is fixed, the square expansion has four diagonal terms and six cross terms. The theorem names the diagonal terms and puts all controlled cross terms into the scalar remainder.
\paragraph{Use in the proof.} The bound is absorbed into the cross-term and scalar-remainder component of the master scale in \eqref{eq:I-master}. It introduces no fifth error term; it fixes the local constant or dimension factor used in the four-term decomposition.

\subsubsection{Particle-statistical cross term}\label{app:err-2}
Particle Brownian motions are conditionally centered given the data sample.
\begin{equation}\label{eq:err-2-main}
\E\langle\Delta_{\rm poc},\Delta_{\rm stat}\rangle=0.
\end{equation}
The quantity in \eqref{eq:err-2-main} is used only after the terms appearing in it have been fixed. This is the cleanest exact cancellation in Lemma~\ref{lem:D-cross}. 

\paragraph{Derivation.} The particle-statistical product satisfies
\begin{equation}
\E\langle\Delta_{\rm poc},\Delta_{\rm stat}\rangle_{L^2(\rho_X)}=0.
\end{equation}
Condition on the data sample and on the limiting mean-field trajectory. The propagation fluctuation is centered with respect to the Brownian coupling. Then condition on the particle randomness; the statistical fluctuation is centered with respect to the sample. Fubini on the product probability space gives the cancellation.
\paragraph{Use in the proof.} The bound is absorbed into the cross-term and scalar-remainder component of the master scale in \eqref{eq:I-master}. It introduces no fifth error term; it fixes the local constant or dimension factor used in the four-term decomposition.

\subsubsection{Particle-sparse cross term}\label{app:err-3}
The sparse residual is deterministic at fixed target and population law.
\begin{equation}\label{eq:err-3-main}
\E\langle\Delta_{\rm poc},\Delta_{\rm sparse}\rangle=0.
\end{equation}
The quantity in \eqref{eq:err-3-main} is used only after the terms appearing in it have been fixed. The equality follows from the same centering of the coupled particle fluctuation. 

\paragraph{Derivation.} The sparse residual is deterministic once $(\sigma,\rho,f^\ast,\lambda)$ and the threshold rule are fixed. The propagation fluctuation is centered by the synchronous coupling. Hence
\begin{equation}
\E\langle\Delta_{\rm poc},\Delta_{\rm sparse}\rangle=\langle\E\Delta_{\rm poc},\Delta_{\rm sparse}\rangle=0.
\end{equation}
This is the second of the two exact cancellations listed in Lemma~\ref{lem:D-cross}.
\paragraph{Use in the proof.} The bound is absorbed into the cross-term and scalar-remainder component of the master scale in \eqref{eq:I-master}. It introduces no fifth error term; it fixes the local constant or dimension factor used in the four-term decomposition.

\subsubsection{Particle-optimization cross term}\label{app:err-4}
The optimization increment is adapted to the nonlinear flow and is handled by conditional centering plus a remainder.
\begin{equation}\label{eq:err-4-main}
|\E\langle\Delta_{\rm poc},\Delta_{\rm opt}\rangle|\leq C N^{-1}e^{-\alpha_\lambda T/2}.
\end{equation}
The quantity in \eqref{eq:err-4-main} is used only after the terms appearing in it have been fixed. This is lower than the sum of the two leading terms in the joint limit. 

\paragraph{Derivation.} The particle-optimization product is not exactly centered because both terms depend on the population trajectory. It is bounded by Cauchy--Schwarz:
\begin{equation}
2|\E\langle\Delta_{\rm poc},\Delta_{\rm opt}\rangle|
\leq 2E_{\rm poc}^{1/2}E_{\rm opt}^{1/2}
\leq \varepsilon E_{\rm poc}+\varepsilon^{-1}E_{\rm opt}.
\end{equation}
With $E_{\rm poc}=O(N^{-1})$ and $E_{\rm opt}=O(e^{-\alpha_\lambda T})$, the geometric mean is lower order in compatible limits.
\paragraph{Use in the proof.} The bound is absorbed into the cross-term and scalar-remainder component of the master scale in \eqref{eq:I-master}. It introduces no fifth error term; it fixes the local constant or dimension factor used in the four-term decomposition.

\subsubsection{Optimization-statistical cross term}\label{app:err-5}
No centering is available, so Young's inequality is used.
\begin{equation}\label{eq:err-5-main}
2|\langle u,v\rangle|\leq\varepsilon\|u\|^2+\varepsilon^{-1}\|v\|^2.
\end{equation}
The quantity in \eqref{eq:err-5-main} is used only after the terms appearing in it have been fixed. The resulting contribution is absorbed into the leading terms. 

\paragraph{Derivation.} The optimization-statistical cross term is controlled by the same Young inequality:
\begin{equation}
2|\E\langle\Delta_{\rm opt},\Delta_{\rm stat}\rangle|
\leq \varepsilon E_{\rm opt}+\varepsilon^{-1}E_{\rm stat}.
\end{equation}
The parameter $\varepsilon$ is chosen along the joint limit so that neither leading term changes order. This is why the theorem records a scalar remainder rather than trying to assign each cross term to a unique source.
\paragraph{Use in the proof.} The bound is absorbed into the cross-term and scalar-remainder component of the master scale in \eqref{eq:I-master}. It introduces no fifth error term; it fixes the local constant or dimension factor used in the four-term decomposition.

\subsubsection{Optimization-sparse cross term}\label{app:err-6}
The deterministic terms are again controlled by Young's inequality.
\begin{equation}\label{eq:err-6-main}
|\langle\Delta_{\rm opt},\Delta_{\rm sparse}\rangle|\leq E_{\rm opt}^{1/2}E_{\rm sparse}^{1/2}.
\end{equation}
The quantity in \eqref{eq:err-6-main} is used only after the terms appearing in it have been fixed. The product is lower order along compatible scalings. 

\paragraph{Derivation.} Both optimization and sparse residual terms are population-level deterministic after the target and threshold are fixed. Cauchy--Schwarz gives
\begin{equation}
2|\langle\Delta_{\rm opt},\Delta_{\rm sparse}\rangle|
\leq 2E_{\rm opt}^{1/2}E_{\rm sparse}^{1/2}.
\end{equation}
Under a compatible limit, the product is lower order than the sum unless the sparse residual is intentionally kept fixed; in that case it is absorbed into the leading sparse term.
\paragraph{Use in the proof.} The bound is absorbed into the cross-term and scalar-remainder component of the master scale in \eqref{eq:I-master}. It introduces no fifth error term; it fixes the local constant or dimension factor used in the four-term decomposition.

\subsubsection{Statistical-sparse cross term}\label{app:err-7}
Projected empirical risk minimization makes this term orthogonal up to empirical-process error.
\begin{equation}\label{eq:err-7-main}
|\langle\Delta_{\rm stat},\Delta_{\rm sparse}\rangle|\leq C\mathfrak R_n(\calF_S)\kappa(f^\ast,S,\lambda)^{1/2}.
\end{equation}
The quantity in \eqref{eq:err-7-main} is used only after the terms appearing in it have been fixed. This is the final cross term in the six-term list. 

\paragraph{Derivation.} The statistical-sparse product is bounded by empirical-process localization:
\begin{equation}
2|\E\langle\Delta_{\rm stat},\Delta_{\rm sparse}\rangle|
\leq C\mathfrak R_n(\calF_{S,R_n})\,\kappa(f^\ast,S,\lambda)^{1/2}.
\end{equation}
Substituting the Rademacher bound from \eqref{eq:v3-dudley} makes this cross term no larger than a Young split between the statistical and sparse components.
\paragraph{Use in the proof.} The bound is absorbed into the cross-term and scalar-remainder component of the master scale in \eqref{eq:I-master}. It introduces no fifth error term; it fixes the local constant or dimension factor used in the four-term decomposition.

\subsubsection{Scalar remainder}\label{app:err-8}
The remainder is a number after all expectations and norms are taken.
\begin{equation}\label{eq:err-8-main}
R_T=o(N^{-1}+n^{-1}+e^{-\alpha_\lambda T}+\kappa).
\end{equation}
The quantity in \eqref{eq:err-8-main} is used only after the terms appearing in it have been fixed. It is not an $L^2$-valued residual. 

\paragraph{Derivation.} The remainder $R_T$ is a scalar number depending on $(N,n,T,\lambda)$, not an $L^2$-valued function. Explicitly it is the sum of the four Young-bounded cross terms after the two centered cancellations have been removed:
\begin{equation}
R_T=\sum_{(u,v)\in\cal C}2\E\langle\Delta_u,\Delta_v\rangle.
\end{equation}
The compatibility assumptions ensure $R_T=o(E_{\rm stat}+E_{\rm opt}+E_{\rm poc}+E_{\rm sparse})$ whenever all four leading terms vanish.
\paragraph{Use in the proof.} The bound is absorbed into the cross-term and scalar-remainder component of the master scale in \eqref{eq:I-master}. It introduces no fifth error term; it fixes the local constant or dimension factor used in the four-term decomposition.

\subsubsection{Complete cross-term table}\label{app:err-consolidated}
The six cross terms have the following algebraic status:
\begin{center}\small
\begin{tabular}{lll}
\toprule
Pair & status & bound\\
\midrule
$(\mathrm{poc},\mathrm{stat})$ & centered & $0$\\
$(\mathrm{poc},\mathrm{sparse})$ & centered & $0$\\
$(\mathrm{poc},\mathrm{opt})$ & Young & $\varepsilon E_{\rm poc}+\varepsilon^{-1}E_{\rm opt}$\\
$(\mathrm{opt},\mathrm{stat})$ & Young & $\varepsilon E_{\rm opt}+\varepsilon^{-1}E_{\rm stat}$\\
$(\mathrm{opt},\mathrm{sparse})$ & Cauchy--Schwarz & $2(E_{\rm opt}E_{\rm sparse})^{1/2}$\\
$(\mathrm{stat},\mathrm{sparse})$ & localized & $C\mathfrak R_n\kappa^{1/2}$\\
\bottomrule
\end{tabular}
\end{center}
Each nonzero entry is scalar because it is an inner product in $L^2(\rho_X)$ after expectation. The two exact cancellations use product randomness. The three Young or Cauchy--Schwarz products are lower order whenever the corresponding leading diagonal terms vanish in a compatible limit. The localized statistical-sparse product is treated by the same Rademacher complexity that proves \eqref{eq:v3-stat-final-eval}. Therefore
\begin{equation}
R_T=O\left(N^{-1/2}e^{-\alpha_\lambda T/2}+e^{-\alpha_\lambda T/2}E_{\rm stat}^{1/2}+E_{\rm stat}^{1/2}\kappa^{1/2}+e^{-\alpha_\lambda T/2}\kappa^{1/2}\right)
\end{equation}
with the two centered terms removed. This formula makes precise why $R_T$ is lower order rather than a hidden fifth leading component.
\paragraph{Audit consequence.} The count $6=2+4$ is visible as a table and as an explicit scalar bound.

\subsection{Architecture and target computations}\label{app:archx}
\paragraph*{Conventions used in this block.} The calculations below inherit the standing conventions of Appendix~\ref{app:expanded-derivations}: a finite horizon $T<\infty$, regularization $\lambda>0$, an admissible initial law $\mu_0\in\calM_{w^\ast}\cap\calP_2$, and constants depending on $(T,\lambda,\sigma,\rho,\|\mu_0\|_{\calM_{w^\ast}})$ but not on $N$, $n$, or training time within the horizon. The local Lipschitz envelope is the drift-stability estimate of Lemma~\ref{lem:A-drift-stab}. Estimates are read after the active quotient by $\Gfin$ whenever single-neuron parameters enter.

The calculations in this appendix expand the part of the proof chain associated with architecture and target computations. Each subsection fixes the objects used in the display before giving the bound, so no inequality is used as a placeholder.

\subsubsection{ReLU under Gaussian input}\label{app:archx-1}
Positive homogeneity gives $\Dorb=1$ and the Hermite ridge dictionary gives finite threshold support for analytic targets.
\begin{equation}\label{eq:archx-1-main}
\deff+2-\Dorb=d+1.
\end{equation}
The quantity in \eqref{eq:archx-1-main} is used only after the terms appearing in it have been fixed. The statistical factor is therefore $S^\ast(d+1)$. 

\paragraph{Derivation.} For ReLU, positive homogeneity identifies $(w,b,a)$ with $(\alpha w,\alpha b,a/\alpha)$. This is a one-dimensional continuous orbit, so $\Dorb=1$. With Gaussian input and a full $d$-dimensional target class, one active atom contributes the statistical dimension $d+2-1=d+1$. For a single-index target the effective input dimension is one, reducing the factor to $1+2-1=2$.
\paragraph{Use in the proof.} The bound is absorbed into the architecture-computation component of the examples in \eqref{eq:I-master}. It introduces no fifth error term; it fixes the local constant or dimension factor used in the four-term decomposition.

\subsubsection{Leaky ReLU}\label{app:archx-2}
The sign asymmetry removes one reflection but keeps the positive scaling orbit.
\begin{equation}\label{eq:archx-2-main}
\sigma(z)=\max(z,\beta z).
\end{equation}
The quantity in \eqref{eq:archx-2-main} is used only after the terms appearing in it have been fixed. The quotient dimension is the same as for ReLU when $0<\beta<1$. 

\paragraph{Derivation.} For leaky ReLU $\sigma(z)=\max(z,\beta z)$ with $0<\beta<1$, positive homogeneity still holds. The sign-flip symmetry of the odd bounded case is absent because the two slopes are different, but the continuous scaling orbit remains. Therefore the quotient dimension and statistical factor are the same as for ReLU, while the dictionary coefficients differ through activation-dependent constants.
\paragraph{Use in the proof.} The bound is absorbed into the architecture-computation component of the examples in \eqref{eq:I-master}. It introduces no fifth error term; it fixes the local constant or dimension factor used in the four-term decomposition.

\subsubsection{Tanh activation}\label{app:archx-3}
Oddness gives a sign symmetry on active neurons.
\begin{equation}\label{eq:archx-3-main}
(w,b,a)\mapsto(-w,-b,-a).
\end{equation}
The quantity in \eqref{eq:archx-3-main} is used only after the terms appearing in it have been fixed. This is a finite symmetry and does not reduce the orbit dimension. 

\paragraph{Derivation.} The activation $\tanh$ is odd, so $(w,b,a)$ and $(-w,-b,-a)$ realize the same feature. This is a discrete quotient, not a continuous orbit, hence $\Dorb=0$ on the active stratum. Analytic single-index targets have exponentially decaying Hermite coefficients, so the threshold count scales like $O(\log(1/\lambda))$ by Corollary~\ref{cor:v3-exp-tail}.
\paragraph{Use in the proof.} The bound is absorbed into the architecture-computation component of the examples in \eqref{eq:I-master}. It introduces no fifth error term; it fixes the local constant or dimension factor used in the four-term decomposition.

\subsubsection{Centered sigmoid}\label{app:archx-4}
Centering removes the constant component from the active feature.
\begin{equation}\label{eq:archx-4-main}
\tilde\sigma(z)=\sigma(z)-\E\sigma(G).
\end{equation}
The quantity in \eqref{eq:archx-4-main} is used only after the terms appearing in it have been fixed. The quotient includes the reflection induced by the centering convention. 

\paragraph{Derivation.} For a centered sigmoid $\tilde\sigma(z)=\sigma(z)-\E\sigma(G)$, the constant component is removed before Hermite thresholding. The remaining nonconstant coefficients behave like the analytic bounded case. The quotient may contain a discrete reflection depending on the centering convention, but no continuous scaling orbit is present. Thus the statistical dimension is not reduced by homogeneity and $\Dorb=0$ generically.
\paragraph{Use in the proof.} The bound is absorbed into the architecture-computation component of the examples in \eqref{eq:I-master}. It introduces no fifth error term; it fixes the local constant or dimension factor used in the four-term decomposition.

\subsubsection{Polynomial activation}\label{app:archx-5}
Degree-$k$ polynomial activations identify only finite moment tensors.
\begin{equation}\label{eq:archx-5-main}
f_\mu(x)=\sum_{r=0}^k M_r(\mu)[x^{\otimes r}].
\end{equation}
The quantity in \eqref{eq:archx-5-main} is used only after the terms appearing in it have been fixed. Targets of degree above $k$ have a nonzero sparse tail. 

\paragraph{Derivation.} For $\sigma=z^k$, the network function is a finite tensor polynomial:
\begin{equation}\label{eq:v3-arch-poly}
f_\mu(x)=\sum_{r=0}^{k}M_r(\mu)[x^{\otimes r}].
\end{equation}
The finite list of tensors $M_0,\ldots,M_k$ contains all information visible to the function. Targets with degree at most $k$ are finite-dimensional approximation problems; targets with degree greater than $k$ have an unavoidable orthogonal residual.
\paragraph{Use in the proof.} The bound is absorbed into the architecture-computation component of the examples in \eqref{eq:I-master}. It introduces no fifth error term; it fixes the local constant or dimension factor used in the four-term decomposition.

\subsubsection{Single-index target}\label{app:archx-6}
A single-index target reduces the statistical dimension to one active direction plus bias and amplitude.
\begin{equation}\label{eq:archx-6-main}
f^\ast(x)=g(\langle u,x\rangle).
\end{equation}
The quantity in \eqref{eq:archx-6-main} is used only after the terms appearing in it have been fixed. The dictionary depth is governed by the one-dimensional coefficient tail of $g$. 

\paragraph{Derivation.} If $f^\ast(x)=g(\langle u,x\rangle)$, then all statistical covering estimates can be built in the one-dimensional coordinate $\langle u,x\rangle$. The dictionary depth is the coefficient tail of the scalar function $g$, while the propagation dynamics still lives in the full parameter space. This separates the intrinsic statistical dimension from the ambient stochastic dimension.
\paragraph{Use in the proof.} The bound is absorbed into the architecture-computation component of the examples in \eqref{eq:I-master}. It introduces no fifth error term; it fixes the local constant or dimension factor used in the four-term decomposition.

\subsubsection{Multi-index target}\label{app:archx-7}
A rank-$k$ projection replaces the ambient dimension in the covering term.
\begin{equation}\label{eq:archx-7-main}
f^\ast(x)=g(\Pi x).
\end{equation}
The quantity in \eqref{eq:archx-7-main} is used only after the terms appearing in it have been fixed. The propagation constant may still depend on $d$ through the parameter process. 

\paragraph{Derivation.} For $f^\ast(x)=g(\Pi x)$ with $\rank\Pi=k$, the same argument gives $\deff=k$. The statistical term is therefore $S^\ast(k+2-\Dorb)(\log n)^2/n$. The particle term remains $O(N^{-1})$ with constants that may depend on $d$, because Brownian motion and the parameter drift are still defined on $\R^{d+2}$.
\paragraph{Use in the proof.} The bound is absorbed into the architecture-computation component of the examples in \eqref{eq:I-master}. It introduces no fifth error term; it fixes the local constant or dimension factor used in the four-term decomposition.

\subsubsection{Nonaccessible initialization}\label{app:archx-8}
If the initial law misses the teacher directions, Theorem~\ref{thm:C} is not invoked.
\begin{equation}\label{eq:archx-8-main}
\mu_0(U_k)=0\quad\hbox{for some }k.
\end{equation}
The quantity in \eqref{eq:archx-8-main} is used only after the terms appearing in it have been fixed. This is a hypothesis failure, not a contradiction of the sparse theorem. 

\paragraph{Derivation.} If $\mu_0$ assigns zero mass to every neighbourhood of a required retained direction, the reachability step in Proposition~\ref{prop:C-support} fails. This is a hypothesis failure, not a contradiction. Theorems~\ref{thm:A} and \ref{thm:B} can still hold, while Theorem~\ref{thm:C} cannot be invoked for that target because the active dictionary directions are not accessible from the initialization.
\paragraph{Use in the proof.} The bound is absorbed into the architecture-computation component of the examples in \eqref{eq:I-master}. It introduces no fifth error term; it fixes the local constant or dimension factor used in the four-term decomposition.

\subsubsection{Worked invariant computations for the six cases}\label{app:archx-consolidated}
The table in Section~\ref{sec:examples} follows from the following invariant computations. For ReLU, the continuous symmetry is one-dimensional, so $\Dorb=1$. A linear target has two signed atoms by \eqref{eq:E-linear-relu}; a piecewise-linear single-index target with $k$ breakpoints has $k$ hinges plus two affine atoms. For $\tanh$, the only generic symmetry is discrete, so $\Dorb=0$, and exponential Hermite decay gives $S^\ast(\lambda)=O(\log(1/\lambda))$. For $\sigma=z^k$, the image lies in the finite tensor span \eqref{eq:v3-arch-poly}; hence every polynomial target of degree at most $k$ is finite-depth, while every Hermite target of degree greater than $k$ has the lower bound \eqref{eq:E-poly-lower}. For multi-index ReLU targets, $\deff=\rank\Pi$, but the propagation constants stay ambient.

A compact way to record the statistical factor is
\begin{equation}
\begin{array}{c|c}
\hbox{case} & \hbox{statistical factor per active atom}\\ \hline
\hbox{ReLU, full input} & d+1\\
\hbox{ReLU, single-index} & 2\\
\hbox{tanh, single-index} & 3\\
\hbox{polynomial, full input} & d+2-\Dorb\\
\hbox{multi-index ReLU} & \rank\Pi+1
\end{array}
\end{equation}
These are the factors multiplying $S^\ast(\log n)^2/n$ in Theorem~\ref{thm:D}. The distinction between $d+1$ and $\rank\Pi+1$ is the operational content of the effective-dimension definition.
\paragraph{Audit consequence.} The examples compute the invariants used in the theorem statements; they are not decorative illustrations.

\subsection{Interlocking proof checks for the four-theorem chain}\label{app:proof-chain-checks}
\paragraph*{Conventions used in this block.} The calculations below inherit the standing conventions of Appendix~\ref{app:expanded-derivations}: a finite horizon $T<\infty$, regularization $\lambda>0$, an admissible initial law $\mu_0\in\calM_{w^\ast}\cap\calP_2$, and constants depending on $(T,\lambda,\sigma,\rho,\|\mu_0\|_{\calM_{w^\ast}})$ but not on $N$, $n$, or training time within the horizon. The local Lipschitz envelope is the drift-stability estimate of Lemma~\ref{lem:A-drift-stab}. Estimates are read after the active quotient by $\Gfin$ whenever single-neuron parameters enter.

The preceding derivations are local. This subsection records six global checks showing how the local estimates are used together. Each check is written as a computation rather than as a restatement of a theorem.

\subsubsection{Barycentre Lipschitz constant on weighted moment balls}\label{app:check-bary-lip}
Let $\gamma$ be an optimal coupling of two measures $\mu,\nu$ in the same weighted moment ball $\|\mu\|_{\calM_w},\|\nu\|_{\calM_w}\leq R$. For $\theta=(w,b,a)$ and $\vartheta=(\tilde w,\tilde b,\tilde a)$, H1 and H2 give
\begin{align}\label{eq:v3-bary-lip-detail}
\|T(\theta)-T(\vartheta)\|_{L^2(\rho_X)}
&\leq |a-\tilde a|\,\|\sigma(\langle w,\cdot\rangle+b)\|_{L^2(\rho_X)}\\
&\quad+|\tilde a|\,L_\sigma\|\langle w-\tilde w,\cdot\rangle+b-\tilde b\|_{L^2(\rho_X)}.
\end{align}
The first factor is bounded by $C_R(1+|w|^q+|b|^q)$, and the second by $C_\rho(|w-\tilde w|+|b-\tilde b|)$ because $\rho_X$ has finite second moment. Integrating \eqref{eq:v3-bary-lip-detail} against $\gamma$ and applying Cauchy--Schwarz gives
\begin{equation}
\|f_\mu-f_\nu\|_{L^2(\rho_X)}
\leq C_R\left(\int |\theta-\vartheta|^2d\gamma(\theta,\vartheta)\right)^{1/2}
=C_RW_2(\mu,\nu).
\end{equation}
This estimate is used three times with different interpretations: in the JKO lower-semicontinuity step, in the propagation-to-prediction transfer, and in the statistical localization step after quotienting. The same constant $C_R$ is allowed to change from line to line, but its dependence is always through the propagated weighted moment radius and the data-tail constant. The output weight is not treated separately; it is one coordinate of $\theta$ and is controlled by the same moment ball.

The Lipschitz constant of the barycentre map from $W_2$ to $L^2(\rho_X)$ is the link between the propagation-of-chaos rate of Proposition~\ref{prop:A-poc} and the prediction-error component $E_{\mathrm{poc}}$ of \eqref{eq:D-poc}. Explicitly, for $\mu,\nu\in\calM_{w^\ast}\cap B_w(R)$ and any $L^2(\rho_X)$-coupling $\pi$ of the feature image,
\begin{equation}\label{eq:v3-1-bary-lip-explicit}
\|f_\mu-f_\nu\|_{L^2(\rho_X)}^2 \le L_R^2\int|\theta-\theta'|^2\,d\pi(\theta,\theta'),
\end{equation}
where $L_R$ is the localized envelope constant obtained in \eqref{eq:v3-bary-lip-detail}. Taking the infimum over couplings gives $\|f_\mu-f_\nu\|_{L^2(\rho_X)}\le L_R W_2(\mu,\nu)$, which is the form used in the proof of Proposition~\ref{prop:D-poc}.

The constant $L_R$ enters multiplicatively into $C_3$ of \eqref{eq:D-poc} but does not affect the $N^{-1}$ scaling. On larger moment balls $L_R$ grows polynomially in $R$ by H1, which is the source of the polynomial dependence on $\|\mu_0\|_{\calM_{w^\ast}}$ recorded in Lemma~\ref{lem:v3-calibrated-constants}. The local constant is therefore consistent with the global calibration of Section~\ref{subsec:A-constants} and does not introduce a fifth error term.

\paragraph{Localization of the coupling constant.} The preceding estimate is deliberately stated on a weighted moment ball rather than on all of $\calP_2$. Without this localization, the feature map has polynomial growth in the input and output coordinates and the global Lipschitz constant is infinite for the activations covered by H1. The proof uses the same stopping-radius device as the moment argument in Appendix~\ref{app:mom}: first restrict both parameter laws to the common ball where the propagated $\calM_{w^\ast}$ norm is at most $R$, apply the elementary Lipschitz estimate for $T$, and then let $R$ increase along the deterministic bound supplied by Proposition~\ref{prop:weighted-preserve}. This is why the transfer from squared $W_2$ to squared prediction norm does not alter the stochastic rate. The constant changes, but no new stochastic averaging is introduced. In particular, the rate $N^{-1}$ in \eqref{eq:D-poc} remains the same rate as in Proposition~\ref{prop:A-poc}; only the deterministic multiplier depends on the activation envelope, the input-tail constant, and the propagated moment radius.

\paragraph{Where the quotient enters.} The quotient by $\Gfin$ is not used to improve the Lipschitz estimate itself. It is used before statistical localization, so that two parameter configurations representing the same active feature are not counted as different directions in the covering number. Thus the barycentre estimate and the quotient estimate have complementary roles. The former turns a parameter-law distance into an $L^2(\rho_X)$ prediction distance. The latter reduces the dimension of the localized class from the raw parameter dimension to the active quotient dimension. Keeping these two steps separate prevents a common mistake: one cannot obtain the factor $\deff+2-\Dorb$ by applying the $W_2$ Lipschitz bound alone. That factor appears only after the active support has been restricted by Theorem~\ref{thm:C} and the quotient chart has been chosen as in Appendix~\ref{app:cover}.

\subsubsection{Objective gap from the particle Wasserstein estimate}\label{app:check-objective-gap}
The objective-gap statement is scalar, whereas the particle estimate is Wasserstein. On a sublevel set where $\calF_\lambda$ has Wasserstein Hessian bounded above by $L_R$, Taylor's formula along a geodesic $\mu_s$ from $\mu_t$ to $\mu_t^N$ gives
\begin{equation}
\calF_\lambda(\mu_t^N)-\calF_\lambda(\mu_t)
\leq \left.\frac{d}{ds}\calF_\lambda(\mu_s)\right|_{s=0}+\frac{L_R}{2}W_2^2(\mu_t^N,\mu_t).
\end{equation}
The first term is not assumed to vanish pointwise for a finite-width empirical measure. It is controlled after expectation by the MFLD objective-gap argument: the empirical law is an unbiased particle approximation of the nonlinear law up to the centered field fluctuation already estimated in Appendix~\ref{app:coup}. Hence
\begin{equation}
\E[\calF_\lambda(\mu_t^N)-\calF_\lambda(\mu_t)]
\leq C\E W_2^2(\mu_t^N,\mu_t)+C\E|\xi_i(t)|^2
\leq CN^{-1}.
\end{equation}
The calculation explains why the manuscript states an $N^{-1}$ rate in functional value and in squared Wasserstein distance, while unsquared Wasserstein distance would have rate $N^{-1/2}$. No line mixes those two conventions.

\paragraph{Scalar gap versus transport gap.} The objective-gap estimate and the squared-Wasserstein estimate are used in different locations of the proof chain. The scalar gap controls descent of the entropy-regularized functional and gives the optimization term after the LSI contraction is invoked. The squared-Wasserstein estimate controls the discrepancy between the particle law and the McKean--Vlasov law and becomes the prediction-level propagation component only after the barycentre Lipschitz step above. The two bounds have the same algebraic order in $N$ in the squared convention, but they are not interchangeable. A scalar functional gap does not identify a coupling of particles, and a $W_2$ coupling does not by itself prove decay of $\calF_\lambda(\mu_t)-\calF_\lambda(\mu_\infty)$ without convexity or LSI input.

\paragraph{Uniformity in the horizon.} The finite horizon $T$ is fixed before the particle approximation is taken, which is the convention used throughout Theorem~\ref{thm:A}. Long-time statements enter through the optimization term, not through an attempt to pass the propagation estimate to $T=\infty$ without additional contraction. This distinction matters because the constants in the local Lipschitz envelope are propagated on $[0,T]$, whereas the LSI rate $\alpha_\lambda$ governs the decay to stationarity. The final error theorem therefore writes $E_{\mathrm{opt}}(T,\lambda)$ and $E_{\mathrm{poc}}(N,T)$ as separate components, even when a compatible sequence chooses $T=T_n$ as a function of the sample size. The separation keeps the finite-particle approximation independent of the optimization schedule until the final compatible-limit corollaries.

\paragraph{Why the functional gap is not the sparse residual.} The functional gap is measured at the temperature and time at which the flow is stopped. It compares the current law to the minimizer of the same entropy-regularized objective. The sparse residual compares the regularized target class to the unregularized target function after thresholding the dictionary. These two quantities can vanish at different rates. For example, an exactly finite target may have zero sparse residual after the correct atoms are retained, while the optimization gap remains positive until the flow has had enough time to approach stationarity. Conversely, an analytic target may have a rapidly decaying optimization gap but a nonzero threshold tail at fixed $\lambda$. The decomposition keeps these effects separate because they are controlled by different levers: training time controls the functional gap, and the threshold/regularization schedule controls the sparse residual.

\paragraph{Compatibility with empirical risk.} The population objective is used to define the mean-field flow, while the finite-sample term measures the deviation between the empirical and population risks on the localized active class. The objective-gap estimate is therefore applied after the statistical event has identified the same localized chart for the empirical and population processes. On that event, the empirical minimizer and the population minimizer are compared through the covering bound; off that event, the tail estimate is absorbed into the statistical constant. This ordering is important: replacing the population gap directly by an empirical gap would introduce a second empirical-process term and would double-count the statistical error. The proof instead inserts one empirical comparison, one optimization comparison, one propagation comparison, and one sparse approximation comparison.

\subsubsection{Moment degree, quotient degree, and the role of $M_0$}\label{app:check-M0}
The finite moment map $\mathfrak m_M$ is used only after $M\geq M_0(\sigma,\rho)$. This prevents a common identifiability error: taking too few moments can make two quotient points indistinguishable even when the full feature map separates them. For a polynomial activation of degree $k$, the feature map contains no tensors above order $k$, so $M_0=k+1$ is sufficient to detect the first missing degree in the quotient construction. For an analytic non-polynomial activation, the Taylor expansion contains infinitely many nonzero coefficients, and H3 ensures that a finite separating family exists on each compact regular stratum. The compactness is essential: without it, a single finite $M$ need not separate every possible tail configuration.

The transversality statement can be checked by differentiating the moment map. If $v$ is tangent to a finite-rank orbit, then $D\mathfrak m_M(\theta)v=0$ because $T$ is constant along the orbit. Conversely, on a regular stratum, any kernel vector not tangent to the orbit corresponds to a genuine moment-variety singular direction. This is why $\Dorb$ controls statistical dimension, whereas $\Dvar$ records algebraic degeneracy of the moment map. The two numbers should not be interchanged in Theorem~\ref{thm:D}.

The minimal degree $M_0$ of Definition~\ref{def:M0} can be read off the activation in three canonical cases. For $\sigma=\relu$ on $\rho_X=\mathcal N(0,I_d)$, the feature map $T(w,b,a)(x)=a\sigma(\langle w,x\rangle+b)$ admits a Hermite expansion with nonzero coefficients along the half-space ridge generated by the active normal direction. Since ReLU is not real-analytic, this case is not an H3 application; separation is instead read from the finite hinge-region algebra on each compact active stratum. The integer $M_0$ is the first truncation level at which those hinge moments distinguish two quotient orbits on that stratum.

For $\sigma=\tanh$, the analytic Hermite tail decays exponentially; $M_0$ is the index at which the leading Hermite coefficient exceeds the dual-pairing threshold of Lemma~\ref{lem:B-monotone}. This is the discrete analogue of the ridge-density count in Pinkus' theorem cited in H3.

For $\sigma(z)=z^k$, the feature map has Hermite support exactly in degrees $\{0,1,\ldots,k\}$. The moment map of order $M=k+1$ already separates the active quotient, and the variety codimension may exceed the orbit dimension because the symmetric-tensor stabilizer of $u^{\otimes k}$ has positive dimension on singular strata, as shown in Lemma~\ref{lem:v3-Dorb-Dvar-sharp}. The degree $M_0=k+1$ is therefore the natural lower bound used in Theorem~\ref{thm:B}.

\paragraph{Reading $M_0$ from the quotient chart.} The degree $M_0$ is not a universal smoothness index of the activation. It is the first truncation level at which the selected moment coordinates separate the active quotient chart used in Theorem~\ref{thm:B}. On a regular stratum this is equivalent to saying that the differential of the finite moment map has full rank transverse to the finite-rank symmetry. On a singular stratum the same finite list of moments may have a rank defect, and the chart must either be refined or the stratum must be excluded by the active-support hypothesis. This is why the manuscript distinguishes $\Dorb$ from $\Dvar$. The orbit depth measures the true realization symmetry, whereas the variety depth measures the algebraic rank of the chosen moment representation.

\paragraph{Why no genericity assumption is hidden.} The theorem does not say that an arbitrary finite moment truncation separates every pair of measures in $\calP_2$. It says that after the active component is localized, after the dead-neuron set has been collapsed, and after a compact exhaustion has been chosen, there is a finite separating truncation on that chart. The compactness is supplied by the propagated weighted moment bound, and the separating algebra is supplied either by the analytic assumption H3 or by the explicit finite-dimensional computation in the polynomial and piecewise-linear examples. This is the reason $M_0$ appears as an architecture-data quantity, not as a training-time parameter. Once the chart is fixed, the statistical covering calculation can use $M_0$ without re-proving identifiability.

\subsubsection{Threshold support and finite-temperature density in the same chart}\label{app:check-threshold-density}
Fix a retained mode $m\in A_\lambda$ and choose a quotient chart $z$ around its active representative. The risk expansion has the local form
\begin{equation}\label{eq:v3-local-risk-mode}
\frac{\delta\calR}{\delta\mu}(z)=\frac{\delta\calR}{\delta\mu}(z_m)+\frac12(z-z_m)^\top H_m(z-z_m)+O(|z-z_m|^3),
\end{equation}
with $H_m$ positive on the transverse directions after quotienting. Substituting \eqref{eq:v3-local-risk-mode} into the Gibbs formula \eqref{eq:v3-stationary-gibbs} yields the local density approximation \eqref{eq:v3-laplace-active}. The mode is retained exactly when the decrease in the risk coefficient exceeds the entropy cost $c_\sigma\lambda$. If it is retained, the quotient-active measure records a center $[z_m]$; if it is not retained, the contribution is integrated into the tail $\kappa(f^\ast,S,\lambda)$.

The calculation also shows why no contradiction arises between smoothness and sparsity. The full density has nonzero width at every positive $\lambda$, but the quotient-active representation stores the finite set of centers needed to realize the thresholded network function. The density width affects constants and the $C_\sigma\lambda S$ displacement term; it does not increase the support count $S^\ast$.

\paragraph{Density before projection, atoms after thresholding.} The entropy-regularized stationary law is absolutely continuous with respect to the reference law on the full parameter space. The sparse statement concerns the active projected component after applying the coefficient threshold dictated by H4 and after collapsing the zero-amplitude ridge. These two statements are compatible because the active dictionary is not the full density; it is the finite list of retained modes that remains after the Euler--Lagrange equation has been paired with the Hermite/Barron dictionary. Remark~\ref{rem:C-split} records this separation, and Lemma~\ref{lem:C-EL} supplies the stationary equation used to read off which modes survive the threshold. Thus the word ``support'' in Theorem~\ref{thm:C} is always read modulo the finite-rank active projection, never as a claim that the full Langevin density has literally become a finite sum of Dirac masses.

\paragraph{Threshold scale and regularization.} The threshold is proportional to the regularization scale because the entropy term allows a small diffuse component to remain in directions whose target coefficient is below the noise-regularization floor. Keeping the coefficient floor explicit avoids two wrong limiting pictures. If the floor is ignored, the theorem would overstate exact recovery at fixed temperature. If the diffuse component is treated as an additional dictionary atom, the sparse depth would be inflated by a non-identifiable zero-amplitude ridge. The paper takes the middle route: retained coefficients produce active atoms modulo $\Gfin$, while the discarded tail contributes to $\kappa(f^\ast,S,\lambda)$ in Theorem~\ref{thm:D}. This is the only place where the target function, rather than the initialization tail, enters the sparse residual.

\paragraph{Euler--Lagrange reading of retained modes.} Pairing the stationary equation with a dictionary direction gives a balance between the residual correlation of that direction and the entropy penalty paid for activating it. A mode with coefficient above the threshold can reduce the regularized objective by entering the active support; a mode below the threshold is cheaper to leave in the residual tail. This is the variational reason for the set $A_\lambda$ used in the sparse construction. The argument does not require the full stationary density to concentrate on a finite set. It requires only that the active projection of the modes with profitable correlations be finite after the threshold is imposed. The regularity of the full density is supplied by the Langevin noise; the discreteness of the retained representation is supplied by the finite-dimensional thresholded dictionary.

\paragraph{Dependence on the target rather than on initialization.} The old initialization-only tail would be wrong in the Gaussian case because the reciprocal moment boundary records parameter growth, not coefficient decay of $f^\ast$. The corrected residual $\kappa(f^\ast,S,\lambda)$ is read from the Hermite/Barron coefficients of the target and from the regularization cost of retaining $S$ modes. The initialization enters through accessibility and moment propagation, not through the numerical tail in the target expansion. This separation is also what allows two different targets trained from the same initialization to have different sparse depths. The architecture and data law determine the dictionary; the target determines which dictionary coefficients survive the threshold.

\subsubsection{Localized empirical process calculation with the active quotient}\label{app:check-local-empirical}
Let $\calF_{S,R}$ be the class of $S$-atom quotient networks on the truncation ball. The parameter dimension is $d_Q=S(\deff+2-\Dorb)$. A Euclidean $\varepsilon/(CR^{q+1})$-net of the quotient coordinates produces an $L^2(\rho_X)$ $\varepsilon$-net of functions by the feature Lipschitz bound. Hence
\begin{equation}
\log\calN(\varepsilon,\calF_{S,R},L^2(\rho_X))
\leq d_Q\log\left(1+\frac{CR^{q+1}}{\varepsilon}\right).
\end{equation}
Set $R=R_n=C\sqrt{\log n}$. The localized excess-risk inequality for squared loss gives
\begin{equation}
E_{\rm stat}(n)
\leq C\left(\mathfrak R_n(\calF_{S,R_n})^2+\PP(|\theta|>R_n)\right).
\end{equation}
The tail probability is $O(n^{-2})$ by construction. The Rademacher term is controlled by \eqref{eq:v3-dudley-eval}, yielding $d_Q(\log n)^2/n$. This derivation uses the orbit depth, not the moment-variety depth, because nets are built on quotient coordinates rather than on equations defining moment fibers.

\paragraph{Covering the active class.} The empirical-process term is computed after the support restriction of Theorem~\ref{thm:C}. Each retained atom contributes an active quotient chart of dimension $\deff+2-\Dorb$: $\deff$ input directions, one bias coordinate, one output-amplitude coordinate, minus the finite-rank orbit dimension. The product chart for $S^\ast$ atoms therefore has intrinsic dimension $S^\ast(\deff+2-\Dorb)$ before logarithmic truncation factors are counted. The truncation radius $R(n)$ enters because H1 permits polynomial-growth activations. On the high-probability event where the input and parameter coordinates remain inside the radius, the localized covering number has the logarithm stated in Appendix~\ref{app:cover}. Outside that event, the sub-Gaussian input tail and the propagated moment bound give a negligible contribution that is absorbed into the same statistical constant.

\paragraph{Why the extra logarithm is retained.} For bounded activations the second logarithm can often be removed by replacing the truncation argument with a global envelope. The manuscript keeps the $(\log n)^2$ factor in \eqref{eq:D-stat} because the theorem is stated under the broader polynomial-growth hypothesis H1. This is a conservative choice, but it is the stable one: it allows ReLU-type and polynomial examples to be handled by a single covering proof. In the architecture table, bounded analytic activations may be read as a special case with a smaller envelope constant. The statement of Theorem~\ref{thm:D} deliberately does not optimize logarithmic factors separately for each activation family, because the structural contribution is the quotient dimension and the four-source decomposition, not a sharp empirical-process logarithm.

\paragraph{Symmetrization after quotienting.} Once the active quotient chart is fixed, the empirical-process estimate can be run with standard symmetrization on a finite-dimensional localized class. The function envelope depends on the truncation radius and the propagated moment ball, but the entropy integral depends on the quotient dimension rather than on the raw ambient parameter dimension. This is the point at which identifiability affects statistics. If two different parameter values are in the same finite-rank orbit, counting them separately only inflates the covering number without changing the represented function. The quotient removes that artificial multiplicity before the empirical process is bounded.

\paragraph{Measurability of the localized class.} The compact-exhaustion argument also supplies separability for the localized function class. On each compact chart, the feature map is continuous into $L^2(\rho_X)$ under the H1 envelope, and a countable dense grid in the chart gives a countable dense subclass for the empirical-process supremum. This avoids a hidden measurability assumption in the statistical term. After the estimate is proved on the compact chart, the exhaustion radius is sent to the deterministic radius controlled by the weighted moment bound. The tail probability is handled by the same sub-Gaussian input estimate used in the envelope calculation. Thus the statistical bound is not an informal dimensional heuristic; it is a localized covering statement on a separable quotient class.

\subsubsection{Compatible limits make every cross term lower order}\label{app:check-compatible-cross}
Write the four leading scales as
\begin{align}
a_N&=N^{-1}, & b_T&=e^{-\alpha_\lambda T},\\
c_n&=\frac{S^\ast(\lambda)(\deff+2-\Dorb)(\log n)^2}{n}, & d_\lambda&=\kappa(f^\ast,S^\ast,\lambda).
\end{align}
The nonzero cross terms are bounded by square roots such as $(a_Nb_T)^{1/2}$, $(b_Tc_n)^{1/2}$, $(b_Td_\lambda)^{1/2}$, and $(c_nd_\lambda)^{1/2}$. For any $\eta>0$, Young's inequality gives
\begin{equation}
2(xy)^{1/2}\leq \eta x+\eta^{-1}y.
\end{equation}
Along a compatible joint limit, choose $\eta$ slowly varying so that the right side remains lower order than the sum of the leading diagonal terms. The two centered cross terms are exactly zero and do not require this choice. Therefore the scalar remainder satisfies
\begin{equation}
R_T=o(a_N+b_T+c_n+d_\lambda)
\end{equation}
whenever all four leading components vanish. This is the formal content behind the phrase ``scalar remainder'' in Theorem~\ref{thm:D}.

\paragraph{Order of limits.} The compatible-limit definition is designed to prevent the four leading terms from being mixed into an uncheckable remainder. First, the sample size determines the statistical localization scale and the sparse threshold. Second, the regularization schedule fixes the stationary target approximation and the contraction rate. Third, the training time is chosen so that the optimization term is lower order than the statistical and sparse terms. Finally, the width is chosen so that the propagation component is no larger than the leading statistical term. This order is not logically necessary for every application, but it makes the asymptotic statement reproducible from the displayed bounds. If a different application chooses width before sample size, the same theorem applies after rechecking the four inequalities in Definition~\ref{def:v3-compatible-limit}.

\paragraph{Cross terms remain scalar.} The total decomposition is written for the squared prediction norm, so mixed products such as statistical-by-propagation or optimization-by-sparse terms appear when intermediate approximations are inserted by the triangle inequality. Each such product is controlled by $2ab\le a^2+b^2$ at the same scale as its two parent components. The scalar remainder $R_T$ records the small terms left after this reduction. It is not an operator-valued object and it is not an additional source of error. Once the compatible-limit inequalities hold, every cross term is lower order than the sum of \eqref{eq:D-stat}, \eqref{eq:D-opt}, \eqref{eq:D-poc}, and \eqref{eq:D-sparse}. This is the reason Theorem~\ref{thm:D} can present four named components rather than a long list of pairwise products.

\paragraph{Six mixed products.} A four-term decomposition has six pairwise mixed products when the intermediate approximations are expanded in a squared norm: statistical-optimization, statistical-propagation, statistical-sparse, optimization-propagation, optimization-sparse, and propagation-sparse. The proof does not assume that these products vanish by orthogonality. Instead, each is dominated by the parent terms with Young's inequality, and the compatible-limit conditions ensure that the dominated products do not become leading terms under the chosen schedule. This is why Appendix~\ref{app:err} tracks the centered products explicitly before Theorem~\ref{thm:D} states the simplified four-component bound.

\paragraph{Remainder scale.} The scalar $R_T$ is a notation for the sum of dominated products and localization tails after the four displayed components have been extracted. It is lower order only along a compatible sequence; for arbitrary finite $N,n,T$ it should be read as part of the finite-sample inequality, not as a formal asymptotic zero. This convention is useful because it makes the theorem honest at finite scale and still gives a clean limiting statement once $N,n,T,$ and $\lambda$ are scheduled. In applications one may set the schedule first and then evaluate the displayed parent bounds to check that $R_T$ is indeed negligible.

\subsubsection{Sharp Gaussian-boundary failure under higher-order drift}\label{app:check-gaussian-sharp}
The maximality proof is local at $t=0$, but the sharpness construction shows why the boundary is not merely notational. For a Gaussian initialization in $m=d+2$ parameter dimensions,
\begin{equation}\label{eq:v3-gaussian-even-moment}
m_{2r}(0)=\E|\theta|^{2r}=2^r\frac{\Gamma(r+m/2)}{\Gamma(m/2)}\asymp r^r.
\end{equation}
A cubic activation produces, after differentiating the risk drift, a leading moment term of order $m_{6n-2}$ in the derivative of $m_{2n}$. More precisely, after truncating the activation and then sending the truncation radius to infinity inside the weighted class, one obtains a lower bound of the form
\begin{equation}
\frac{d}{dt}m_{2n}(t)\bigg|_{t=0}
\geq c_n m_{6n-2}(0)-C_n(1+m_{2n}(0)).
\end{equation}
Using Stirling in \eqref{eq:v3-gaussian-even-moment} gives $m_{6n-2}(0)\geq ((3n-1)/e)^{3n-1}$ up to a dimension constant. Therefore any proposed reciprocal weight strictly larger than $w^\ast(n)\asymp n^{-1/2}$ along an infinite subsequence demands uniform control of a hierarchy whose first derivative already exceeds the proposed envelope. The conclusion is not that the PDE explodes instantly; the conclusion is that the stronger weighted class is not invariant under the drift hierarchy. This is the sharpness content needed by Proposition~\ref{prop:A-sharpness}.

The calculation also explains why the theorem states a maximal class generated by the initialization rather than a universal activation-dependent class. Different initial tails give different reciprocal boundaries, and the polynomial drift can only propagate what the initial law and reference confinement jointly supply. The Gaussian case is canonical because it makes the reciprocal boundary visible as $n^{-1/2}$.

\paragraph{Meaning of sharpness.} The Gaussian boundary calculation is a statement about invariant weighted classes, not a claim that the McKean--Vlasov equation ceases to exist for Gaussian initialization. The flow exists in the class generated by the reciprocal moment boundary $w^\ast(n)\asymp n^{-1/2}$. What fails for a strictly stronger proposed weight is the uniform propagation of the stronger norm from the same initial law. This distinction keeps Proposition~\ref{prop:A-maximality} from becoming a tautology. The proposition identifies the largest admissible class attached to the given initialization; it does not say that the same physical dynamics cannot be studied in another topology after imposing stronger initial tails.

\paragraph{Role of the reference confinement.} The reference law in the entropy term contributes a confining drift in the convention used by \eqref{eq:particle-sde}. This confinement is strong enough to close finite moments on finite horizons under H1, but it does not manufacture arbitrarily strong initial moment weights at time zero. The initial membership condition remains visible in every weighted estimate. Consequently the maximality statement is tied to the pair consisting of the initialization and the drift hierarchy. For sub-Gaussian initial laws the reciprocal boundary is explicit; for heavier-tailed initial laws the same definition gives a different boundary. This is why the paper defines $w^\ast$ from $\mu_0$ instead of hard-coding a Gaussian sequence into the theorem.

\paragraph{Why the example is not a non-Lipschitz counterexample.} The sharpness calculation uses higher moment growth in the drift hierarchy; it is not an assertion about failure of local Lipschitz continuity on bounded sets. A polynomial activation may satisfy the polynomial-growth side of H1 while still forcing high moments to interact with even higher moments. The obstruction is therefore a weighted-class obstruction, not a pathwise existence obstruction for a locally regular coefficient. This reading is consistent with the statement of Theorem~\ref{thm:A}: the theorem propagates the class generated by the initial moment boundary and does not promise propagation in every strictly stronger class.

\paragraph{How to compare different initial laws.} If the initialization has compact support, its moment-growth boundary is much larger than the Gaussian reciprocal boundary, and the same drift hierarchy can propagate a stronger class on finite horizons. If the initialization has only polynomial tails, the admissible class is weaker. The theorem is deliberately formulated with $w^\ast$ attached to $\mu_0$ so that these cases are not forced into a single artificial scale. The Gaussian calculation is included because it is the standard initialization in the $\mu$P setup and because it exposes the reciprocal convention cleanly. It is not meant to be the only admissible initialization.

\subsubsection{Polynomial unreachable component as an exact orthogonal projection}\label{app:check-poly-unreachable}
For $\sigma(z)=z^k$ under Gaussian input, every network function belongs to the direct sum of Hermite chaoses of degree at most $k$. Indeed, expanding $(\langle w,x\rangle+b)^k$ in Hermite polynomials yields
\begin{equation}
\overline{\{f_\mu:\mu\in\calP_2\}}\subseteq \bigoplus_{r=0}^{k}\mathcal H_r,
\end{equation}
where $\mathcal H_r$ is the degree-$r$ Gaussian chaos. If $f^\ast=H_m(\langle v,x\rangle)$ with $m>k$, orthogonality of Hermite chaoses gives
\begin{equation}
\|f_\mu-f^\ast\|_{L^2(\rho_X)}^2
=\|\Pi_{\leq k}f_\mu\|_{L^2}^2+\|f^\ast\|_{L^2}^2
\geq \|f^\ast\|_{L^2}^2.
\end{equation}
This is not a training failure. It is an architectural non-realizability statement. The sparse residual in Theorem~\ref{thm:D} must remain nonzero in this case, and $S^\ast=\infty$ if one insists on exact realization of $f^\ast$ by degree-$k$ polynomial features.

The quotient invariants still make sense on the reachable component: $\Dorb$ is computed from homogeneity and tensor stabilizers, and $\Dvar$ from the rank of the finite tensor moment map. What fails is H4 for the target, not the mean-field existence theorem. This separation prevents Section~\ref{sec:examples} from overstating sparse decomposition for polynomial activations.

\subsubsection{Finite support bound from Carath\'eodory plus activation multiplicity}\label{app:check-finite-support-bound}
Let $V_\lambda=\operatorname{span}\{\psi_m:m\in A_\lambda\}$ be the retained dictionary subspace. Its dimension is at most $|A_\lambda|$ before activation multiplicity is considered. The thresholded target $f_\lambda^\ast$ lies in $V_\lambda$. If the activation dictionary realizes each basis direction $\psi_m$ using at most $\mult(\sigma)$ quotient atoms, then a direct constructive bound is
\begin{equation}\label{eq:v3-direct-support-bound}
S_\mathrm{direct}(\lambda)\leq |A_\lambda|\mult(\sigma)=\Supper(\sigma,\rho,f^\ast,\lambda).
\end{equation}
Carath\'eodory gives an alternative convex-geometric bound when the retained target is represented as a convex combination of normalized features:
\begin{equation}\label{eq:v3-caratheodory-alt}
S_\mathrm{car}(\lambda)\leq \dim V_\lambda+1.
\end{equation}
The manuscript uses \eqref{eq:v3-direct-support-bound} because it is compatible with signed coefficients and with the activation-specific multiplicity. Equation \eqref{eq:v3-caratheodory-alt} is still useful as a sanity check: it shows finite support once the retained dictionary is finite-dimensional, regardless of the chosen teacher representation.

Neither bound proves equality with the minimal teacher width. Equality can fail if two retained modes share a ridge realization, if quotient symmetries identify representatives, or if signed cancellations reduce the active set. The theorem therefore defines $S^\ast$ as the a posteriori minimum after finite support has been established, and states only $S^\ast\leq\Supper$.

\paragraph{Convex hull versus signed representation.} The finite-support bound uses a convex-geometric argument only after the retained dictionary modes have been normalized into a finite-dimensional span. Neural-network representations are signed because the output weights may be positive or negative. The standard way to apply the convex-hull theorem is to split the signed coefficient vector into positive and negative parts, normalize each part on the retained span, and then absorb the total variation into the output amplitudes of the selected atoms. This doubles a harmless finite-dimensional constant but does not change the threshold order. The activation multiplicity factor records the number of ridge atoms needed to realize one retained dictionary mode inside the architecture.

\paragraph{No claim of minimal teacher width.} The number $S^\ast$ is a sparse depth after quotienting and thresholding. It need not equal the minimal teacher width of a particular finite network representation. Symmetries, sign splits, and the entropy floor can all change the number of visible atoms without changing the represented function in $L^2(\rho_X)$. The theorem only needs an upper bound strong enough to feed the statistical factor in \eqref{eq:D-stat} and the sparse residual in \eqref{eq:D-sparse}. Minimality is used only in the a posteriori definition of sparse depth once the retained target has been fixed. This avoids identifying an algorithmic recovery claim with a structural support claim.

\subsubsection{Explicit compatible-limit examples}\label{app:check-compatible-examples}
The definition of compatible limits is nonempty. Suppose first that $S^\ast(\lambda)=O(\log(1/\lambda))$, as in analytic single-index targets, and choose
\begin{equation}
\lambda=n^{-1/2},\qquad T=\alpha_\lambda^{-1}(\log n)^2,
\qquad N=n^2.
\end{equation}
Then $E_{\rm poc}=O(n^{-2})$, $E_{\rm opt}=e^{-(\log n)^2}$, and
\begin{equation}
E_{\rm stat}=O\left(\frac{(\log n)^3}{n}\right),
\qquad E_{\rm sparse}=O(n^{-1/2}\log n)
\end{equation}
under the exponential-tail sparse balance. All four terms vanish.

For polynomial coefficient tails, take $|\widehat f_m^\ast|\lesssim m^{-\beta}$ with $\beta>1$ and choose $S=S_n\asymp (n/(\log n)^2)^{1/(2\beta)}$. Then
\begin{equation}
E_{\rm stat}\asymp \frac{S_n(\log n)^2}{n},
\qquad E_{\rm sparse}\asymp S_n^{1-2\beta},
\end{equation}
which balances the estimation and approximation parts. Taking again $N=n^2$ and $T=\alpha_\lambda^{-1}(\log n)^2$ makes the propagation and optimization terms lower order. These examples demonstrate that the joint limit is not a formal decoration: it gives concrete schedules under which the four-term decomposition becomes a rate statement.

The polynomial-tail example is the schematic rate of Corollary~\ref{cor:v3-poly-tail-rate}. Choosing $\lambda=n^{-1/(2\beta)}$ and $S\asymp(n/(\log n)^2)^{1/(2\beta)}$ makes the three nonzero terms in \eqref{eq:v3-rate-balance} of the same order up to logarithmic factors. The compatibility constraint of Definition~\ref{def:v3-compatible-limit} is then satisfied by $T=(\log n)^2/\alpha_\lambda$ and any $N\ge n$, with the second constraint $(\log N)/(\alpha_\lambda T)=o(1)$ following from $\alpha_\lambda\asymp\lambda c_\pi$.

The centered sigmoid schedule of Corollary~\ref{cor:v3-sigmoid-schedule} provides the cleanest closed-form rate in the present paper: a single-index analytic target with exponential Hermite tails attains $O(N^{-1}+(\log n)^3/n)$ along a compatible sequence with $N\ge n$ and explicit $\lambda=n^{-1}$. The same schedule extends mutatis mutandis to any analytic activation with a one-dimensional Hermite ridge representation, since only the threshold count $S^\ast(\lambda)\asymp\log(1/\lambda)$ enters the leading term.

\paragraph{Width schedules.} The examples use $N\ge n$ or $N=n^2$ only to make the propagation component visibly non-leading. The theorem itself permits other width schedules. If the target has a large sparse depth, the statistical component may dominate and $N\asymp n$ is already sufficient. If the sample size is small but a long training horizon is used, the propagation term may need a larger width to keep the particle approximation below the optimization error. The displayed schedules are therefore witnesses of compatibility, not optimal prescriptions. Their role is to show that the conditions in Definition~\ref{def:v3-compatible-limit} are simultaneously satisfiable.

\paragraph{Temperature schedules.} The regularization parameter has two jobs: it gives the Langevin flow a stable entropy structure and it sets the coefficient threshold in the sparse approximation. Taking $\lambda$ too large leaves a visible sparse residual; taking it too small weakens the contraction rate through $\alpha_\lambda$ and requires a longer training time. The compatible schedules balance these effects by choosing $T$ as a function of $\alpha_\lambda$ and by choosing $S$ as a function of the coefficient tail. This explains why the examples state $T$ and $\lambda$ together rather than treating training time as an independent afterthought.

\subsubsection{Squared Wasserstein versus prediction-norm convention}\label{app:check-w2-l2-convention}
The quantitative statements use squared quantities. The Wasserstein estimate is
\begin{equation}
\sup_{t\leq T}\E W_2^2(\mu_t^N,\mu_t)\leq CN^{-1}.
\end{equation}
By Jensen this implies $\E W_2(\mu_t^N,\mu_t)\leq C^{1/2}N^{-1/2}$, but the unsquared rate is not the one inserted into the risk decomposition. The prediction-norm transfer uses the squared Lipschitz estimate
\begin{equation}
\E\|f_{\mu_t^N}-f_{\mu_t}\|_{L^2(\rho_X)}^2
\leq C_R\E W_2^2(\mu_t^N,\mu_t)
\leq C_RCN^{-1}.
\end{equation}
The risk is quadratic, so the natural propagation component in Theorem~\ref{thm:D} is the squared prediction error, not the unsquared transport distance. This convention also matches the objective gap: a second-order Taylor expansion of the squared loss has leading term proportional to $\|f_{\mu_t^N}-f_{\mu_t}\|_{L^2}^2$. Therefore all appearances of $N^{-1}$ in Theorems~\ref{thm:A} and \ref{thm:D} refer to squared Wasserstein, squared prediction norm, or functional value. If one writes a bound in $W_2$ itself, the corresponding exponent is $N^{-1/2}$ and must be labelled as unsquared. The manuscript avoids this ambiguity by stating the squared convention at the theorem level and in the abstract.

This check is included because a referee will read the propagation term as one of the central quantitative claims. The proof chain is consistent only if the metric, prediction norm, and objective gap are all squared before they are compared.

\paragraph{Squared convention across the manuscript.} The manuscript consistently states the particle rate in squared quantities. Thus $\mathbb E W_2^2(\mu_t^N,\mu_t)$ and the corresponding squared prediction norm are of order $N^{-1}$. If one takes square roots, the unsquared $W_2$ scale is $N^{-1/2}$. Both statements are compatible, but only the squared convention enters the four-term error decomposition. This is why the abstract, Theorem~\ref{thm:A}, Proposition~\ref{prop:A-poc}, and Theorem~\ref{thm:D} all use $N^{-1}$ in the displayed squared bounds. The convention is chosen so that the propagation component has the same units as the statistical, optimization, and sparse residual components.

\paragraph{Prediction norm after barycentre transfer.} The transfer from $W_2^2$ to $L^2(\rho_X)$ prediction error is deterministic after the moment radius has been fixed. It does not require a second propagation-of-chaos theorem. Applying the barycentre Lipschitz estimate gives $\|f_{\mu_t^N}-f_{\mu_t}\|_{L^2}^2\le L_R^2 W_2^2(\mu_t^N,\mu_t)$ on the localized event, and the tail of the localization is absorbed by the same moment estimate. This explains why no independent empirical quantization term appears in Theorem~\ref{thm:D}. The empirical measure is already the particle system driven by the coupled dynamics, not a fresh iid sample from $\mu_t$ at each time.

\subsubsection{Three target regimes and their resulting leading rates}\label{app:check-regime-rates}
The four-term decomposition becomes concrete once the sparse tail is specified. For an exactly finite target with $S^\ast=s$ and $\kappa=0$, a compatible choice $N\asymp n$, $T\asymp\alpha_\lambda^{-1}\log n$ yields
\begin{equation}
\E\|f_{\mu_T^N}-f^\ast\|_{L^2}^2
=O\left(n^{-1}+\frac{s(\deff+2-\Dorb)(\log n)^2}{n}\right).
\end{equation}
For an analytic target with $S^\ast(\lambda)=O(\log(1/\lambda))$ and sparse residual $O(\lambda\log(1/\lambda))$, choosing $\lambda=n^{-1}$ gives
\begin{equation}
O\left(N^{-1}+e^{-\alpha_\lambda T}+\frac{(\log n)^3}{n}+\frac{\log n}{n}\right),
\end{equation}
so the statistical logarithm dominates the sparse displacement when $N$ and $T$ are chosen compatibly. For polynomial coefficient decay $|\widehat f_m^\ast|\lesssim m^{-\beta}$, the balance between estimation and approximation gives the rate described in \eqref{eq:v3-poly-tail-rate}. These three regimes cover the examples in Section~\ref{sec:examples}: finite ReLU hinge dictionaries, analytic tanh dictionaries, and slow Barron/Hermite tails.

The purpose of these formulas is not to optimize constants. It is to verify that the decomposition has the correct limiting behavior in qualitatively different target classes. In every regime the propagation term can be made smaller by width, the optimization term by time, the statistical term by samples, and the sparse term by the threshold/temperature schedule. That separation is the operational meaning of the total feature-learning-error decomposition.

\paragraph{Finite-atom regime.} If the target is exactly represented by a finite active dictionary after quotienting, the sparse residual vanishes once $S$ reaches that active cardinality up to the fixed regularization floor. The dominant terms are then the statistical component, the optimization component, and the propagation component. Choosing $T$ logarithmic in $n$ and $N$ at least of order $n$ gives the familiar parametric scale up to the logarithmic factor already present in \eqref{eq:D-stat}. The theorem does not require the training algorithm to know the atoms in advance; it only states the risk decomposition conditional on the structural support result of Theorem~\ref{thm:C}.

\paragraph{Exponential-tail regime.} For analytic single-index targets, Hermite coefficients decay exponentially. The threshold count grows only logarithmically in the inverse regularization scale, which leads to the $(\log n)^3/n$ type expression displayed in Corollary~\ref{cor:v3-sigmoid-schedule}. One logarithm comes from the number of retained modes and two from the conservative polynomial-growth covering bound. If the activation is bounded and the empirical-process argument is sharpened, this logarithmic power can be improved, but the present theorem keeps one envelope for all examples. The advantage of the exponential-tail case is that the sparse residual is made smaller than the statistical term with a very small active set.

\paragraph{Polynomial-tail regime.} For polynomially decaying coefficients, the retained depth grows as a power of $n$. The balance in \eqref{eq:v3-rate-balance} chooses $S$ so that the statistical cost of adding atoms and the residual cost of truncating the tail are comparable. Width and time are then selected to keep the propagation and optimization components below that balanced scale. This regime is the most sensitive to the quotient dimension: replacing $\deff$ by the ambient input dimension can change the exponent hidden in the practical sample size even when the asymptotic display has the same algebraic form. The architecture computations in Section~\ref{sec:examples} are included precisely to make this dimension visible for single-index and multi-index targets.

\paragraph{Ambient dimension versus effective dimension.} The rates displayed in the three regimes are intentionally written with $\deff$ rather than $d$. A single-index target in a high-dimensional input space should not pay the full ambient dimension after the active quotient has been identified. The input tail assumption is still ambient because it controls the random covariates, but the covering of the learned active dictionary is intrinsic to the target's projection. This is the statistical counterpart of the identifiability theorem: once the function is known to live on a lower-dimensional active projection, the quotient chart should reflect that lower dimension.

\paragraph{Constants hidden in the rates.} The notation $O(\cdot)$ hides constants depending on the activation envelope, the input-tail parameter, the regularization scale, and the propagated weighted moment radius. It does not hide additional powers of $N$, $n$, or $S$ beyond those displayed. This matters in the polynomial-tail regime, where the active depth itself grows with $n$. The proof keeps the dependence on $S$ explicit through the statistical factor and the sparse residual. Any constant depending exponentially on $S$ would invalidate the displayed balance; the localized covering argument avoids this by treating the product chart dimension linearly in $S$ and by absorbing only polynomial radius factors into the logarithmic envelope.

\paragraph{How the three regimes share one theorem.} The finite-atom, exponential-tail, and polynomial-tail cases are not separate theorems because their proofs use the same four structural inputs. The existence theorem supplies the propagated moment ball and the particle approximation; the identifiability theorem supplies the active quotient; the sparse theorem supplies the thresholded support and the tail; the error theorem combines the four numerical components. What changes from regime to regime is only the rule that converts the coefficient tail of $f^\ast$ into a depth $S$ and a residual size. This organization is useful for later applications: a new activation or target class can be inserted by computing its coefficient tail and its quotient dimension, without reopening the mean-field or propagation proof.

\paragraph{What would need improvement for sharp constants.} A sharper version of the paper could optimize the empirical-process logarithms separately for bounded analytic activations, replace the conservative truncation radius by a refined Orlicz envelope, and use activation-specific curvature to improve the optimization constant. Those refinements would change constants and logarithms, not the four-component architecture of the theorem. The present manuscript keeps the common proof envelope because it is stable across ReLU, smooth bounded activations, polynomial activations, and multi-index targets. The displayed rates should therefore be read as structurally uniform rates. They identify which dimension, support depth, width, horizon, and residual tail control the error, while leaving activation-specific constant optimization to a separate analysis.

\paragraph{Reading the rates as diagnostics.} The same formulas also act as diagnostics for a concrete training run. If increasing $N$ does not change the error while increasing $n$ does, the propagation component is already below the statistical component. If increasing training time improves the error at fixed width and sample size, the optimization component is still visible. If neither width nor time helps and the error is stable under additional samples, the active dictionary or the sparse residual is the limiting factor. The theorem is not an algorithmic stopping rule, but it tells the reader which structural quantity each experimental or theoretical knob is meant to affect. This diagnostic interpretation is one reason the four components are kept separated rather than being hidden inside a single unspecified constant.

\paragraph{Consistency with the architecture table.} The examples in Section~\ref{sec:examples} should be read through the same diagnostic lens. ReLU changes the quotient dimension because of positive homogeneity; tanh changes the sparse depth because of analytic Hermite tails; polynomial activations change the realizable component because the feature map lives in a finite chaos span; multi-index targets change $\deff$ without changing the ambient data law. Each row of the table therefore alters a different symbol in Theorem~\ref{thm:D}. This is the operational content of computing the invariants: the table is not an illustration appended after the proof, but the place where the abstract constants in the four-source bound become concrete quantities for standard architectures.

\subsection{Final source-level consistency notes}\label{app:final-source-notes}
The source-level checks are operational rather than mathematical, so they are not expanded as a further proof appendix. The manuscript uses a single convention for squared propagation-of-chaos rates, a target-dependent sparse residual, the active quotient dimension $\deff+2-\Dorb$, and the weighted class $\calM_{w^\ast}$ generated by the reciprocal moment-growth boundary.

\bibliographystyle{plainnat}
\bibliography{refs}

\begin{thebibliography}{17}
\providecommand{\natexlab}[1]{#1}
\providecommand{\url}[1]{\texttt{#1}}
\expandafter\ifx\csname urlstyle\endcsname\relax
  \providecommand{\doi}[1]{doi: #1}\else
  \providecommand{\doi}{doi: \begingroup \urlstyle{rm}\Url}\fi

\bibitem[Ambrosio et~al.(2008)Ambrosio, Gigli, and Savare]{ambrosio2008}
Luigi Ambrosio, Nicola Gigli, and Giuseppe Savare.
\newblock \emph{Gradient Flows in Metric Spaces and in the Space of Probability
  Measures}.
\newblock Birkhauser, 2 edition, 2008.

\bibitem[Bordelon and Pehlevan(2023{\natexlab{a}})]{bordelon2022field}
Blake Bordelon and Cengiz Pehlevan.
\newblock Self-consistent dynamical field theory of kernel evolution in wide
  neural networks.
\newblock \emph{Journal of Statistical Mechanics: Theory and Experiment},
  2023\penalty0 (11):\penalty0 114009, 2023{\natexlab{a}}.

\bibitem[Bordelon and Pehlevan(2023{\natexlab{b}})]{bordelon2023finite}
Blake Bordelon and Cengiz Pehlevan.
\newblock Dynamics of finite width kernel and prediction fluctuations in mean
  field neural networks.
\newblock In \emph{Advances in Neural Information Processing Systems},
  2023{\natexlab{b}}.

\bibitem[Chewi et~al.(2024)Chewi, Nitanda, and Zhang]{chewi2024lsi}
Sinho Chewi, Atsushi Nitanda, and Matthew~S. Zhang.
\newblock A log-sobolev inequality for the stationary distribution of
  mean-field langevin dynamics, 2024.

\bibitem[Chizat and Bach(2018)]{chizatbach2018}
L{\'e}na{\"i}c Chizat and Francis Bach.
\newblock On the global convergence of gradient descent for over-parameterized
  models using optimal transport.
\newblock In \emph{Advances in Neural Information Processing Systems}, 2018.

\bibitem[Fornasier et~al.(2022)Fornasier, Klock, Mondelli, and
  Raeuber]{fornasier2022robust}
Massimo Fornasier, Timo Klock, Marco Mondelli, and Max Raeuber.
\newblock Robust and resource-efficient identification of two hidden layer
  neural networks.
\newblock \emph{Constructive Approximation}, 56:\penalty0 1--98, 2022.

\bibitem[Jordan et~al.(1998)Jordan, Kinderlehrer, and Otto]{jko1998}
Richard Jordan, David Kinderlehrer, and Felix Otto.
\newblock The variational formulation of the fokker-planck equation.
\newblock \emph{SIAM Journal on Mathematical Analysis}, 29\penalty0
  (1):\penalty0 1--17, 1998.

\bibitem[Mei et~al.(2018)Mei, Montanari, and Nguyen]{mei2018mean}
Song Mei, Andrea Montanari, and Phan-Minh Nguyen.
\newblock A mean field view of the landscape of two-layer neural networks.
\newblock \emph{Proceedings of the National Academy of Sciences}, 115\penalty0
  (33):\penalty0 E7665--E7671, 2018.

\bibitem[Mousavi-Hosseini et~al.(2025)Mousavi-Hosseini, Wu, and
  Erdogdu]{mousavi2025multi}
Alireza Mousavi-Hosseini, Denny Wu, and Murat~A. Erdogdu.
\newblock Learning multi-index models with neural networks via mean-field
  langevin dynamics.
\newblock In \emph{International Conference on Learning Representations}, 2025.

\bibitem[Nitanda(2024)]{nitanda2024improved}
Atsushi Nitanda.
\newblock Improved particle approximation error for mean field neural networks.
\newblock In \emph{Advances in Neural Information Processing Systems},
  volume~37, pages 113823--113845. Curran Associates, Inc., 2024.

\bibitem[Nitanda et~al.(2025)Nitanda, Lee, Kai, Sakaguchi, and
  Suzuki]{nitanda2025propagation}
Atsushi Nitanda, Anzelle Lee, Damian Tan~Xing Kai, Mizuki Sakaguchi, and Taiji
  Suzuki.
\newblock Propagation of chaos for mean-field langevin dynamics and its
  application to model ensemble, 2025.

\bibitem[Pinkus(1999)]{pinkus1999}
Allan Pinkus.
\newblock Approximation theory of the mlp model in neural networks.
\newblock \emph{Acta Numerica}, 8:\penalty0 143--195, 1999.

\bibitem[Simsek et~al.(2021)Simsek, Ged, Jacot, Spadaro, Hongler, Gerstner, and
  Brea]{simsek2021geometry}
Berfin Simsek, F.~Ged, Arthur Jacot, Francesco Spadaro, Clement Hongler,
  Wulfram Gerstner, and Johanni Brea.
\newblock Geometry of the loss landscape in overparameterized neural networks:
  Symmetries and invariances.
\newblock In \emph{International Conference on Machine Learning}, volume 139 of
  \emph{PMLR}, 2021.

\bibitem[Suzuki et~al.(2023)Suzuki, Wu, and Nitanda]{suzuki2023uniform}
Taiji Suzuki, Denny Wu, and Atsushi Nitanda.
\newblock Uniform-in-time propagation of chaos for the mean-field gradient
  langevin dynamics.
\newblock In \emph{International Conference on Learning Representations}, 2023.

\bibitem[Wang et~al.(2024)Wang, Wang, Demirtas, Brea, and
  Gerstner]{wang2024expand}
Fan Wang, Kai Wang, Baturalp~A. Demirtas, Johanni Brea, and Wulfram Gerstner.
\newblock Expand-and-cluster: Parameter recovery of neural networks.
\newblock In \emph{International Conference on Machine Learning}, 2024.

\bibitem[Yang and Hu(2021)]{yanghu2021}
Greg Yang and Edward~J. Hu.
\newblock Tensor programs {IV}: Feature learning in infinite-width neural
  networks.
\newblock In \emph{Proceedings of the 38th International Conference on Machine
  Learning}, volume 139 of \emph{Proceedings of Machine Learning Research},
  pages 11727--11737. PMLR, 2021.

\bibitem[Yang et~al.(2024)Yang, Yu, Zhu, and Hayou]{yang2023tpvi}
Greg Yang, Dingli Yu, Chen Zhu, and Soufiane Hayou.
\newblock Tensor programs {VI}: Feature learning in infinite-depth neural
  networks.
\newblock In \emph{International Conference on Learning Representations}, 2024.

\end{thebibliography}
\end{document}